\RequirePackage{fix-cm}
\documentclass[smallextended]{svjour3}       
\smartqed  
\usepackage{graphicx}
 \journalname{Data Mining and Knowledge Discovery}
\usepackage[colorlinks,
            linkcolor=black,
            anchorcolor=black,
            citecolor=black
            ]{hyperref}
\usepackage{natbib}
\usepackage{subfigure}
\usepackage{amsfonts,amsmath,amssymb,bm}       
\usepackage{multirow}
\usepackage{color}
\usepackage{booktabs}

\usepackage[titletoc,toc,title]{appendix}

\newcommand{\myminitab}[2][l]{\begin{tabular}{#1}#2\end{tabular}}

\def\methodname{\textsc{Hermit}}

\DeclareMathOperator*{\argmax}{argmax}

\def\trans{^{\rm T}}

\def\oby{\tilde{\mathbf{y}}}

\def\L{\mathbf{L}}

\def\X{\mathbf{X}}

\def\Y{\mathbf{Y}}

\def\P{\mathbf{P}}

\def\I{\mathbf{I}}
\def\A{\mathbf{A}}

\def\S{\mathbf{S}}

\def\x{\mathbf{x}}
\def\y{\mathbf{y}}

\def\aalpha{\mbox{\boldmath$\alpha$}}
\def\bbeta{\mbox{\boldmath$\beta$}}
\def\zzeta{\mbox{\boldmath$\zeta$}}

\def\PPhi{\mbox{\boldmath$\Phi$}}

\def\xxi{\mbox{\boldmath$\xi$}}

\def\vvarphi{\mbox{\boldmath$\varphi$}}

\newtheorem{mydef}{Definition}
\newtheorem{mythe}{Theorem}
\newtheorem{mycon}{Condition}
\newtheorem{mylem}{Lemma}
\newtheorem{mycor}{Corollary}

\begin{document}

\title{Robust Finite Mixture Regression for Heterogeneous Targets}



\author{Jian Liang
        \and
        Kun Chen
        \and
        Ming Lin
        \and
        Changshui Zhang
        \and
        Fei Wang
}


\institute{Jian Liang \at{
              Department of Automation, Tsinghua University\\
       State Key Lab of Intelligent Technologies and Systems\\
       Tsinghua National Laboratory for Information Science and Technology(TNList)\\
       Beijing. 100084. P.R.China}\\
              \email{liangjian12@mails.tsinghua.edu.cn}           
           \and
           Kun Chen \at{
              Department of Statistics\\
       University of Connecticut\\
       Storrs, CT 06269, USA}\\
              \email{kun.chen@uconn.edu}
              \and
           Ming Lin \at{
              Department of Computational Medicine and Bioinformatics\\
       University of Michigan\\
       Ann Arbor, MI 48109, USA}\\
              \email{linmin@umich.edu}
              \and
           Changshui Zhang \at{
              Department of Automation, Tsinghua University\\
       State Key Lab of Intelligent Technologies and Systems\\
       Tsinghua National Laboratory for Information Science and Technology(TNList)\\
       Beijing. 100084. P.R.China}\\
              \email{zcs@mails.tsinghua.edu.cn}
              \and
           Fei Wang \at{
              Department of Healthcare Policy and Research\\
       Cornell University\\
       New York City. NY 10065. USA}\\
              \email{few2001@med.cornell.edu}
}

\date{Received: date / Accepted: date}

\maketitle

\begin{abstract}
Finite Mixture Regression (FMR) refers to the mixture modeling scheme which learns multiple regression models from the training data set. Each of them is in charge of a subset. FMR is an effective scheme for handling sample heterogeneity, where a single regression model is not enough for capturing the complexities of the conditional distribution of the observed samples given the features. In this paper, we propose an FMR model that 1) finds sample clusters and jointly models multiple incomplete mixed-type targets simultaneously, 2) achieves shared feature selection among tasks and cluster components, and 3) detects anomaly tasks { or clustered structure among tasks, }and accommodates outlier samples.
We provide non-asymptotic oracle performance bounds for our model under a high-dimensional learning framework. The proposed model is evaluated on both synthetic and real-world data sets. The results show that our model can achieve state-of-the-art performance.
\end{abstract}

\section{Introduction}

Regression modeling, which refers to building models to learn conditional relationship between output targets and input features on some training samples, is a fundamental problem in statistics and machine learning. Some classical regression modeling approaches include least square regression, logistic regression and Poisson regression; see, e.g.,~\citet{bishop2006pattern},~\citet{kubat2015introduction},~\citet{fahrmeir2013regression} and the references therein.

The aforementioned classic approaches usually train a single model of a single target over the entire data set. However, real-world problems can be much more complicated. In particular, the needs of utilizing high-dimensional features, population heterogeneity, and multiple interrelated targets are among the most prominent complications. To handle high-dimensional data, the celebrated regularized estimation approaches have undergone exciting developments in recent years; see, e.g.,~\citet{FanLv2010} and~\citet{Huang2012}. In the presence of population heterogeneity, the samples may form several distinct clusters corresponding to mixed relationships between the targets and the features. A popular modeling strategy in such a scenario is the {\em Finite Mixture Regression} (FMR)~\citep{mclachlan2004finite}, which is capable of adaptively learning multiple models, each of which is responsible for one subset/cluster of the data. FMR models have been widely used in market segmentation studies, patients' disease progression subtyping, motif gene-expression research, etc.; see, e.g.,~\citet{stadler2010},~\citet{khalili2011overview},~\citet{khalili2012variable},~\citet{dougru2016parameter}, and the references therein. The problem of joint learning for multiple targets is usually referred to as {\em Multi-Task Learning} (MTL) in machine learning or {\em multivariate learning} in statistics; see, e.g.,~\citet{evgeniou2007multi},~\citet{argyriou2007spectral},~\citet{chen2011integrating}, and~\citet{gong2012multi}. {We stress that the main definition of MTL considers tasks that do not necessarily share the same set of samples (and features), and that this paper focuses on a special case of MTL, where the multivariate outcomes are collected from the same set of samples and share the same set of features, whose reason will be explained later. }There have also been multi-task FMR models, e.g.,~\citet{wedel1995mixture},~\citet{xian2004robust},~\citet{kyung2016robust} and~\citet{bai2016mixture}, which mainly built on certain multivariate probability distribution such as Gaussian distribution or multivariate $t$ distribution.

Thus far, a comprehensive study on multi-task mixture-regression modeling with high-dimensional data is still lacking. To tackle this problem for handling real-world applications, there remain several challenges and practical concerns.
\begin{itemize}
  \item {\em Task Heterogeneity}. Current MTL approaches usually assume that the targets are of the same type. However, it is common that the multiple targets are of different types, such as continuous, binary, count, etc., which we refer to as task heterogeneity. For example, in anesthesia decision making~\citep{tan2010decision}, the anesthesia drugs will have impacts on multiple indicators of an anesthesia patient, such as anesthesia depth, blood pressures, heart rates, etc. The anesthesiologist needs to consider all those different aspects as well as their intrinsic dependence before making the decision.

  \item {\em Task Integration}. As in the anesthesiology example, the multiple tasks are typically inter-related to each other, and the potential benefit from a MTL approach needs to be realized through properly exploring and taking advantage of these relationships. In existing high-dimensional MTL approaches, the tasks are usually integrated by assuming certain shared conditional mean structures between the targets and the features. The problem is more difficult in the presence of both task and population heterogeneities.

  \item {\em Task Robustness}. Similar to the idea in the robust MTL approaches~\citep{passos2012flexible, gong2012robust, chen2011integrating}, it is not always the case that jointly considering all tasks by assuming certain shared structures among them would be helpful. Certain tasks, referred to as anomaly tasks, may not follow the assumed shared structure and thus can ruin the overall model performance. { More generally, tasks may even cluster into groups with different shared structures.}
 \end{itemize}

In this paper, we propose a novel {method} named {HEterogeneous-target Robust MIxTure regression} (\methodname), to address the above challenges in a unified framework. {Here we explain that we mainly consider the setting{, where} the multivariate outcomes are collected from the same set of samples and share the same set of features because our main objective is to learn potentially shared sample clusters and feature sets among tasks.} Rigorous theoretical analysis and performance guarantees are provided. It is worthwhile to highlight the key aspects of our approach as follows.


\begin{itemize}
\item Our method handles mixed type of targets simultaneously. Each target follows an exponential dispersion family model~\citep{jorgensen1987exponential}, so that multiple different types of targets, e.g., continuous, binary, and counts, can be handled jointly. The tasks are naturally integrated through sharing the same clustering structure arising from population heterogeneity. Our theory allows~\methodname\, to cover sub-exponential distributions, including the commonly-encountered Bernoulli, Poisson and Gaussian as special cases.

\item Our method imposes structural constraints in each mixture component of~\methodname, to deal with the curse of dimensionality and at the same time further take advantage of the interrelationship of the tasks. In particular, the group $\ell_1$ penalization is adopted to perform shared feature selection among tasks within each mixture component.

\item Our method adopts {three} strategies for robustness. First, we adopt a mean-shift regularization technique~\citep{she2015robust} to detect the outlier samples automatically and adjust for its outlying effects in model estimation. The second strategy measures discrepancy of different conditional distributions to detect anomaly tasks. {The third strategy measures similarity between each pair of tasks to discover a clustered structure among tasks.} Moreover, our model can work with incomplete data and impute entry-wise missing values in the multiple targets.
\end{itemize}

The aforementioned key elements, e.g., multi-task learning, sample clustering, shared feature selection, and anomaly detection, are integrated in a unified mixture model setup, so that they can benefit from and reinforce each other. A generalized Expectation-Maximization (GEM)~\citep{neal1998view} algorithm is developed to conduct model estimation efficiently. For theoretical analysis, we generalize the results of~\citet{stadler2010} to establish non-asymptotic oracle performance bounds for~\methodname\, under a high-dimensional learning framework. This is not trivial due to the non-convexity (due to the population heterogeneity) and the target heterogeneity of the problem.

The rest of this paper is organized as follows. Section 2 provides a brief review of the background and the related works to our method. Section 3 presents the details of the proposed~\methodname\, model and the computational algorithm. {Section 4 discusses several extensions of our method.} Section 5 shows the theoretical analysis. The empirical evaluations are presented in Section 6, followed by the {discussions and} conclusions in Section 7.

\section{Background \& Related Work}\label{sec:related}


Let $Y \in \mathcal{Y} \subset \mathbb{R} $ be the output target and $ \x \in \mathcal{X} \subset \mathbb{R}^{ d}$ the input feature vector. GLMs~\citep{nelder1972generalized} postulate that the conditional probability density function of $Y$ given $\x$ is
\begin{equation*}
f(y \mid {\x,\theta}) = f(y \mid \varphi,\phi) = \exp\biggl\{\frac{y\varphi-b(\varphi)}{a(\phi)}+c(y,\phi)\biggr\},
\end{equation*}
where $\varphi = \x\trans\bbeta$ with $\bbeta$ being the regression coefficient vector, $\phi$ is a dispersion parameter, and $a(\cdot),b(\cdot),c(\cdot)$ are known functions whose forms are determined by the specific distribution. Here we use $\theta$ to denote the collection of all the unknown parameters, i.e., $\theta = (\bbeta,\phi)$. Least square regression, logistic regression and Poisson regression are all special cases of GLMs. In the presence of {population} heterogeneity, a standard finite mixture model of GLM postulates that the conditional probability density function of $Y$ given $\x$ is
\begin{equation*}
f(y \mid \x,\theta) = \sum_{r=1}^k\pi_r f(y \mid \varphi_r,\phi_{r}),
\end{equation*}
where $\varphi_r = \x\trans\bbeta_r$ with $\bbeta_r$ being the regression coefficient vector for the $r$th mixture component, and $\pi_r>0 \ (r=1,\dots,k)$ with $\sum_{r=1}^k\pi_r=1$. So FMR model assumes that there are $k$ sub-populations, each of which admits a different conditional relationship between $Y$ and $\x$. 

\citet{mclachlan2004finite} introduced finite mixture of GLM models. {\citet{bartolucci2005use} considered a special case that $\bbeta_1,\dots,\bbeta_k$ are different only in their first entries.}~\citet{khalili2012variable} discussed using sparse penalties such as Lasso and SCAD to perform feature selection for FMR models and showed asymptotic properties of the penalized estimators.~\citet{stadler2010} reparameterized the finite mixture of Gaussian regression model and used $\ell_1$ penalization to achieve bounded log-likelihood and consistent feature selection. For multiple targets,~\citet{wedel1995mixture} proposed finite mixture of GLM models with multivariate targets. { These methods only consider the univariate-outcome case}.~\citet{weruaga2015sparse} proposed multivariate Gaussian mixture regression and used $\ell_1$ penalty for sparseness of parameters. Besides mixture of GLMs, there have been many works on mixture of other continuous distributions such as $t$ and Laplace distributions, mainly motivated by the needs of robust estimation for handling heavy tailed or skewed continuous distribution; see, e.g.,~\citet{xian2004robust,dougru2016parameter,alfo2016finite,dougru2016robust, kyung2016robust,bai2016mixture}. { However, these methods assume that the targets are of the same type, and only consider interrelationship among tasks with continuous outcomes. Additionally, they all assumed that their FMR model is shared by all the tasks.}

In MTL,~\citet{kumar2012learning},~\citet{passos2012flexible},~\citet{gong2012robust},~\citet{chen2011integrating},~\citet{jacob2009clustered},~\citet{chen2010graph} and~\citet{he2011graph} proposed to tackle the problem that different groups of tasks share different information, providing methods to handle anomaly tasks, clustered structure or graph-based structure among tasks.~\cite{yang2009heterogeneous} proposed a multi-task framework to jointly learn tasks with output types of Gaussian and multinomial.~\citet{zhang2012multi} proposed a multi-modal multi-task model to predict clinical variables for regression and categorical variable for classification jointly.~\citet{li2014heterogeneous} proposed a heterogeneous multi-task learning framework to learn a pose-joint regressor and a sliding window body-part detector in a deep network architecture simultaneously. {Nevertheless, these MTL methods cannot handle the heterogeneity of conditional relationship between features and targets.

By contrast, the proposed FMR framework~\methodname\, is effective for handling sample heterogeneity with mixed type of tasks whose interrelationship are harnessed by structural constraints. Non-asymptotic theoretical guarantees are provided. It also handles anomaly tasks or clustered structure among tasks, for the case that not all the tasks share the same FMR structure.
}

\section{HEterogeneous-target Robust MIxTure regression}\label{sec:method}

In this section, we first present the formulation of the { main} ~\methodname\ model, followed by penalized likelihood estimators with sparse constraint and structural constraint, respectively. We then introduce the associated optimization procedures, and describe how to perform sample clustering and make imputation of the missing/unobserved outcomes on incomplete multi-target outcomes based on the main model. {Hyper-parameter tuning is discussed at last.}  {Various extensions of the main methodology, including strategies to handle anomaly tasks or clustered tasks, will be introduced in Section~\ref{sec:extension}.}

\subsection{Model Formulation and Estimation Criterion}

Let $\Y\in\mathbb{R}^{n\times m}$ be the output/target data matrix and $\X\in\mathbb{R}^{n\times d}$ the input/feature data matrix, consisting of $n$ independent samples $(\y_i,\x_i)$, $i=1,\ldots, n$. As such, there are $m$ different targets with a common set of $d$ features. We allow $\Y$ to contain missing values at random; define $\Omega_i =\{j\in\{1,\ldots,m\}: y_{ij} \mbox{ is observed.}\}$ be the collection of indices of observed outcome in the $i$th sample $\y_i$ ($\Omega_i\neq \emptyset$), for $i=1,\ldots,n$.

To model multiple types of targets, such as continuous, binary, count, etc., we allow $y_{ij}$ to potentially follow different distributions in the exponential-dispersion family, for each $j=1,\ldots, m$. Specifically, we assume that given $\x_i$, the joint probability density function of $\oby_i = \{y_{ij};j\in \Omega_i\}$ is
\begin{equation}\label{eq:exp_fm}
f(\oby_i \mid \x_i,\theta) = f(y_{ij}, j\in \Omega_i \mid \x_i,\theta) = \sum_{r=1}^k\pi_r\prod_{j\in\Omega_i}f(y_{ij}\mid \varphi_{ijr},\phi_{jr}),
\end{equation}
where
\begin{equation*}
f(y_{ij}\mid \varphi_{ijr},\phi_{jr}) = \exp\biggl\{\frac{y_{ij}\varphi_{ijr}-b_j(\varphi_{ijr})}{a_j(\phi_{jr})}+c_j(y_{ij},\phi_{jr})\biggr\},
\end{equation*}
$\varphi_{ijr}$ is the natural parameter for the $i$th sample of the $j$th target in the $r$th mixture component, $\phi_{jr}$ is the dispersion parameter of the $j$th target in the $r$th mixture component, and the functions $a_j,b_j,c_j \ (j=1,\ldots,m)$ are determined by the specific distribution of the $j$th target. Here, the key assumption is that the $m$ tasks all correspond to the same cluster structure (e.g., the $m$ tasks all have $k$ clusters) determined by the underlying population heterogeneity; given the shared cluster label (e.g., $r$), the tasks within each mixture component then become independent of each other (depicted by the product of their probability density functions). As such, by allowing cluster label sharing, the model provides an effective way to genuinely integrate the learning of the multiple tasks.

Following the setup of GLMs, we assume a linear structure in the natural parameters, i.e.,
\begin{equation}\label{eq:natural_parameter}
\varphi_{ijr} = \x_i\bbeta_{jr},
\end{equation}
where $\bbeta_{jr}$ is the regression coefficient vector of the $j$th response in the $r$th mixture component. Since $\x_i$ is possibly of high dimensionality, the $\bbeta_{jr}$s are potentially sparse vectors. For example, when the $\bbeta_{jr}$s for $j=1,\ldots,m$ share the same sparsity pattern, the tasks share the same set of relevant features within each mixture component. For $r=1,\ldots,k$, write $\bbeta_r \in \mathbb{R}^{d\times m} = [\bbeta_{1r},\bbeta_{2r},\ldots,\bbeta_{mr}]$ and $\phi_r=[\phi_{1r},\ldots,\phi_{mr}]^T$. Also write $\bbeta\in\mathbb{R}^{(d\times m)\times k} =[\bbeta_1,\ldots, \bbeta_k]$. Let $\theta = \{ \bbeta,\phi_1,\ldots,\phi_k,\pi_1,\ldots,\pi_{k}\}$ collecting all the unknown parameters, with the parameter space given by
$
\Theta = \mathbb{R}^{ (d \times m) \times k}  \times \mathbb{R}^{m \times k}_{>0} \times \Pi,
$
where $\Pi = \{\pi;\pi_r>0 \mbox{ for } r = 1,\ldots,k \mbox{ and }  \sum_{r=1}^{k} \pi_r = 1\}$. 

The data log-likelihood of the proposed model is
\begin{equation}\label{eq:loglike}
\begin{split}
\ell(\theta\mid\Y,\X) & = \sum_{i=1}^n\log\left(\sum_{r=1}^k\pi_r\prod_{j\in\Omega_i}\exp\biggl\{\frac{y_{ij}\varphi_{ijr}-b_j(\varphi_{ijr})}{a_j(\phi_{jr})}+c_j(y_{ij},\phi_{jr})\biggr\}\right).
\end{split}
\end{equation}
The missing values in $\Y$ simply do not contribute to the likelihood, which follows the same spirit as in matrix completion~\citep{candes2009}. The proposed model indeed possesses a genuine multivariate flavor, as the different outcomes share the same underlying latent cluster pattern of the heterogeneous population. We then propose to estimate $\theta$ by the following penalized likelihood criterion:
\begin{equation}\label{eq:l1}
\hat{\theta} =  {\arg\min}_{\theta\in\Theta} \  - \ell(\theta\mid\Y,\X)/n
+ \mathcal{R}(\bbeta;\lambda),
\end{equation}
where $\mathcal{R}(\bbeta;\lambda)$ is some certain penalty term on the regression coefficients with $\lambda$ being a tuning parameter.

We thus name our proposed method the HEterogeneous-target Robust MIxTure regression (\methodname). The $\mathcal{R}(\bbeta;\lambda)$ can be flexibly chosen based on specific needs of feature selection. The first sparse penalties adopted by our model is the $\ell_1$ norm (lasso-type) penalty,
\begin{equation}\label{eq:pen_lasso}
\mathcal{R}(\bbeta;\lambda,\pi) = \lambda\sum_{r=1}^k\pi_r^{\gamma}\|\bbeta_r\|_1,
\end{equation}
where $\lambda$ is the tuning parameter, $\|\cdot\|_1$ is the entry-wise $\ell_1$ norm, and $\pi_r^{\gamma}$s $(r=1,\ldots,k)$ are the penalty weights with $\gamma\in\{0,1/2,1\}$ being a pre-specified constant. Here the penalty also depends on the unknown mixture proportions $\pi$; when the cluster sizes are expected to be imbalanced, using this weighted penalization with some $\gamma>0$ is preferred~\citep{stadler2010}. This entry-wise regularization approach allows the tasks to have independent set of relevant features. Alternatively, in order to enhance the integrative learning and potentially boost the performance of clustering, it could be beneficial to encourage the internal similarity within each sub-population. Then certain group-wise regularization of the features could be considered, which are widely adopted in multi-task learning. In particular, we consider a component-specific group sparsity pattern to achieve shared feature selection among different tasks, in which the group $\ell_1$ norm penalty is used~\citep{gong2012robust,jalali2010dirty},
\begin{align}\label{eq:pen_lasso_group}
\mathcal{R}(\bbeta;\lambda,\pi) = \lambda\sum_{r=1}^k\pi_r^{\gamma}\|\bbeta_r\|_{1,2},
\end{align}
where $\|\cdot\|_{1,2}$ denotes the sum of the row $\ell_2$ norms of the enclosed matrix, and the weights are constructed as in~\eqref{eq:pen_lasso}. {The shared feature set in each sub-population can be used to characterize the sub-population and render the whole model more interpretable.}


\subsection{Optimization}\label{subsec:optimization}

We propose a generalized EM (GEM) algorithm~\citep{dempster1977maximum} to solve the minimization problem in (\ref{eq:l1}). For each $i=1,\ldots,n$, define $(\delta_{i,1},\ldots,\delta_{i,k})$ be a set of latent indicator variables, where $\delta_{i,r} = 1$ if the $i$th sample $(\y_i,\x_i)$ belongs to the $r$th component of the mixture model~\eqref{eq:exp_fm} and $\delta_{i,r} = 0$ otherwise. So $\sum_{r=1}^k\delta_{i,r}=1$, $\forall i$. These indicators are not observed since the cluster labels of the samples are unknown. Let $\delta$ denote the collection of all the indicator variables. By treating $\delta$ as missing, the EM algorithm proceeds by iteratively optimizing the conditional expectation of the complete log-likelihood criterion.

The complete log-likelihood is given by
\begin{align*}
\ell_c(\theta\mid \Y,\X,\delta)   = & \sum_{r=1}^k\biggl\{\sum_{i=1}^n\sum_{j\in\Omega_i}
\delta_{i,r} \left( \frac{y_{ij}\varphi_{ijr}-b_j(\varphi_{ijr})}{a_j(\phi_{jr})} +c_j(y_{ij},\phi_{jr}) \right)\\
 &+  \sum_{i=1}^n\delta_{i,r}\log(\pi_r)\biggr\},
\end{align*}
where $\varphi_{ijr}  = \x_i\bbeta_{jr}$, for $i = 1,\ldots,n,$ $j=1,\ldots,m$, and $r= 1,\ldots,k$. The conditional expectation of the penalized complete negative log-likelihood is then given by
\begin{align*}
Q_{pen}(\theta\mid \theta') =  - \mathbb{E}[\ell_c(\theta\mid \Y,\X,\delta)|\Y,\X,\theta']/n
+ \mathcal{R}(\bbeta;\lambda,\pi),
\end{align*}
where $\mathcal{R}(\bbeta;\lambda,\pi)$ can be any of the penalties in~\eqref{eq:pen_lasso} or~\eqref{eq:pen_lasso_group}. It is easy to show that deriving $Q_{pen}(\theta\mid \theta')$ boils down to the computation of $\mathbb{E}[\delta_{i,r}\mid\Y,\X,\theta']$, which admits an explicit form.

The EM algorithm proceeds as follows. Let $\theta^{(0)}$ be some given initial values. We repeat the following steps for $t=0,1,2,\ldots$, until convergence of the parameters or the pre-specified maximum number of iteration $T_{out}$ is reached.\\

\noindent\textbf{E-Step:} Compute $\hat{\rho}_{i,r}^{(t+1)}=\mathbb{E}[\delta_{i,r}\mid\Y,\X,\theta^{(t)}]$. For $\varphi_{ijr}^{(t)} = \x_i\bbeta_{jr}^{(t)}$,
\begin{align}\label{eq:cond_prob_t}
\hat{\rho}_{i,r}^{(t+1)} &  = \frac{ \pi_r^{(t)}\prod_{j\in\Omega_i}\exp\{(y_{ij}\varphi_{ijr}^{(t)}-b_j(\varphi_{ijr}^{(t)}))/a_j(\phi_{jr}^{(t)}) +c_j(y_{ij},\phi_{jr}^{(t)})\}}
{\sum_{r'=1}^k \pi_{r'}^{(t)}\prod_{j\in\Omega_i}\exp\{(y_{ij}g_{ijr'}^{(t)}-b_j(g_{ijr'}^{(t)}))/a_j(\phi_{jr'}^{(t)}) +c_j(y_{ij},\phi_{jr'}^{(t)})\}}.
\end{align}

\noindent\textbf{M-Step:} Minimize $Q_{pen}(\theta\mid\theta^{(t)})$.

a) Update $\pi = (\pi_1,\ldots,\pi_k)$ by solving
\begin{align*}
&\pi^{(t+1)} = \arg \min_{\pi} \ -\frac{1}{n}  \sum_{r=1}^k\sum_{i=1}^n\hat{\rho}^{(t+1)}_{i,r} \log(\pi_r)
  + \mathcal{R}(\bbeta^{(t)};\lambda,\pi) \\
  &s.t. \ \sum_{r=1}^k\pi_r = 1, \pi_r > 0, \forall r.
\end{align*}

b) Update $\bbeta,\PPhi$.
\begin{align*}
(\bbeta^{(t+1)},\PPhi^{(t+1)}) = \arg \min_{\bbeta,\PPhi} \ &-\frac{1}{n}\sum_{r=1}^k\sum_{i=1}^n\hat{\rho}^{(t+1)}_{i,r}\sum_{j\in\Omega_i}
 \biggl( \frac{y_{ij}\x_i\bbeta_{jr} -b_j(\x_i\bbeta_{jr} )}{a_j(\phi_{jr})}\\
 & +c_j(y_{ij},\phi_{jr}) \biggr )+ \mathcal{R}(\bbeta;\lambda,\pi^{(t+1)}). \nonumber
\end{align*}

For the problem in a),~\citet{stadler2010} proposed a procedure to lower the objective function by a feasible point, and we find that simply setting $\pi_r^{(t+1)} = \sum_{i=1}^n\hat{\rho}_{i,r}/n$ is good enough. For the problem in b), we use an accelerated proximal gradient (APG) method introduced in~\cite{nesterov2007gradient} with the maximum number of iteration of $T_{in}$. The update steps by proximal operators correspond to the chosen penalty form. For the entry-wise $\ell_1$ norm penalty in~\eqref{eq:pen_lasso},
\begin{align}\label{eq:proximal_lasso}
\widehat{\bbeta}_r^{(t+1)} = \mbox{sign}(\widetilde{\bbeta}_r^{(t)}) \circ \max\{0, |\widetilde{\bbeta}_r^{(t)}|-\tau\lambda(\pi_r^{(t+1)})^{\gamma}\},
\end{align}
where $\circ$ denotes entry-wise product, $\widetilde{\bbeta}_r^{(t)} = \bbeta_r^{(t)} + \tau\triangle\bbeta_r^{(t)}$, $\tau$ denotes the step size, and $\triangle\bbeta_r^{(t)}$ denotes the update direction of $\bbeta_r^{(t)}$ determined by APG. For the group $\ell_1$ norm penalty in~\eqref{eq:pen_lasso_group},
$$\widehat{\bbeta}_{r,j}^{(t+1)} = \widetilde{\bbeta}_{r,j}^{(t)} \circ \max\{0, 1 -\tau\lambda(\pi_r^{(t+1)})^{\gamma}/\|\widetilde{\bbeta}_{r,j}^{(t)}\|_2\},
$$
where $\bbeta_{r,j}$ denotes the $j$th column of $\bbeta_r$. We  adopt the active set algorithm in~\cite{stadler2010} to speed up the computation.

The time complexity of our algorithm using the speed up technique is $O(T_{out}kmnsT_{in})$ with $s$ being the number of non-zero parameters. The algorithm performs well in practice, and we have not observed any convergence issues in our extensive numerical studies.
\subsection{Clustering of Samples \& Imputation of Missing Targets}\label{subsec:FMMOGLR_miss}



From the model estimation, we can get estimates of both the conditional probabilities $p(\delta_{i,r} = 1 \mid y_{ij'},j' \in \Omega_i,\x_i,\theta)$ and the conditional means $\mathbb{E}[Y_{ij} \mid \x_i ,\theta,\delta_{i,r} = 1]$, where $ \mathbb{E}[Y_{ij} \mid \x_i ,\theta,\delta_{i,r} = 1 ] = \mu_j( \varphi_{ijr}) = b'_j(\x_i\bbeta_{jr})$. Specifically, the conditional probabilities can be estimated by $p(\delta_{i,r} = 1\mid   y_{ij'},j' \in \Omega_i ,\x_i , \hat{\theta})$ which corresponds to~\eqref{eq:cond_prob_t}, taking $t = T_{out}$. The conditional expectations can be estimated as $\mathbb{E}[Y_{ij} \mid \x_i ,\hat{\theta},\delta_{i,r} = 1] =  b'_j(\x_i\hat{\bbeta}_{jr})$.

For clustering the samples, we adopt the Bayes rule, i.e., for $i=1,\ldots, n$,
\begin{equation}\label{eq:hat_r_i}
\hat{r}_i =  \mathop{\argmax}_{r}  \ p(\delta_{i,r} = 1\mid   y_{ij'},j' \in \Omega_i ,\x_i , \hat{\theta}).
\end{equation}


Following the idea of~\citet{jacobs1991adaptive}, we propose to make imputation for the missing outcomes by
\begin{equation}\label{eq:ensem2}
\hat{y}_{ij}  = \sum_{r=1}^{k}p(\delta_{i,r} = 1\mid   y_{ij'},j' \in \Omega_i ,\x_i , \hat{\theta}) \mathbb{E}[Y_j\mid \x_i ,\hat{\theta},\delta_{i,r} = 1 ],\ \mbox{for} \ j \notin \Omega_i.
\end{equation}


{
\subsection{Tuning Hyper-Parameters}\label{subsec:tune_hyperparam}

Unless otherwise specified, all the hyper-parameters, including regularization coefficients $\lambda$s and the number of clusters $k$, are tuned to maximize the data log-likelihood in~\eqref{eq:loglike} on the held-out validation data set. In other words, we fit models on training data with different specific hyper-parameter settings, and then the optimal model is chosen as the one that gives the highest log-likelihood in~\eqref{eq:loglike} of the held-out validation data set. This approach is fairly standard and has been widely used in existing works~\citep{stadler2010}. Moreover, cross validation and various information criteria~\citep{bhat2010derivation,aho2014model} can also be applied to determine hyper-parameters.

}

{
\section{Extensions}\label{sec:extension}
We provide several extensions of the proposed~\methodname$\,$ approach described in Section~\ref{sec:method}, including robust estimation against outlier samples, handling anomaly tasks or clustered structure among tasks, and modeling mixture probabilities for feature-based prediction.

}

\subsection{Robust Estimation}\label{subsec:robust_estimation}

To perform robust estimation for parameters in the presence of outlier samples, we propose to adopt the mean shift penalization approach~\citep{she2011outlier}. Specifically, we extend the natural parameter model to the following additive form,
\begin{equation}\label{eq:outlier_g}
\varphi_{ijr} = \x_i\bbeta_{jr}+\zeta_{ijr},
\end{equation}
where $\zeta_{ijr}$ is a case-specific mean shift parameter to capture the potential deviation from the linear model. Apparently, when $\zeta_{ijr}$ is allowed to vary without any constraint, it can make the model fit as perfect as possible for every $y_{ijr}$. The merit of this approach is realized by assuming certain sparsity structure of the $\zeta_{ijr}$s, so that only a few of them have nonzero values corresponding to anomalies. Write $\zzeta_r \in \mathbb{R}^{n\times m} = [\zzeta_{1r},\zzeta_{2r},\ldots,\zzeta_{mr}]$ for $r=1,\ldots,k$, and $\zzeta\in\mathbb{R}^{(n\times m)\times k} =[\zzeta_1,\ldots, \zzeta_r]$. We can then conduct joint model estimation and outlier detection by extending~\eqref{eq:l1} to
\begin{equation}\label{eq:estimator_outlier}
(\hat{\theta},\hat{\zzeta}) =  {\arg\min}_{\theta\in\Theta, \zzeta} \  -  \ell(\theta\mid\Y,\X)/n
+ \mathcal{R}(\bbeta;\lambda_1) +  \mathcal{R}(\zzeta;\lambda_2),
\end{equation}
where, for example, the penalty on $\zzeta$ can be chosen as the group $\ell_1$ penalty,
\begin{align}\label{eq:pen_outlier_group}
\mathcal{R}(\zzeta;\lambda_2) =  \lambda_2\sum_{i=1}^n\sqrt{\sum_{jr}\zeta_{ijr}^2},
\end{align}
{so that entries of $\zzeta$ are nonzero for only a few data samples.}

The proposed GEM algorithm can be readily extended to handle the inclusion of $\zzeta$, for which we omit the details.

\subsection{Handling Anomaly Tasks}\label{subsec:anomaly_tasks}

{ Besides outlier samples, certain tasks, referred to as anomaly tasks, may not follow the assumed shared structure and thus can ruin the overall model performance.}
To handle anomaly tasks, though it is also intuitive to adopt the approach above, our numerical study suggests that its performance is sensitive to the tuning parameters. Here, we adopt the { idea of~\citet{koller1996toward} }, by utilizing the estimated conditional probabilities to measure how well a task is concordant with the estimated mixture structure. Consider the $h$th task. The main idea is to measure the discrepancy between $p(\delta_{ir} = 1\mid y_{ij}, j \in \Omega_i, \x_i,\theta)$, the conditional probability based on data from all observed targets, and $p(\delta_{ir} = 1 \mid y_{ih},\x_i,\theta)$, the conditional probability based on only the $h$th task. If $h$th task is an anomaly task, it is expected that the two conditional probabilities would differ more from each other~\citep{koller1996toward,law2002feature}.

For $r= 1,\ldots,k$, $i = 1,\ldots,n$, let
\begin{align}\label{eq:outlier_post}
P_{\Omega,ir} & = p(\delta_{ir} = 1\mid y_{ij}, j \in \Omega_i,\x_i, \hat{\theta}) = \frac{\hat{\pi}_r\prod_{j\in\Omega_i}f(y_{ij}\mid \hat{\varphi}_{ijr},\hat{\phi}_{jr})}{\sum_{r'=1}^k\hat{\pi}_{r'}\prod_{j\in\Omega_i}f(y_{ij}\mid \hat{\varphi}_{ijr'},\hat{\phi}_{jr'})}, \\
P_{h,ir} &= p(\delta_{ir} = 1\mid y_{ih},\x_i,\hat{\theta}) = \frac{\hat{\pi}_rf(y_{ih}\mid \hat{\varphi}_{ihr},\hat{\phi}_{hr})}{\sum_{r'=1}^k\hat{\pi}_{r'}f(y_{ih}\mid \hat{\varphi}_{ihr'},\hat{\phi}_{hr'})}. \nonumber
\end{align}
Define $\P_{\Omega} = [P_{\Omega,ir}]_{n\times k}$ and $\P_h = [P_{h,ir}]_{n\times k}$. Then we define the concordant score of the $h$th task as
\begin{equation}\label{eq:outlier_score}
\begin{split}
score(h) =   - (D_{KL}(\P_{\Omega} \parallel \P_h) + D_{KL}(\P_h \parallel \P_{\Omega}))/{ (2n)}, \ h=1,\ldots,m,
\end{split}
\end{equation}
where $D_{KL}$ is the widely used Kullback-Leibler divergence~\citep{cover2012elements}.


The tasks can then be ranked based on their concordant scores. As such, the detection of anomaly tasks boils down to a one-dimensional outlier detection problem. { After anomaly tasks are detected, their FMR models can be built.}

\subsection{Handling Clustered Structure among tasks}\label{subsec:clustered_tasks}

{ In practice, tasks may be clustered into groups such that each task group has its own model structure. Here we assume that each cluster of tasks shares a FMR structure defined in~\eqref{eq:exp_fm}, and propose to construct a similarity matrix to discover the potential cluster pattern among tasks.

We consider a two-stage strategy. First, each task learns a FMR model on the training data independently with the same pre-fixed $k$. Then we get $\P_h = [P_{h,ir}]_{n\times k}$ for all $h=1,\ldots,m$, where
\begin{equation*}
  P_{h,ir} = p(\delta_{h,ir} = 1\mid y_{ih},\x_i,\hat{\theta}_h) = \frac{\hat{\pi}_{hr}f(y_{ih}\mid \hat{\varphi}_{ihr},\hat{\phi}_{hr})}{\sum_{r'=1}^k\hat{\pi}_{hr'}f(y_{ih}\mid \hat{\varphi}_{ihr'},\hat{\phi}_{hr'})},
\end{equation*}
and $\delta_{h,ir}(i=1,\dots,n,r = 1,\dots,k)$ and $\hat{\pi}_{hr}(r = 1,\dots,k)$ are the latent variables and the estimated prior probabilities of the $h$th task, respectively.


Second, we adopt Normalized Mutual Information (NMI)~\citep{Strehl2003Cluster} to measure the similarity between each pair of tasks. We choose NMI instead of Kullback-Leibler divergence because NMI can handle the case that the orders of clusters of two $k$-cluster structures are different. Specifically, given two methods to estimate latent variables, which are denoted by method $u$ and method $v$, let $\P_u = [P_{u,ir}]_{n\times k}$ and $\P_v = [P_{v,ir}]_{n\times k}$ denote the estimated probability of latent variables of method $u$ and method $v$, respectively, where $P_{u,ir} = p(\delta_{u,ir} = 1), P_{v,ir} = p(\delta_{v,ir} = 1)$ for $i = 1,\ldots,n,r=1,\ldots,k$. NMI is defined as
\begin{equation}\label{eq:NMI}
NMI(\P_u ,\P_v ) = \frac{I(\P_u,\P_v)}{\sqrt{I(\P_u,\P_u)I(\P_v,\P_v)}}, \ u=1,\ldots,m,v=1,\ldots,m,
\end{equation}
where $I(\P_u,\P_v)$ denotes the mutual information between $\P_u,\P_v$ such that
\begin{equation*}\label{eq:MI}
I(\P_u,\P_v) =  \sum_{a=1}^k \sum_{b=1}^k p(\delta_{u,a} = 1, \delta_{v,b} = 1)\log\left(\frac{p(\delta_{u,a} = 1,\delta_{v,b} = 1)}{p(\delta_{u,a} = 1)p(\delta_{v,b} = 1)}\right).
\end{equation*}

Following~\citet{Strehl2003Cluster}, we approximate $p(\delta_{u,a} = 1,\delta_{v,b} = 1)$, $p(\delta_{u,a} = 1)$ and $p(\delta_{v,b} = 1)$ by
\begin{equation*}\label{eq:prob_in_MI}
\begin{split}
p(\delta_{u,a} = 1) &\approx \frac{1}{n}\sum_{i=1}^np(\delta_{u,ia} = 1) = \frac{1}{n}\sum_{i=1}^nP_{u,ia},\\
p(\delta_{v,b} = 1) &\approx \frac{1}{n}\sum_{i=1}^np(\delta_{v,ib} = 1) = \frac{1}{n}\sum_{i=1}^nP_{v,ib},\\
p(\delta_{u,a} = 1, \delta_{v,b} = 1) &\approx \frac{1}{n}\sum_{i=1}^np(\delta_{u,ia} = 1, \delta_{v,ib} = 1)\\
&\approx \frac{1}{n}\sum_{i=1}^np(\delta_{u,ia} = 1)p(\delta_{v,ib} = 1)
=\frac{1}{n}\sum_{i=1}^nP_{u,ia}P_{v,ib}.
\end{split}
\end{equation*}

As such, given the estimated models $\hat{\theta}_u,\hat{\theta}_v$ for the $u$th and the $v$th task, respectively, we treat $p(\delta_{u,ir} = 1\mid y_{iu},\x_i,\hat{\theta}_u)$ and $p(\delta_{v,ir} = 1\mid y_{iv},\x_i,\hat{\theta}_v)$ as $p(\delta_{u,ir} = 1)$ and $p(\delta_{v,ir} = 1)$, respectively, for $i = 1,\ldots,n,r=1,\ldots,k$. Then NMI between each pair of tasks are computed by~\eqref{eq:NMI}. We note that for simplicity we set the pre-fixed $k$ to be the same, but in general $k$ can be different for different tasks by the definition of Mutual Information. 

Given the similarity between each pair of tasks, any similarity-based clustering method can be applied to cluster $m$ tasks into groups. Empirically, the performance of task clustering is not sensitive to the pre-fixed $k$. As such, we set the pre-fixed $k$ to be $20$ in this paper. We then apply the proposed~\methodname$\,$ approach separately for each task group.


}

{

\subsection{Modeling Mixture Probabilities}\label{subsec:predition}


In real-applications, one may require to use $\x_i$ only to infer the latent variables $\delta_{i,r}$ and then to predict $\y_i$, for $i=1,\ldots,n,r=1,\ldots,k$. Here we further extend our method following the idea of Mixture-Of-Experts (MOE)~\citep{yuksel2012twenty} model; the only modification is that $\pi_r$ in~\eqref{eq:exp_fm} is assumed to be function of $\x_i$, for $i=1,\ldots,n$.


To be specific, let $\aalpha = [\aalpha_1,\ldots,\aalpha_k] \in \mathbb{R}^{d\times k}$ collect regression coefficient vectors for a multinomial linear model.
We assume that given $\x_i$, the joint probability density function of $\oby_i = \{y_{ij};j\in \Omega_i\}$ in~\eqref{eq:exp_fm} is replaced by
\begin{equation*}
\begin{split}
f(\oby_i \mid \x_i,\theta,\aalpha) &= f(y_{ij}, j\in \Omega_i \mid \x_i,\theta,\aalpha) \\
&= \sum_{r=1}^kp(\delta_{i,r} = 1\mid \x_i,\alpha_r ) \prod_{j\in\Omega_i}f(y_{ij}\mid \varphi_{ijr},\phi_{jr}),
\end{split}
\end{equation*}
where
\begin{equation}\label{eq:gate_prob}
p(\delta_{i,r} = 1\mid \x_i,\alpha_r ) = \frac{\exp(\x_i\aalpha_r) }{\sum_{r'=1}^{k}\exp(\x_i\aalpha_{r'})}
 \end{equation}
is referred to as the gating probability. All the other terms are defined the same as in~\eqref{eq:exp_fm}. 

Let $\theta_2 = \{ \bbeta,\phi_1,\ldots,\phi_k\}$, with the parameter space
$
\Theta_2 = \mathbb{R}^{ (d \times m) \times k}  \times \mathbb{R}^{m \times k}_{>0}.
$
The data log-likelihood of the MOE model is
\begin{equation}\label{eq:loglike_MOE}
\begin{split}
\ell(\theta_2,\aalpha\mid\Y,\X)  = \sum_{i=1}^n\log\biggl(&\sum_{r=1}^kp(\delta_{i,r} = 1\mid \x_i,\alpha_r ) \\ &\prod_{j\in\Omega_i}\exp\biggl\{\frac{y_{ij}\varphi_{ijr}-b_j(\varphi_{ijr})}{a_j(\phi_{jr})}+c_j(y_{ij},\phi_{jr})\biggr\}\biggr).
\end{split}
\end{equation}
The model estimation is conducted by extending~\eqref{eq:l1} to
\begin{equation}\label{eq:estimator_MOE}
(\hat{\theta}_2,\hat{\aalpha}) =  {\arg\min}_{\theta_2\in\Theta_2, \aalpha} \  - \ell(\theta_2,\aalpha\mid\Y,\X)
+ \mathcal{R}(\bbeta;\lambda_1)/n +  \mathcal{R}(\aalpha;\lambda_2),
\end{equation}
where, for example, the penalty on $\aalpha$ can be chosen as the lasso type penalty,
\begin{align}\label{eq:pen_MOE}
\mathcal{R}(\aalpha;\lambda_2) =  \lambda_2\|\aalpha\|_1.
\end{align}

%

The minimization problem in~\eqref{eq:estimator_MOE} is also solved by GEM, for which the optimization procedure is similar to that in Section~\ref{subsec:optimization}. The differences occur at the E-step:
\begin{equation}\label{eq:cond_prob_t_MOE}
\hat{\rho}_{i,r}^{(t+1)} = \frac{ p(\delta_{i,r} = 1\mid \x_i,\alpha_{r}^{(t)} )\prod_{j\in\Omega_i}f(y_{ij}\mid \varphi_{ijr}^{(t)},\phi_{jr}^{(t)})}
{\sum_{r'=1}^k p(\delta_{i,r'} = 1\mid \x_i,\alpha_{r'}^{(t)} )\prod_{j\in\Omega_i}f(y_{ij}\mid \varphi_{ijr'}^{(t)},\phi_{jr'}^{(t)})},
\end{equation}
where $f(y_{ij}\mid \varphi_{ijr}^{(t)},\phi_{jr}^{(t)}) = \exp\{(y_{ij}\varphi_{ijr}^{(t)}-b_j(\varphi_{ijr}^{(t)}))/a_j(\phi_{jr}^{(t)}) +c_j(y_{ij},\phi_{jr}^{(t)})\}$, at the optimization for $\aalpha$:
\begin{equation}\label{eq:update_alpha_MOE}
 \aalpha^{(t+1)} = \arg \min_{\aalpha} \ -\frac{1}{n}  \sum_{r=1}^k\sum_{i=1}^n\hat{\rho}^{(t+1)}_{i,r} \log p(\delta_{i,r} = 1\mid \x_i,\alpha_{r} )
  + \mathcal{R}(\aalpha^{(t)};\lambda_2),
\end{equation}
and at the computation of $\pi$: $\pi_r^{(t+1)} = \frac{1}{n}\sum_{i=1}^{n}p(\delta_{i,r} = 1\mid \x_i,\alpha_r^{(t+1)} )$ for $r=1,\ldots,k$. 


}

\section{Theoretical Analysis}

We study the estimation and variable selection performance of~\methodname\, under the high-dimensional framework with $d\gg n$. Both $m$ and $k$, on the other hand, are considered as fixed. This is because usually, the number of interested targets and the number of desired clusters are not large in many real problems. Here we only present the setup and the main results on non-asymptotic oracle inequalities to bound the excess risk and false selection, leaving detailed derivations in the Appendix. Our results generalize~\citet{stadler2010} to cover mixture regression models with 1) multivariate, heterogeneous (mixed-type) and incomplete response and 2) shared feature grouping sparse structure. This is not trivial due to the non-convexity and the triple heterogeneity of the problem. It turns out that additional condition on the tail behaviors of the conditional density $f(\y\mid \x,\theta)$ is required. Fortunately, the required conditions are still satisfied by a broad range of distributions.

\subsection{Notations and Conditions on the Conditional Density}

We firstly introduce some notations. Denote the regression parameters that are subject to regularization by $\beta = \mbox{vec}(\bbeta_1,\ldots,\bbeta_k),\phi = \mbox{vec}(\Phi_1,\ldots,\Phi_k)$, where $\mbox{vec}(\cdot)$ is the vectorization operator. The other parameters in the mixture model are denoted by $\eta = \mbox{vec} (\log(\phi),\log(\pi))$, where $\log(\cdot)$ is entry-wisely applied. Denote the true parameter by $\theta_0 = ( \bbeta_0,\Phi_{0,1},\ldots,\Phi_{0,k},\pi_{0,1},\pi_{0,k-1})$ to be estimated under the FMR model defined in~\eqref{eq:exp_fm} and~\eqref{eq:natural_parameter}. In the sequel, we always use subscripts ``$0$'' to represent parameters or structures under the true model. To study sparsity recovery, denote the set of indices of non-zero entries of the true parameter by $S$. We use $\lesssim$ to indicate that the inequality holds up to some multiplicative numerical constants. {To focus on the main idea, we consider the case of $\gamma=0$ in the following analysis}.

We define average excess risk for fixed design points $\x_1,\ldots,\x_n$ based on Kullback-Leibler divergence as
\begin{align*}
\bar{\varepsilon}(\theta\mid \theta_0) &= \frac{1}{n}\sum_{i=1}^n \varepsilon(\theta\mid \x_i,\theta_0), \
\varepsilon(\theta\mid \x_i,\theta_0)\\
& = -  \int \log\left(\frac{f(\oby_i\mid \x_i,\theta)}{f(\oby_i \mid \x_i,\theta_0)}\right)f(\oby_i \mid \x_i,\theta_0)d\oby_i,
\end{align*}
where $f(\oby_i \mid \x_i,\theta)$ is defined in~\eqref{eq:exp_fm}.

To impose the conditions on $f(\oby_i \mid \x_i,\theta)$, denote $\psi_i = \mbox{vec}(\varphi_i,\eta)$, where $\varphi_i = \mbox{vec}(\{\varphi_{ijr};j\in\Omega_i,r = 1,\ldots,k\})$, and denote $ \psi = \mbox{vec}(\psi_1,\ldots,\psi_n)$. As such, we may write $f(\oby_i\mid \x_i,\theta) = f(\oby_i\mid\psi_i)$, $\ell(\theta\mid \oby_i,\x_i) = \log f(\oby_i\mid \x_i,\theta)= \ell(\psi_i\mid \oby_i)$, and $ \bar{\varepsilon}(\psi\mid\psi_{0}) = \frac{1}{n}\sum_{i=1}^{n}\varepsilon(\psi_i\mid\psi_{0,i}) = \bar{\varepsilon}(\theta\mid\theta_0)$.

Without loss of much generality, the model parameters are assumed to be in a bounded parameter space for a constant $K$:
\begin{equation}\label{eq:tTheta}
\begin{split}
\tilde{\Theta} \subset \{\theta; & \max_{i=1,\ldots,n}\|\varphi_i(\x_i, \bbeta)\|_{\infty}\leq K,
  \max_{j=1,\ldots,m} |\log a_j(\phi)|\leq K,\\
  &\max_{j=1,\ldots,m} \log |b'_j(\phi)|\leq K, \|\log\phi\|_{\infty}\leq K, \\
  &-K\leq \log\pi_1\leq 0, \ldots, -K \leq \log\pi_{k-1} \leq 0,\\
  & \sum_{r=1}^{k-1}\pi_r < 1, \pi_k = 1 - \sum_{r=1}^{k-1}\pi_r\}.
\end{split}
\end{equation}

We present the following conditions on $f(\oby_i\mid \psi_i)$.

\begin{mycon}\label{th:con_G1}
For some function $G_1(\cdot)\in \mathbb{R}$, for $i = 1,\ldots,n$,
\begin{align*}
\sup_{\theta\in\tilde{\Theta}}\|\frac{\partial \ell(\psi_i\mid \oby_i)}{\partial \psi_i}\|_{\infty}\leq G_1(\oby_i).
\end{align*}
\end{mycon}

\begin{mycon}\label{th:con_tail}
For a constant $c_1\geq 0$, and some constants $ c_2,c_3,c_4,c_5\geq 0$ depending $K$, and for $M>c_4$, we assume for $i = 1,\ldots,n$,
\begin{align*}
& \mathbb{E}[|G_1({\oby_i})|1\{|G_1(\oby_i)|>M\}] \leq \biggl[ c_3\biggl(\frac{M}{c_2}\biggr)^{c'}+ c_5 \biggr]\exp\biggl\{-\biggl(\frac{M}{c_2}\biggr)^{1/c_1}\biggr\},\\
& \mathbb{E}[|G_1(\oby_i)|^21\{|G_1(\oby_i)|>M\}] \leq  \biggl[ c_3\biggl(\frac{M}{c_2}\biggr)^{c'}+ c_5 \biggr]^2\exp\biggl\{-2\biggl(\frac{M}{c_2}\biggr)^{1/c_1}\biggr\},
\end{align*}
where $\oby_i = \{y_{ij}; j\in\Omega_i\}$, $c' = 2+3/c_1$ and $1\{\cdot\}$ denotes the indicator function. 
\end{mycon}

\begin{mycon}\label{th:con_fisher}
It holds that,
\begin{align*}
  \min_{i=1,\ldots,n}\Lambda_{\min}(I(\psi_{0,i}))>  {1}/{c_0} > 0,
\end{align*}
where $c_0$ is a constant, $\Lambda_{\min}^2(A)$ is the smallest eigenvalue of a symmetric, positive semi-definite matrix $A$ and for $i = 1,\ldots,n$, $I(\psi_{0,i})$ is the Fisher information matrix such that
\begin{align*}
 I(\psi_{0,i}) =  - \int \frac{\partial^2\ell(\psi_{0,i}\mid \oby_i)}{\partial \psi_{0,i} \partial \psi_{0,i}^T }
 f(\oby_i\mid \psi_{0,i}) d\oby_i.
\end{align*}
\end{mycon}

The first condition follows from~\citet{stadler2010}, which aims to bound $\partial \ell(\psi_i\mid \oby_i)/\partial \psi_i$ with known $\oby_i$, for $i=1,\ldots,n$.
The second condition is about the tail behaviors of $f(\oby_i \mid \x_i,\theta)$. The third condition depicts the local convexity of $\ell$ at the point $\theta_0$. Condition~\ref{th:con_G1} and~\ref{th:con_tail} can cover a broad range of distributions for $f$, including but not limited to mixture of sub-exponential distributions, such as our proposed~\methodname\, model with known dispersion parameters, c.f., Lemma~\ref{th:lem_subexp_tail}.

\begin{mylem}\label{th:lem_subexp_tail}
 Condition~\ref{th:con_G1} and~\ref{th:con_tail} hold for the heterogeneous mixture distribution $f(\oby_i\mid \x_i,\theta)$ defined in~\eqref{eq:exp_fm} with known dispersion parameters.

\end{mylem}

The following two quantities will be used.
\begin{equation}\label{eq:lambda_0}
\lambda_0 =  \sqrt{mk}   M_n\log n \sqrt{\frac{\log(d\vee n)}{n}}, \ M_n = c_2(\log n)^{c_1},
\end{equation}
where $  c_1,c_2$ are the same constants as in Condition~\ref{th:con_tail}. More specifically, we choose $c_1 = 1/2, 0, 1$ for Gaussian, Bernoulli and Poisson task, respectively.

\subsection{Results for Lasso-Type Estimator}

Consider first the penalized estimator defined in (\ref{eq:l1}) with the $\ell_1$ penalty in~\eqref{eq:pen_lasso}. Following~\citet{bickel2009simultaneous} and~\citet{stadler2010}, we impose the following restricted eigenvalue condition on the design.

\begin{mycon}\label{th:con_REC}
(Restricted eigenvalue condition). For all $ w \in \mathbb{R}^{dmk}$ satisfying $\|w_{S^c}\|_1 \leq 6\|w_S\|_1$, it holds that for some constant $\kappa\geq 1$,
\begin{align*}
\|w_S\|_2^2 \leq \kappa^2 \|\varphi\|_{Q_n}^2 = \frac{\kappa^2}{n}\sum_{i=1}^n\sum_{j\in\Omega_i}\sum_{r=1}^k (\x_i\bbeta_{jr})^2.
\end{align*}
\end{mycon}

\begin{mythe}\label{th:th_bound}
Consider the~\methodname\, model in (\ref{eq:exp_fm}) with {  $\theta_0 \in \tilde{\Theta}$ }, and consider the penalized estimator (\ref{eq:l1}) with the $\ell_1$ penalty in (\ref{eq:pen_lasso}). Assume Conditions 1-4 hold. Suppose $\sqrt{mk} \lesssim n/M_n$, and take $\lambda > 2T\lambda_0$ for some constant $T>1$. For some constant $c>0$ and large enough $n$, with probability
\begin{equation}\label{eq:prob_bound}
1 - c\exp\left(-\frac{\log^2n\log(d\vee n)}{c}\right) - \frac{1}{n},
\end{equation}
we have
\begin{equation}\label{eq:bound_low}
 \bar{\varepsilon}(\hat{\theta}\mid \theta_0) + 2(\lambda-T\lambda_0)  \|\hat{\beta}_{S^c}\|_1
 \leq 4(\lambda+T\lambda_0)^2\kappa^2 c_0^2s
\end{equation}
where $s$ is the number of non-zero parameters of $w_0$.
\end{mythe}
Theorem~\ref{th:th_bound} suggests that the average excess risk has a convergence rate of the order {$O(s \lambda_0^2) = O((\log n)^{2+2c_1}\log(d\vee n)smk/n)$}, by taking $\lambda = 2T\lambda_0$ and using $\lambda_0$ and $M_n$ as defined in (\ref{eq:lambda_0}). Also, the degree of false selection measured by $ \|\hat{\beta}_{S^c}\|_1$ converge to zero at rate {$O(s\lambda_0) =O(s\sqrt{(\log n)^{2+2c_1}\log(d\vee n)mk/n)} $}.

Similar to~\citet{stadler2010}, under weaker conditions without the restricted eigenvalue assumption on the design, we still achieve the consistency for the average excess risk.
\begin{mythe}\label{th:high_dim}
Consider the~\methodname\, model in (\ref{eq:exp_fm}) with $\theta_0 \in \tilde{\Theta}$, and consider the penalized estimator (\ref{eq:l1}) with the $\ell_1$ penalty in (\ref{eq:pen_lasso}). Assume Conditions 1-3 hold.
Suppose
\begin{align*}
 \|\beta_0\|_1   &= o(\sqrt{n/((\log n)^{2+2c_1}\log(d\vee n)mk)}),\\
 \sqrt{mk} &= o(\sqrt{n/((\log n)^{2+2c_1}\log(d\vee n))})
\end{align*}
as $n\rightarrow \infty$, and take $\lambda = C\sqrt{(\log n)^{2+2c_1}\log(d\vee n)mk/n}$ for some constant $C>0$ sufficiently large.
For some constant $c>0$ and large enough $n$, with the following probability
$
1 - c\exp\left(-\frac{ (\log n)^2\log(d\vee n)}{c}\right) - \frac{1}{n},
$
we have
$
\bar{\varepsilon}(\hat{\theta}\mid \theta_0) = o_P(1).
$
\end{mythe}

\subsection{Results for Group-Lasso Type Estimator}
Consider the following general form of the group $\ell_1$ penalty,
 \begin{equation}\label{eq:pen_lasso_group_general}
 \begin{split}
 \mathcal{R}(\bbeta) = \lambda\sum_{p=1}^P\|\bbeta_{\mathcal{G}_p}\|_F,
 \end{split}
 \end{equation}
 where $\mathcal{G}_1,\ldots,\mathcal{G}_P$ are index collections such that $ \mathcal{G}_p \bigcap \mathcal{G}_{p'} = \emptyset$ for $p\neq p'$ and $\bigcup_{p=1}^P\mathcal{G}_p = \bigcup_{l=1}^d \bigcup_{j=1}^m \bigcup_{r=1}^k (l,j,r) $ equals the universal set of indices of $\bbeta\in\mathbb{R}^{(d\times m)\times k}$, i.e., $\bbeta_{\mathcal{G}_p}$ is the $p$th group of $\bbeta$. $\|\cdot\|_F$ denotes the Frobenius norm and here for $p = 1,\ldots,P$, $\|\bbeta_{\mathcal{G}_p}\|_F = \sqrt{\sum_{(l,j,r)\in\mathcal{G}_p}w_{ljr}^2}$. This penalty form generalizes the row-wise group sparsity in~\eqref{eq:pen_lasso_group}. 

Denote
$
  \mathcal{I} = \{p: \bbeta_{0,\mathcal{G}_p} = \mathbf{0}\}$ and $ \mathcal{I}^c = \{p: \bbeta_{0,\mathcal{G}_p} \neq \mathbf{0}\},
$
where $ \bbeta_{0,\mathcal{G}_p}$ is the $p$th group of $\bbeta_{0}$. Now denote by $s$ the size of $\mathcal{I}$, with some abuse of notation. We impose the following group-version restricted eigenvalue condition.
 \begin{mycon}\label{th:con_REC_GS}
For all $ \bbeta \in \mathbb{R}^{(d\times m)\times k}$ satisfying
\begin{align*}
\sum_{p\in \mathcal{I}^c}\|\bbeta_{\mathcal{G}_p}\|_F \leq 6\sum_{p\in \mathcal{I}}\|\bbeta_{\mathcal{G}_p}\|_F,
\end{align*}
it holds that for some constant $\kappa\geq 1$,
\begin{align*}
\sum_{p\in \mathcal{I}}\|\bbeta_{\mathcal{G}_p}\|_F^2 \leq \kappa^2 \|\varphi\|_{Q_n}^2.
\end{align*}
\end{mycon}

\begin{mythe}\label{th:th_bound_GS}
Consider the~\methodname\, model in (\ref{eq:exp_fm}) with $\theta_0 \in \tilde{\Theta}$, and consider the penalized estimator (\ref{eq:l1}) with the group $\ell_1$ penalty in (\ref{eq:pen_lasso_group_general}).

\noindent (a) Assume conditions 1-3 and 5 hold. Suppose $\sqrt{mk} \lesssim n/M_n$, and take $\lambda > 2T\lambda_0$ for some constant $T>1$. For some constant $c>0$ and large enough $n$, with the following probability
$
1 - c\exp\left(-\frac{ (\log n)^2\log(d\vee n)}{c}\right) - \frac{1}{n},
$
 we have
\begin{align*}
 \bar{\varepsilon}(\hat{\theta}\mid\theta_0) + 2(\lambda-T\lambda_0) \sum_{p\in\mathcal{I}^c}\|\widehat{\bbeta}_{\mathcal{G}_p}\|_F
 \leq  4(\lambda+T\lambda_0)^2\kappa^2 c_0^2s.
\end{align*}

\noindent(b) Assume conditions 1-3 hold (without condition~\ref{th:con_REC_GS}), and assume
\begin{align*}
 \sum_{p=1}^P\|\bbeta_{0,\mathcal{G}_p}\|_F  & = o(\sqrt{n/((\log n)^{2+2c_1}\log(d\vee n)mk)}),\\
  \sqrt{mk} &= o(\sqrt{n/((\log n)^{2+2c_1}\log(d\vee n))})
\end{align*}
as $n\rightarrow \infty$. Let $\lambda = C\sqrt{(\log n)^{2+2c_1}\log(d\vee n)mk/n}$ for some $C>0$ sufficiently large. Then for some constant $c>0$ and large enough $n$, with the following probability
$
1 - c\exp\left(-\frac{ (\log n)^2\log(d\vee n)}{c}\right) - \frac{1}{n},
$
we have
$
\bar{\varepsilon}(\hat{\theta}\mid\theta_0) = o_P(1).
$
\end{mythe}

So the average excess risk has a convergence rate of $O(s\lambda_0^2)$, and the degree of false group selection, as measured by $ \sum_{p\in\mathcal{I}^c}\|\hat{\bbeta}_{\mathcal{G}_p}\|_F$, converges to zero at rate $O(s\lambda_0)$. The estimator in (\ref{eq:l1}) using other group $\ell_1$ penalties such as~\eqref{eq:pen_lasso_group} are special cases, so the results of Theorem~\ref{th:th_bound_GS} still apply.

\noindent\underline{\textbf{Remark}}. Our results can be extended to the mean-shifted natural parameter model as in~\eqref{eq:outlier_g}, with a modified restricted eigenvalue condition. See the Appendix for some details.

\section{Experiments}
\label{sec:exp}
In this section, we present empirical studies on both synthetic and real-world data sets.

\subsection{Methods for Comparison}\label{subsec:method_comp}

We evaluate the following versions of the proposed~\methodname\, approach.

(1) Single task learning (\textbf{Single}): It is a special case of the~\methodname\, estimator~\eqref{eq:l1} with~\eqref{eq:pen_lasso}, where each task is learned separately.

(2) Separately learning (\textbf{Sep}):  It is a special case of the~\methodname\, estimator~\eqref{eq:l1} with~\eqref{eq:pen_lasso}, where each type (Gaussian, Bernoulli or Poisson) of tasks is learned separately.

(3) Mixed learning with entry-wise sparsity (\textbf{Mix}): It is the proposed~\methodname\, estimator~\eqref{eq:l1} with~\eqref{eq:pen_lasso} where all the tasks are jointly learned. To compare with \textbf{Sep}, we allow different tuning parameters for different types of outcomes.

(4) Mixed learning with group sparsity (\textbf{Mix GS}): It is the proposed~\methodname\, estimator~\eqref{eq:l1} with~\eqref{eq:pen_lasso_group}.

{

(5) Mixed learning Mixture-Of-Experts model with entry-wise sparsity (\textbf{Mix MOE}): It is the proposed~\methodname\, estimator~\eqref{eq:estimator_MOE} with~\eqref{eq:pen_lasso} and~\eqref{eq:pen_MOE}.

(6) Mixed learning Mixture-Of-Experts model with group sparsity (\textbf{Mix MOE GS}): It is the proposed~\methodname\, estimator~\eqref{eq:estimator_MOE} with~\eqref{eq:pen_lasso_group} and~\eqref{eq:pen_MOE}.



}

Besides the above FMR methods, we also evaluate several non-FMR multi-task methods below for comparison, some of which handle certain kinds of heterogeneities, such as anomaly tasks, clustered tasks and heterogeneous responses. Since they are non-FMR, they learn a single regression coefficient matrix $\bbeta \in \mathbb{R}^{d\times m}$.


\begin{itemize}
  \item \textbf{LASSO}: $\ell_1$-norm multi-task regression with $\lambda\|\bbeta\|_1$ as penalty. Each type of tasks are learned independently. It is a special case of \textbf{Sep} when pre-fixed $\hat{k}=1$.
  \item \textbf{Sep L2}: ridge multi-task regression with $\lambda\|\bbeta\|_F^2$ as penalty. Each type of tasks are learned independently.
  \item \textbf{Group LASSO}: $\ell_{1,2}$-norm multi-task regression with $\lambda\|\bbeta\|_{1,2}$ as penalty~\citep{yang2009heterogeneous}, which handles heterogeneous responses, and is a special case of \textbf{Mix GS} when pre-fixed $\hat{k}=1$.
  \item \textbf{TraceReg}: trace-norm multi-task regression~\citep{ji2009accelerated}.
  \item \textbf{Dirty}: dirty model multi-task regression with $\lambda_1\|\S\|_{1} + \lambda_2\|\L\|_{1,\infty}(\bbeta = \L + \S)$ as penalty~\citep{jalali2010dirty}, handling entry-wise heterogeneity in $\bbeta$ comparing with \textbf{Group LASSO}.
  \item \textbf{MSMTFL}: multi-stage multi-task feature learning~\citep{gong2012multi} whose penalty is $\lambda_1\sum_{l=1}^{d}\min(\|\bbeta^l\|_1,\lambda_2)$, where $\bbeta^l$ denotes the $l$th row of $\bbeta$. It also handles entry-wise heterogeneity in $\bbeta$ comparing with \textbf{Group LASSO}.
  \item \textbf{SparseTrace}: multi-task regression, learning sparse and low-rank patterns with $\lambda_1\|\S\|_{1} + \lambda_2\|\L\|_*(\bbeta = \L + \S)$ as penalty~\citep{chen2012learning}, handling entry-wise heterogeneity in $\bbeta$ comparing with \textbf{TraceReg}, where $\|\cdot\|_{*}$ denotes the nuclear norm of the enclosed matrix.
  \item \textbf{rMTFL}: robust multi-task feature learning with $\lambda_1\|\S\|_{2,1} + \lambda_2\|\L\|_{1,2}(\bbeta = \L + \S)$ as penalty~\citep{gong2012robust}, handling anomaly tasks comparing with \textbf{Group LASSO}.
  \item \textbf{RMTL}: robust multi-task regression with $\lambda_1\|\S\|_{2,1} + \lambda_2\|\L\|_*(\bbeta = \L + \S)$ as penalty~\citep{chen2011integrating}, handling anomaly tasks comparing with \textbf{TraceReg}.
  \item \textbf{CMTL}: clustered multi-task learning~\citep{zhou2011clustered}, handling clustered tasks.
  \item \textbf{GO-MTL}: multi-task regression, handling overlapping clustered tasks~\citep{kumar2012learning}.
\end{itemize}

\subsection{Experimental Setting}
In our experiments, for the E-step of GEM, we follow~\citet{stadler2010} to initialize $\rho$. For the M-step, we initialize the entries of $\bbeta$ from $\mathcal{N}(0,10^{-10})$. We fix $\sigma = 1$ for Gaussian tasks, and set $\gamma=1$.  In the APG algorithm, step size is initialized by the Barzilai-Borwein rule~\citep{barzilai1988two} and updated by the TFOCS-style backtracking~\citep{becker2011templates}. 

We terminate the APG algorithm with maximum iteration step $T_{in} = 200$ or when the relative $\ell_2$-norm distance of two consecutive parameters is less than $10^{-6}$. We terminate the GEM with maximum iteration step $T_{out} = 50$, or when the relative change of two consecutive $-\ell(\theta\mid\Y,\X)/n$ is less than $10^{-6}$ or when the relative $\ell_{\infty}$-norm distance of two consecutive parameters is less than $10^{-3}$.

{


In the experiments on both simulated and real-world data sets, we partition the entire data set into three parts: a training set for model fitting, a validation set for tuning hyper-parameters and a testing set for testing the generalization performance of the selected models. The only exception is Section~\ref{subsubsec:sim1}, where we do not generate testing data sets because the models are evaluated by comparing the estimation results to the ground truth. 


In hyper-parameter tuning, the regularization parameters, i.e., $\lambda$s, are tuned from $[1e-6,1e3]$, and the number of clusters are tuned from $\{1,\ldots,10\}$. Hyper-parameters of the baseline methods are tuned according to the descriptions in their respective references.

All the experiments are replicated 100 times under each model setting.

}


\subsection{Evaluation Metrics}

The prediction of latent variable is evaluated by Normalized Mutual Information (NMI)~\citep{strehl2002cluster,fern2003random,Strehl2003Cluster}. {In detail, we compute NMI scores by~\eqref{eq:NMI}, treating estimated conditional probabilities $[\P_{\Omega,ir}]_{n\times k}$ defined in~\eqref{eq:outlier_post} and the ground truth latent variables $[\delta_{i,r}]_{n\times k}$ as $[P_{1,ir}]_{n\times k}$ and $[P_{2,ir}]_{n\times k}$, respectively.}

For feature selection, firstly, the estimated components are reordered to make the best match with the true components. Then feature selection is evaluated by Area Under the ROC Curve (AUC) which is measured by the Wilcoxon-Mann-Whitney statistic provided by~\citet{hanley1982meaning}. {Concretely, absolute values of vectorized estimated regression parameters, i.e., $\hat{\beta}_{\mbox{abs}} =|\mbox{vec}( \hat{\bbeta})|$, and binarized vectorized ground truth regression parameters, i.e., $\beta_{0,\mbox{sign}} = \mbox{sign}(|\mbox{vec}(\bbeta_0)|)$, are used as inputs to AUC, where $\bbeta_0$ denotes the ground truth regression parameter and $\mbox{vec}(\cdot)$  is the vectorization operator.}

{In order to show the existence of mixed relationships between features and targets, imputation performance for incomplete targets is used to compare FMR methods with non-FMR MTL methods.} Concretely, the goal is to predict one-half randomly chosen targets. The other half targets are allowed to be used. FMR methods use the other half targets to compute conditional probabilities $p(\delta_{i,r} = 1 \mid y_{ij'},j' \in \Omega_i,\x_i,\hat{\theta})(i=1,\ldots,n,r=1,\dots,k)$ and make prediction as stated in Section~\ref{subsec:FMMOGLR_miss}. Non-FMR MTL methods perform feature-based prediction.

{
Feature-based prediction performances are also compared between non-FMR MTL methods and our MOE methods, where only features are allowed to use to predict testing targets. For this case, the goal is to predict all the targets.
}

{For target prediction, Gaussian outcomes are evaluated by nMSE~\citep{chen2011integrating,gong2012robust} which is defined as the mean of each task's mean squared error (MSE) divided by the variance of its target vector. Bernoulli outcomes are evaluated by average AUC (aAUC), which is defined as the mean AUC of each task. For Poisson tasks, we firstly compute the logarithms of outcomes, then use nMSE for evaluation. }



{Since our objective functions in~\eqref{eq:l1} and~\eqref{eq:estimator_MOE} are non-convex, estimated parameters may correspond to local minimums of the objective functions. Therefore, we try different initializations and report the results ranking the best 20\% on the validation data set out of the 100 replications to avoid the results that may be stuck at local minimums, suggesting that one can always select any result within the best 20\%.}



\subsection{Simulation}


\subsubsection{Latent Variable Prediction and Feature Selection}\label{subsubsec:sim1}

We consider both low dimensional case and high dimensional case for latent variable prediction and feature selection.
For the low-dimensional case, we set the number of samples $n=100$, feature dimension $d=15$, number of non-zero features (sparsity) $s=3$, and the number of tasks (responses) $m=15$. The data set includes $3$ Gaussian tasks, $10$ Bernoulli tasks, and $2$ Poisson tasks. The number of latent components $k=2$. For $r=1,\dots,k$, in the $r$th component, the first row (biases) and the $(s(r-1)+2)$th to the $(sr+1)$th row (a block of $s$ rows) of the true $\bbeta_r\in \mathbb{R}^{d\times m}$ are non-zero (to let different components have different sets of features). Non-zero parameters in $\bbeta$ are in the range of $[-3,-1]\cup[1,3]$ except that those of Poisson tasks are in the range of $[-0.3,-0.1]\cup[0.1,0.3]$. The biases are all set to $1$ except that those of Poisson tasks are set to $3$. For Gaussian tasks, all $\sigma$s are set to $1$. The entries of $\X\in\mathbb{R}^{d\times n}$ are drawn from $\mathcal{N}(0,1)$ with the first dimension being $1$. $\pi = (0.5,0.5)$. Validation data is independently generated likewise and has $n$ samples.
For the high-dimensional case, we set $n=180$, $d=320$ and $m=20$. The data set includes $8$ Gaussian tasks, $10$ Bernoulli tasks, and $2$ Poisson tasks. Other settings are the same as in the low-dimensional case.
We set the pre-fixed $\hat{k}$ to be equal to the true $k=2$. For targets of training data, we have tried different missing rates, which are in the range of $\{0,0.05,0.1,0.15,0.2\}$. We compare the performances of $\hat{\theta}$s estimated by \textbf{Single}, \textbf{Sep}, \textbf{Mix}, \textbf{Mix GS}, respectively, with that of $\theta_0$ (denoted by ``\textbf{True}''). 

{
The results are shown in Fig~\ref{fig:theory}. The horizontal axis is the missing rates. Intuitively, larger missing rates may result in worse performances due to fewer data samples.
\textbf{Single} provides poor results and is not sensitive to missing rate, because (1) data samples are deficient for single-task learning and (2) the influence of missing rate may be not significant when the number of samples is at this level.
\textbf{Sep} outperforms \textbf{Single} and is affected significantly by missing rate, because (1) \textbf{Sep} uses the prior knowledge in data that multiple tasks share the same FMR structure and (2) \textbf{Sep} constructs separate FMR models such that tasks for each model are deficient, hence the advantage from joint learning multiple tasks can be easily affected when some targets are missing.
Our~\methodname\, method \textbf{Mix} outperforms \textbf{Sep} and is robust against growing missing rate, because (1) \textbf{Mix} uses the prior knowledge in data that all the tasks share the same FMR structure and (2) \textbf{Mix} takes advantage of all the tasks, therefore, the number of tasks is then enough even some targets are missing.
Our~\methodname\, method \textbf{Mix GS} outperforms \textbf{Mix}, even rivals the true model, and is also robust against growing missing rate, because (1) comparing with \textbf{Mix}, \textbf{Mix GS} further uses the prior knowledge in data that all the tasks share the same feature space in each cluster, and (2) \textbf{Mix GS} takes advantage of all the tasks as well.
}

%
%

\begin{figure}[htbp]\small
\centering
\subfigure[Latent variable prediction accuracy]{\includegraphics[width=2.3in]{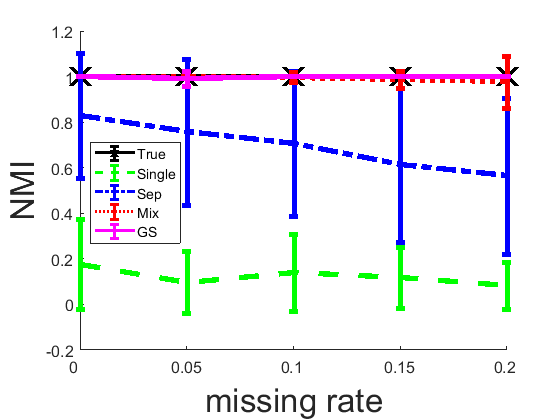}} 
\subfigure[Feature selection accuracy]{\includegraphics[width=2.3in]{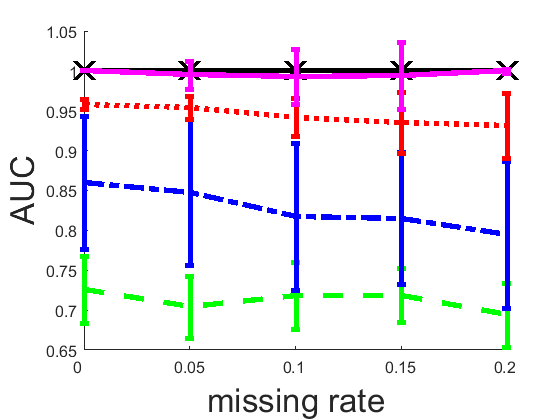}} 
\subfigure[Latent variable prediction accuracy]{\includegraphics[width=2.3in]{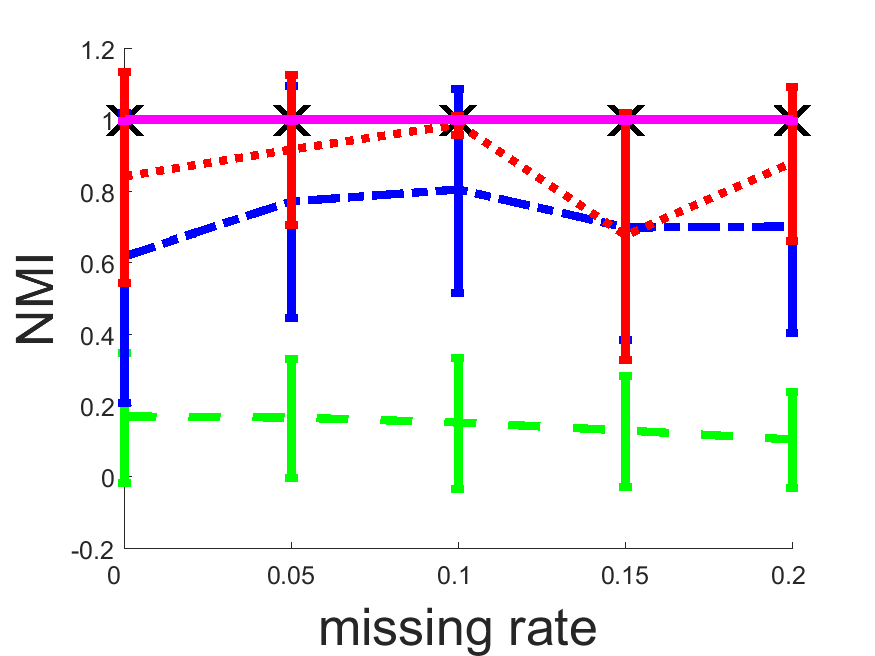}} 
\subfigure[Feature selection accuracy]{\includegraphics[width=2.3in]{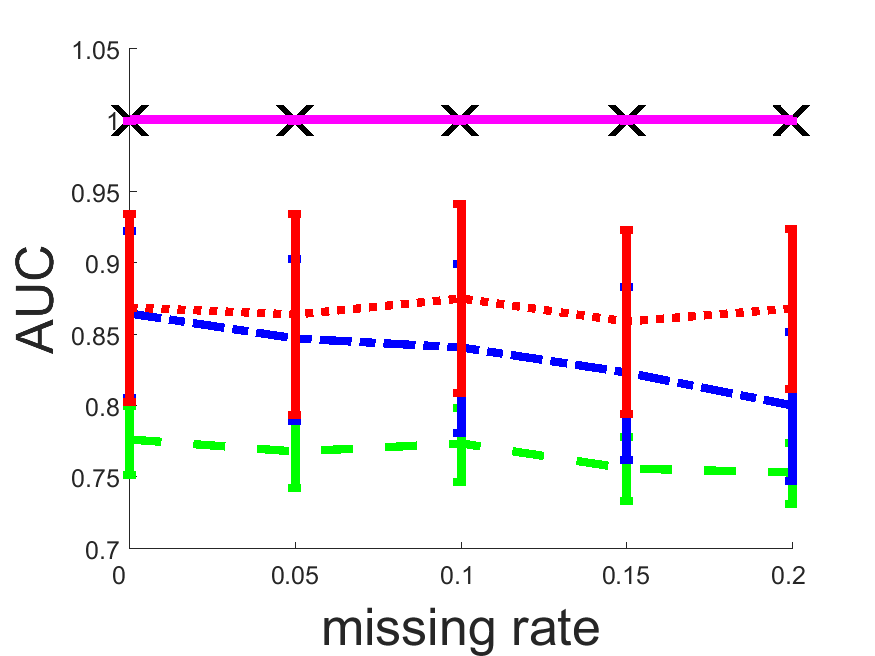}} 
\caption{Latent variable prediction and feature selection performance. (a) and (b) are results on low-dimensional data; (c) and (d) are results on high-dimensional data.}\label{fig:theory}
\end{figure}

\subsubsection{Performances When the Pre-Fixed $\hat{k}$ Is Different with the True $k$}\label{subsubsec:sim_diff_k}
We consider testing the performance of target imputation when the pre-fixed $\hat{k}$ is different with the true $k$.
Four data sets are generated with the true $k = 1,2,3,4$, respectively. We set $n=1000$, $d=32$, and $m=15$. There are $3$ Gaussian tasks, $10$ Bernoulli tasks, and $2$ Poisson tasks. For each $k = 1,2,3,4$, the sparsity $s$ is set to $\lfloor d/(2k)\rfloor$ such that the total numbers of relevant features for different data sets are the same. The values of the non-zero regression parameters for Gaussian and Bernoulli tasks in $\bbeta$ are in the range of $[-6,-2]\cup[2,6]$. We set $\pi_1 = \pi_2 = \ldots = \pi_k$. Validation and testing data are independently generated likewise and both have $n$ samples. We randomly set 20\% of targets to be missing for all the training, validation and testing data. Other settings are the same as in Section~\ref{subsubsec:sim1}.
One intuitive thought is that when the pre-fixed $\hat{k}$ equals the true $k$, the imputation performance will be maximized. So we set the pre-fixed $\hat{k} \in\{ 1,2,3,4,5,6,7,8\}$. We test \textbf{Mix} model in this experiment. Results by \textbf{Mix GS} model are similar.

In Fig~\ref{fig:diffk_simu}, the imputation performances are truly maximized when the pre-fixed $\hat{k}$ equals the true $k$. When pre-fixed $\hat{k}$ is larger than the true $k$, the imputation performances are similar. When the true $k>1$ and when the pre-fixed $\hat{k}$ is less than the true $k$, the imputation performances grow with the pre-fixed $\hat{k}$.
{One may expect that when the pre-fixed $\hat{k}$ is larger than the true $k$, the performances will deteriorate, since imputation would be based on fewer data samples. We think it is because (1) the simulated data are simple, and (2) the information sharing among tasks renders the robustness of our~\methodname\, method against decreasing sample size, which is consistent with the results in Section~\ref{subsubsec:sim1} when facing increasing missing rate (larger missing rate also indicates fewer data samples).}

\begin{figure}[htbp]\small
\centering
\subfigure[nMSE of Gaussian targets]{\includegraphics[width=2.3in]{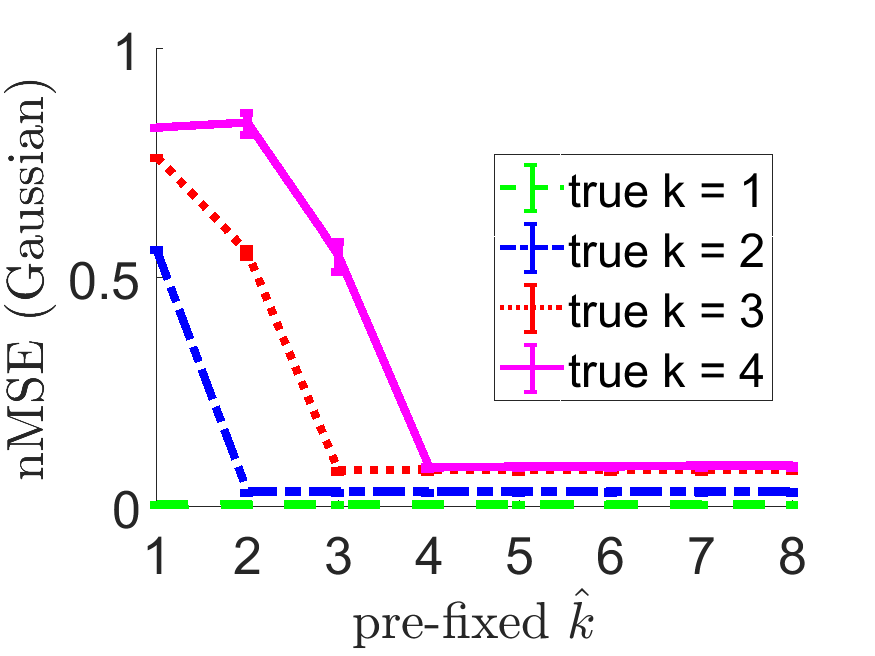}} 
\subfigure[aAUC of Bernoulli targets]{\includegraphics[width=2.3in]{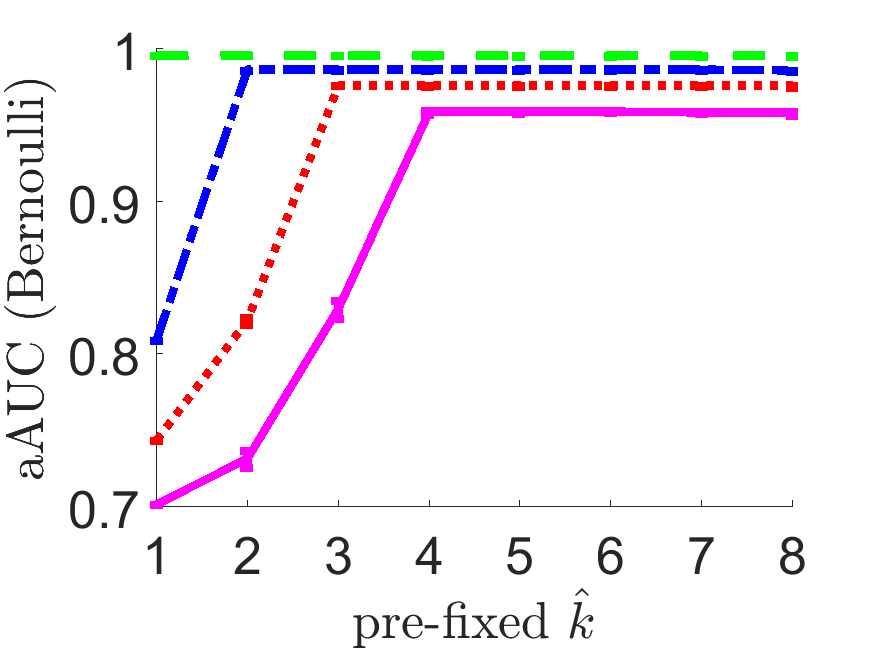}} 
\subfigure[nMSE of log of Poisson targets]{\includegraphics[width=2.3in]{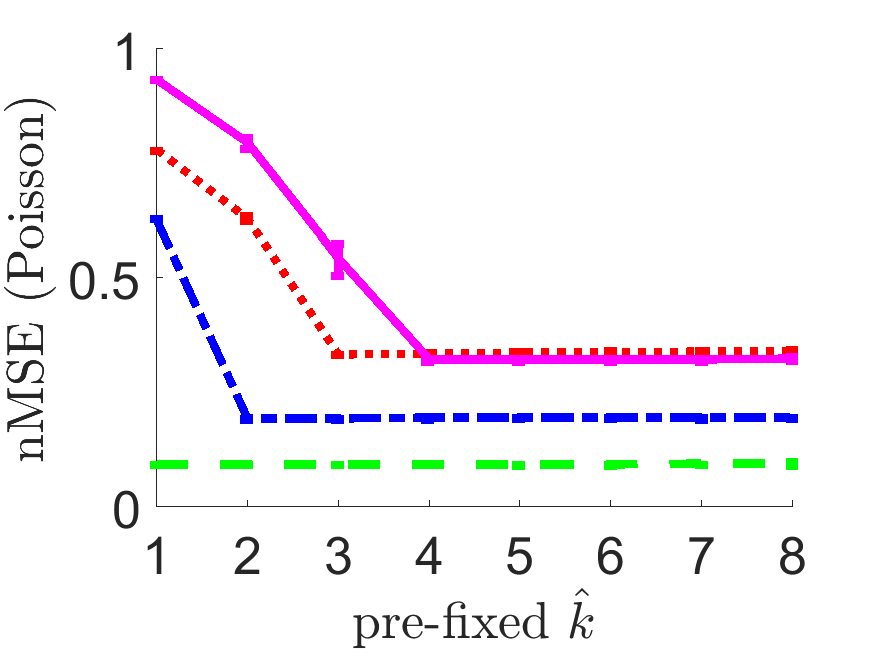}} 
\caption{Imputation performance when the pre-fixed $\hat{k}$ is different with the true $k$.}\label{fig:diffk_simu}
\end{figure}

\subsubsection{Comparison with Non-FMR Methods}\label{subsubsec:sim_non_FMR}

We compare the imputation performance of our~\methodname\, methods \textbf{Mix} and \textbf{Mix GS} with all the non-FMR methods. We choose the data set used in Section~\ref{subsubsec:sim_diff_k} with the true $k=3$. The Poisson targets are removed since many other methods are not able to handle them. {The tuned $\hat{k}=3$.}

In Table~\ref{tab:comp_base}, our~\methodname\, methods \textbf{Mix} and \textbf{Mix GS} not only outperform their special cases, i.e., \textbf{LASSO} and \textbf{Group LASSO}, respectively, but also outperform other multi-task learning methods, including those handling certain kinds of heterogeneities.

\begin{table}[htbp]\small
  \centering
  \renewcommand{\multirowsetup}{\centering}
  \begin{tabular}{lcc}
    \hline
    &nMSE &aAUC\\
    \hline
   LASSO & 0.6892 & 0.7384\\
    \hline
    Mix & \textbf{0.1181} & 0.9525\\
    \hline
    Group LASSO & 0.6850 & 0.7482\\
    \hline
    Mix GS & 0.1212 & \textbf{0.9559}\\
    \hline
    Sep L2 & 0.6912 & 0.7355\\
    \hline
    GO-MTL & 0.8055 & 0.7259\\
    \hline
    CMTL & 0.6916 & 0.7344\\
    \hline
    MSMTFL & 0.6890 & 0.7381\\
    \hline
    TraceReg & 0.6913 & 0.7362\\
    \hline
    SparseTrace & 0.6904 & 0.7374\\
    \hline
    RMTL & 0.6913 & 0.7362\\
    \hline
    Dirty & 0.6850 & 0.7482\\
    \hline
    rMTFL & 0.6850 & 0.7482\\
    \hline
  \end{tabular}
  \caption{Comparison with non-FMR methods}\label{tab:comp_base}
\end{table}

\subsubsection{Detection of Anomaly Tasks}\label{subsubsec:sim_anomaly_task}

We set $n=2000$. The number of tasks (responses) $m=30$. The information about the true $k$s and numbers of different types of tasks is in Table~\ref{tab:data4}. Other settings are the same as in Section~\ref{subsubsec:sim_diff_k}. In Table~\ref{tab:data4}, it can be seen that the true $k$ of the majority of tasks (the first 20 tasks) is 4.  The first 20 tasks are referred to as concordant tasks, while the other 10 tasks are referred to as anomaly tasks.

\begin{table}[htbp]\small
  \centering
 \begin{tabular}{ccccc}
   \hline
Group &   True k & \#Gaussian & \#Bernoulli & \#Poisson  \\
  \hline
 1&  4 & 5 & 10 & 5 \\
 2&  1 & 1 & 1 & 1 \\
 3&  6 & 1 & 0 & 0  \\
 4&  2 & 1 & 1 & 0 \\
 5&  3 & 0 & 1 & 1 \\
 6&  5 & 1 & 1 & 0 \\
   \hline
 \end{tabular}
  \caption{True $k$s and numbers of different types of tasks.}\label{tab:data4}
\end{table}

%
%
%
%

We compute the concordant scores using~\eqref{eq:outlier_score} for the tasks. In Fig~\ref{fig:outlier_score}, the concordant scores separate concordant tasks and anomaly tasks quite well. { Scores of Poisson tasks are similar to scores of Bernoulli tasks, because they all provide less accurate information than Gaussian tasks do.}

\begin{figure}[htbp]\small
\centering
\subfigure[Mix]{\includegraphics[width=2.3in]{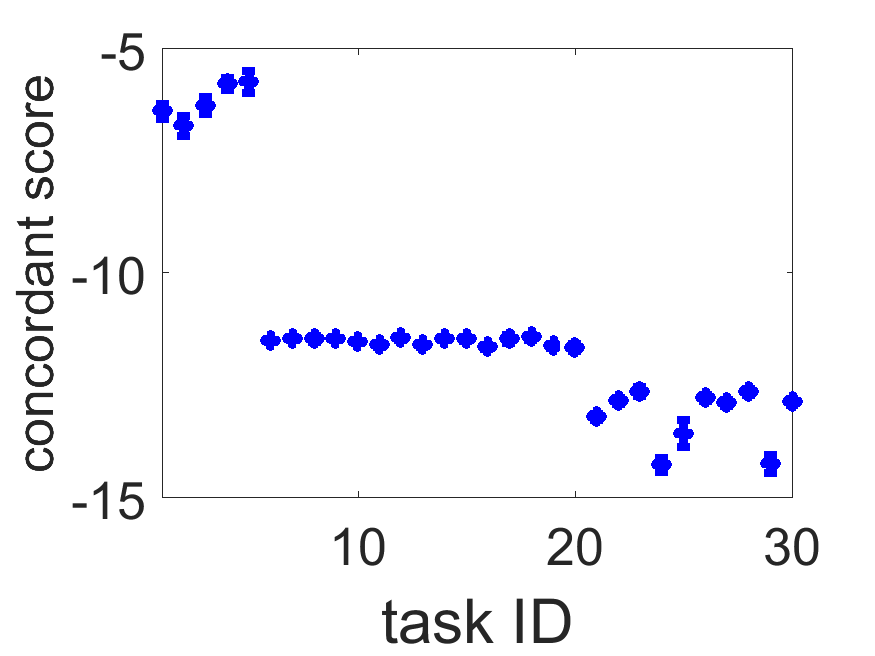}} 
\subfigure[Mix GS]{\includegraphics[width=2.3in]{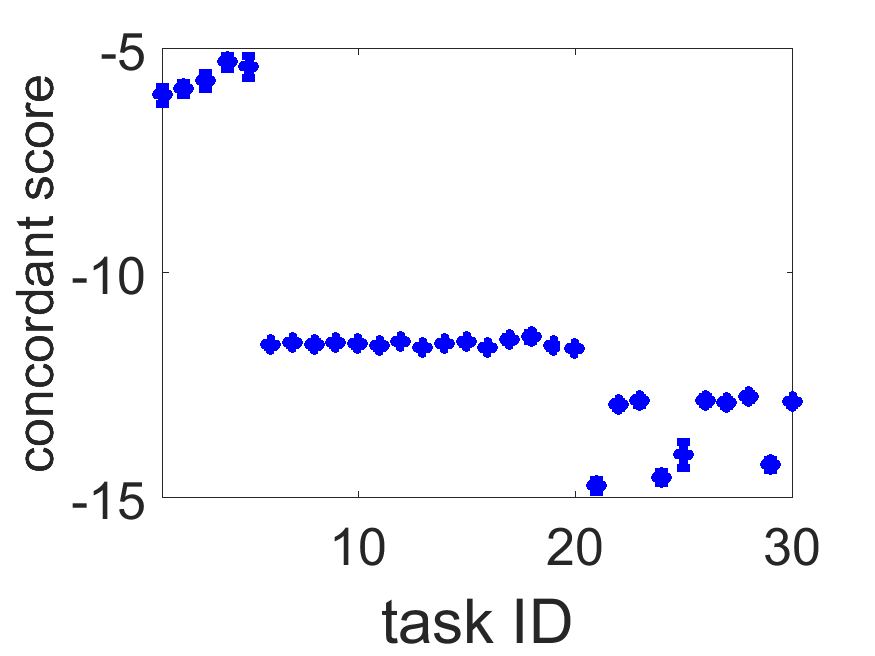}} 
\caption{Concordant scores of tasks{, which are associated with Table~\ref{tab:data4}}. (a) estimated by \textbf{Mix}; (b) estimated by \textbf{Mix GS}. The first 20 tasks are concordant tasks, the last 10 tasks are anomaly tasks. {The first 5 tasks are Gaussian tasks, the subsequent 10 tasks are Bernoulli tasks and then the subsequent 5 tasks are Poisson tasks.}}\label{fig:outlier_score}
\end{figure}

{

\subsubsection{Handling Clustered Relationship among tasks}\label{subsubsec:simu_task_cluster}

We construct 4 groups of tasks. The total number of tasks (responses) $m=60$. The information about the true $k$s and numbers of different types of tasks is in Table~\ref{tab:data6}. Other settings are the same as in Section~\ref{subsubsec:sim_anomaly_task}.
We first apply \textbf{Single} for each task, setting $\hat{k} = 20$. Then we apply the strategy in Section~\ref{subsec:clustered_tasks} to construct a similarity matrix by NMI defined in~\eqref{eq:NMI}. Kernel PCA~\citep{scholkopf1998nonlinear,van2009dimensionality} is then applied using the similarity matrix as the kernel matrix. The similarity matrix and the result of Kernel PCA are shown in Fig~\ref{fig:simu_task_cluster}.

\begin{table}[htbp]\small
  \centering
 \begin{tabular}{ccccc}
   \hline
  Group & True k & \#Gaussian & \#Bernoulli & \#Poisson  \\
  \hline
    1 &1 & 3 & 10 & 2 \\
    2 &2 & 3 & 10 & 2 \\
    3 &3 & 3 & 10 & 2 \\
    4 &4 & 3 & 10 & 2 \\
  \hline
 \end{tabular}
  \caption{True $k$s and numbers of different types of tasks. Tasks are clustered into 4 groups.}\label{tab:data6}
\end{table}

\begin{figure}[htbp]\small
\centering
\subfigure[Similarity Matrix]{\includegraphics[width=2.3in]{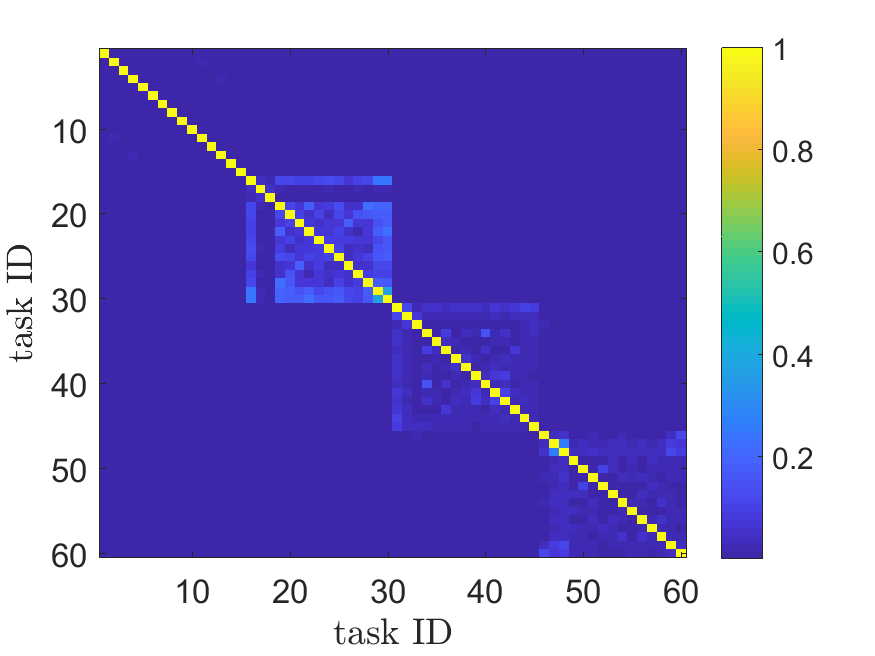}} 
\subfigure[Dimension Reduction by Kernel PCA]{\includegraphics[width=2.3in]{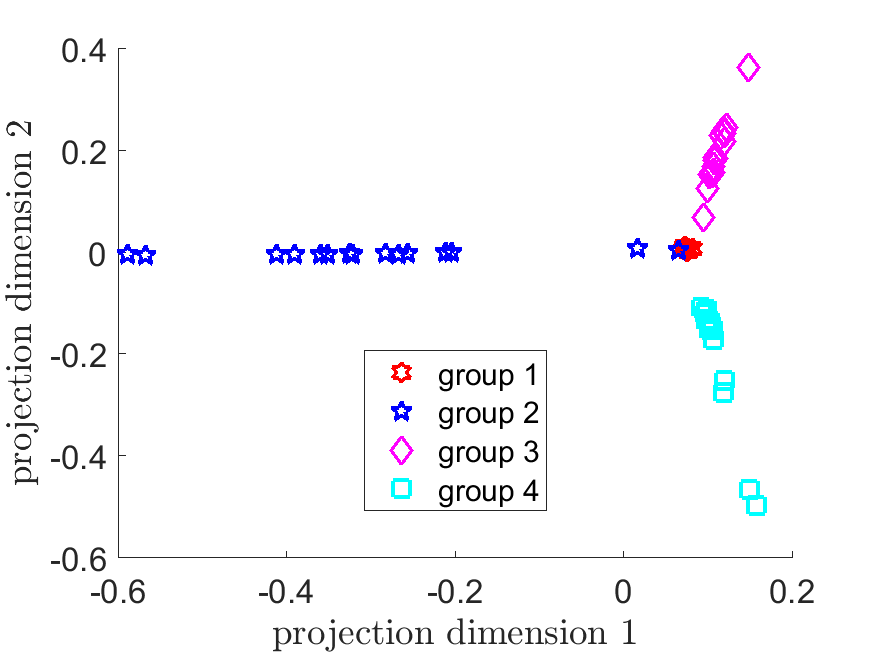}} 
\caption{(a) Similarity matrix among tasks described in Table~\ref{tab:data6}; (b) Relationship among tasks shown by Kernel PCA.}\label{fig:simu_task_cluster}
\end{figure}

In Fig~\ref{fig:simu_task_cluster} (a), Group 2,3 and 4 can be recognized as three groups. In Group 1, each task shows no similarity with other tasks, because with the true $k=1$, the data samples can be randomly partitioned into $\hat{k} = 20$ sub-populations, which results in low NMI scores. In Fig~\ref{fig:simu_task_cluster} (b), basically, 4 groups of tasks are clustered into 4 different regions.

}

\subsubsection{Handling Outlier Samples}\label{subsubsec:simu_outlier_sample}
We choose the data set used in Section~\ref{subsubsec:sim_diff_k} with the true $k = 2$, then randomly shuffle the data pairs $(\y_i,\x_i)$, for $i=1,\ldots,n$, and contaminate the true targets by the following procedure.
For outlier ratio $p_{\mbox{outlier}} = 0\% , 1\% , 2\% , 5\% , 8\% , 10\%$, (1) for Gaussian targets, set all the targets of $p_{\mbox{outlier}}$ of data samples to be $100$; (2) for Bernoulli targets, set all the targets of $p_{\mbox{outlier}}$ of data samples to be $1$.
Such contamination is only performed on training and validation data, leaving testing data clean.

Then we evaluate two groups of methods. For the group of non-robust methods, we choose our~\methodname\, methods \textbf{Mix} and \textbf{Mix GS}. For the group of robust methods, we firstly run the robust version of the non-robust methods by adding $\zzeta$ in the natural parameter models as~\eqref{eq:outlier_g} and adding~\eqref{eq:pen_outlier_group} as the additional penalty, then we clean the data by removing $p_{\mbox{outlier}}$ of data samples associated with the largest value of $\sqrt{\sum_{jr}\zeta_{ijr}^2}$ ($i\in\{1,\ldots,n\}$). Finally, we run their non-robust version of methods on the ``cleaned'' data, respectively. We follow~\citet{gong2012robust} to adopt such two-stage strategy.

The imputation performances are reported in Table~\ref{tab:comp_sample_outlier_simu}, from where it can be seen that, 1) when $p_{\mbox{outlier}} = 0\%$, robust methods are over-parameterized and may underperform non-robust methods; 2) when $p_{\mbox{outlier}} > 0\%$, robust methods significantly outperform non-robust methods.

\setlength{\tabcolsep}{2pt}
\begin{table}[htbp]\small
  \centering
  \renewcommand{\multirowsetup}{\centering}
  \begin{tabular}{ccccccccc}
     \hline
      &   &  & 0\%  & 1\% & 2\% & 5\%  & 8\% & 10\%  \\
    \hline
    \multirow{4}{*}{\myminitab[c]{nMSE \\ for Gaussian}} & \multirow{2}{*}{\textbf{Mix}}   & non-robust  &
    0.0625 & 0.6754 & 0.6894 & 1.0122 & 1.3250 & 1.4953\\
    \cline{3-9}
      &   & robust  &
    0.0620 & 0.0627 & \textbf{0.0626} & 0.0737 & 0.0632 & 0.0635\\
    \cline{2-9}
      &\multirow{2}{*}{\textbf{Mix GS}}   & non-robust  &
    0.0658 & 0.6434 & 0.6741 & 0.7505 & 1.0736 & 1.2939\\
    \cline{3-9}
      &   & robust  &
    \textbf{0.0599} & \textbf{0.0611} & 0.0673 & \textbf{0.0694} & \textbf{0.0602} & \textbf{0.0607}\\
     \hline
    \multirow{4}{*}{aAUC} & \multirow{2}{*}{\textbf{Mix}}   & non-robust  &
    0.9571 & 0.7954 & 0.7961 & 0.7982 & 0.7986 & 0.7981\\
    \cline{3-9}
      &   & robust  &
    0.9570 & 0.9571 & \textbf{0.9574} & \textbf{0.9519} & 0.9568 & 0.9567\\
    \cline{2-9}
      &\multirow{2}{*}{\textbf{Mix GS}}   & non-robust  &
    0.9509 & 0.7979 & 0.7984 & 0.7982 & 0.7979 & 0.7952\\
    \cline{3-9}
      &   & robust  &
    \textbf{0.9581} & \textbf{0.9577} & 0.9519 & 0.9482 & \textbf{0.9578} & \textbf{0.9574}\\
     \hline
    \multirow{4}{*}{\myminitab[c]{nMSE \\ for Poisson}} & \multirow{2}{*}{\textbf{Mix}}   & non-robust  &
    0.2089 & 0.7368 & 0.6905 & 0.6528 & 0.6642 & 0.6736\\
    \cline{3-9}
      &   & robust  &
    \textbf{0.2086} & \textbf{0.2105} & \textbf{0.2099} & 0.2345 & 0.2222 & \textbf{0.2230}\\
    \cline{2-9}
      &\multirow{2}{*}{\textbf{Mix GS}}   & non-robust  &
    0.2136 & 0.7416 & 0.7795 & 0.6587 & 0.6688 & 0.8665\\
    \cline{3-9}
      &   & robust  &
    0.2087 & 0.2109 & 0.2169 & \textbf{0.2202} & \textbf{0.2212} & 0.2236\\
     \hline
  \end{tabular}
 \caption{Comparison between methods handling outlier samples on synthetic data.}\label{tab:comp_sample_outlier_simu}
\end{table}

{
\subsubsection{Feature-Based Prediction by MOE}\label{subsubsec:simu_MOE}

We set the true $k=3$. The true $\aalpha\in \mathbb{R}^{d\times k}$, whose first four rows are non-zero. The non-zero entries of $\aalpha$ are drawn from $\mathcal{N}(0,1)$. Number of data samples $n=1000$. For all $i=1,\ldots,n,r=1,\ldots,k$, the $i$th data sample coming from the $r$th sub-population obeys a multinomial distribution with the probability defined in~\eqref{eq:gate_prob}. Other settings are the same as in Section~\ref{subsubsec:sim_non_FMR}.

We compare our~\methodname\, methods \textbf{Mix MOE} and \textbf{Mix MOE GS}. The prediction performances are shown in Table~\ref{tab:comp_MOE}, which are consistent with the results in Section~\ref{subsubsec:sim_non_FMR}.

We further show in Table~\ref{tab:comp_nmi_pred} the concordance between $p(\delta_{i,r} = 1\mid \x_i,\hat{\alpha}_r )$ and $p(\delta_{i,r} = 1\mid \x_i, \alpha_{0,r} )$, where $\aalpha_0$ denotes the true $\aalpha$, for all $i=1,\ldots,n,r=1,\ldots,k$, for both training and testing data. In~\eqref{eq:update_alpha_MOE}, $\aalpha$ is optimized by partially minimizing the discrepancy between $p(\delta_{i,r} = 1\mid \x_i,\hat{\alpha}_r )$ and $ \hat{\rho}^{(t+1)}_{i,r}$ for $t=0,\ldots,T-1$. As such we also show the concordance between $p(\delta_{i,r} = 1\mid \x_i,\hat{\alpha}_r )$ and $ \hat{\rho}^{(T)}_{i,r} = p(\delta_{i,r} = 1 \mid y_{ij'},j' \in \Omega_i,\x_i,\hat{\theta}_2)$. The concordances are measured by NMI defined in~\eqref{eq:NMI}. We use NMI instead of KL-divergence, because NMI is normalized to the range of $[0,1]$.

Both $p(\delta_{i,r} = 1\mid \x_i, \alpha_{0,r} )$ and $p(\delta_{i,r} = 1 \mid y_{ij'},j' \in \Omega_i,\x_i,\hat{\theta}_2)$ are approximated accurately on the training data. The approximation accuracies are lower on the testing data because the deficiency of data samples comparing with the dimension.

\begin{table}[htbp]\small
  \centering
  \renewcommand{\multirowsetup}{\centering}
  \begin{tabular}{lcc}
    \hline
    &nMSE &aAUC\\
    \hline
LASSO & 0.6390 & 0.7834\\
\hline
Mix MOE & 0.0656 & 0.9466\\
\hline
Group LASSO & 0.6348 & 0.7878\\
\hline
Mix MOE GS & \textbf{0.0579} & \textbf{0.9502}\\
\hline
Sep L2 & 0.6481 & 0.7794\\
\hline
GO-MTL & 0.6946 & 0.7778\\
\hline
CMTL & 0.6496 & 0.7796\\
\hline
MSMTFL & 0.6397 & 0.7831\\
\hline
TraceReg & 0.6509 & 0.7790\\
\hline
SparseTrace & 0.6473 & 0.7805\\
\hline
RMTL & 0.6511 & 0.7797\\
\hline
Dirty & 0.6348 & 0.7878\\
\hline
rMTFL & 0.6483 & 0.7787\\
\hline
  \end{tabular}
  \caption{Prediction performances based on only features.}\label{tab:comp_MOE}
\end{table}

\begin{table}[htbp]\small
  \centering
  \renewcommand{\multirowsetup}{\centering}
  \begin{tabular}{lcccc}
    \hline
      & \multicolumn{2}{c}{Training} & \multicolumn{2}{c}{Testing}\\
      \hline
      &$C(\hat{\aalpha}\parallel \aalpha_0 )$& $C(\hat{\aalpha}\parallel \hat{\theta}_2 )$ &$C(\hat{\aalpha}\parallel \aalpha_0 )$& $C(\hat{\aalpha}\parallel \hat{\theta}_2 )$\\
    \hline
 Mix MOE & 0.9863 & 0.9918  &0.8440  & 0.8457\\
\hline
Mix MOE GS& \textbf{0.9962} &\textbf{0.9933} &  \textbf{0.8455}& \textbf{0.8516}\\
\hline
  \end{tabular}
  \caption{Approximation performances based on only features. $C(\hat{\aalpha}\parallel \aalpha_0 )$ denotes the concordance between $p(\delta_{i,r} = 1\mid \x_i,\hat{\alpha}_{r} )$ and $p(\delta_{i,r} = 1\mid \x_i, \alpha_{0,r} )$, where $\aalpha_0$ denotes the true $\aalpha$. $C(\hat{\aalpha}\parallel \hat{\theta}_2 )$ denotes the concordance between $p(\delta_{i,r} = 1\mid \x_i,\hat{\alpha}_r )$ and $p(\delta_{i,r} = 1 \mid y_{ij'},j' \in \Omega_i,\x_i,\hat{\theta}_2)$. The concordances are measured by NMI defined in~\eqref{eq:NMI}.}\label{tab:comp_nmi_pred}
\end{table}

}

{
\subsubsection{Scalability}\label{subsubsec:simu_scalability}

We discuss the scalability of our method for increasing number of features and tasks. The running time is evaluated.
We choose the data set used in Section~\ref{subsubsec:sim_diff_k} with the true $k=4$. The sparsity $s$ is fixed to be 4. We set the number of features $d \in \{32,64,128,256,512\}$ and the number of tasks $m\in\{15,30,60,120,240\}$. The ratios between the numbers of Gaussian, Bernoulli and Poisson tasks are the same as in Section~\ref{subsubsec:sim_diff_k}. We randomly generate 100 data sets for each pair of $(d,m)$.
We report the results of the method \textbf{Mix GS} only, as the results of \textbf{Mix} are similar. In each case, $\hat{k}$ is tuned and is equal to the true $k=4$.
The estimated parameters in different cases may have different numbers of relevant features. As such, in order to provide a fair comparison, we report the running time per feature, i.e., running time divided by the number of non-zero features of the estimated parameters in each case.

In Fig~\ref{fig:simu_scalability}, both dimension $d$ and the number of tasks $m$ have no significant influence on running time per feature, especially when $d$ and $m$ is large, which is consistent with our time-complexity analysis in Section~\ref{subsec:optimization}.

\begin{figure}[htbp]\small
\centering
\subfigure[Training Time]{\includegraphics[width=2.3in]{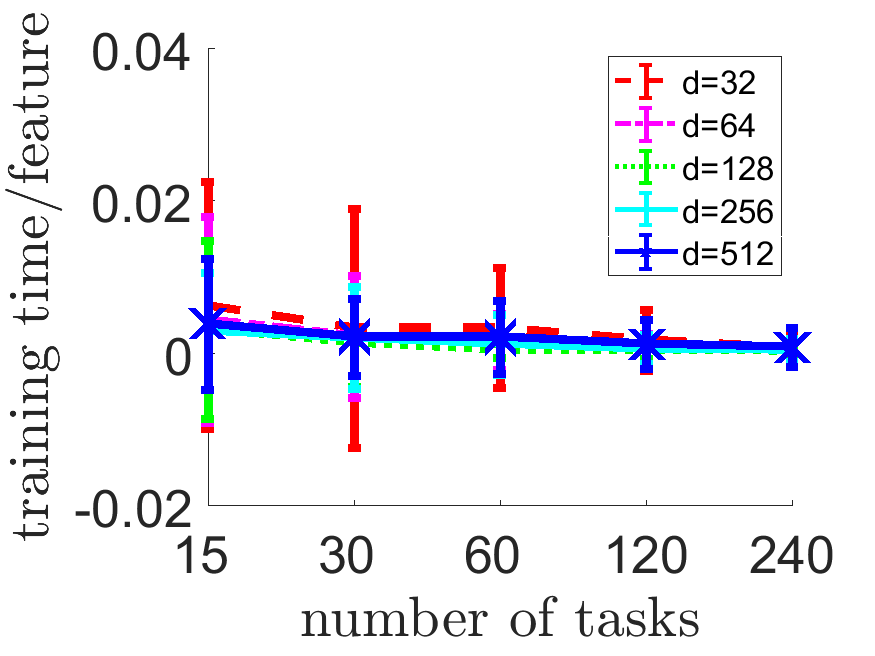}} 
\subfigure[Testing Time]{\includegraphics[width=2.3in]{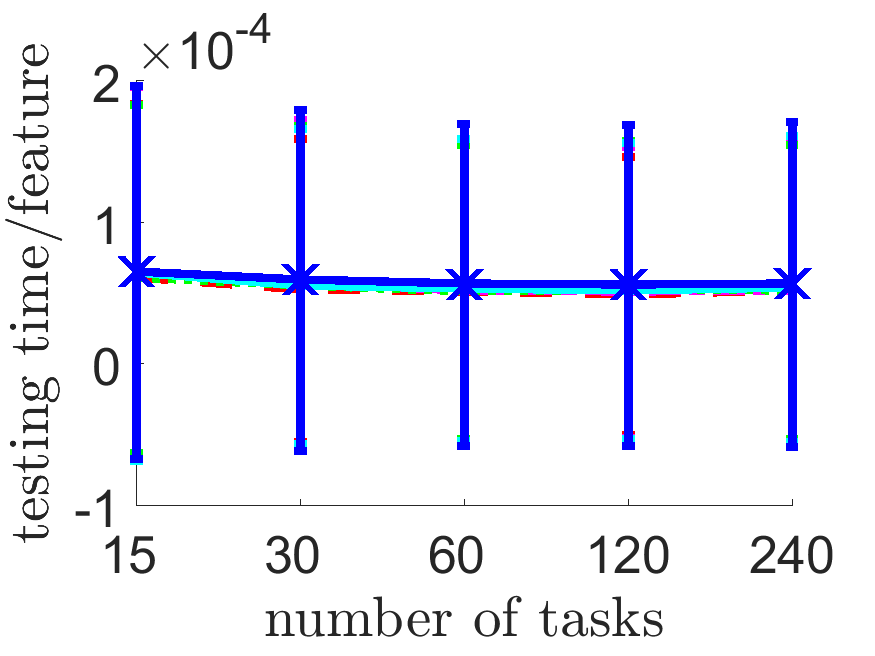}} 
\caption{Running time per feature when dimension and the number of tasks grow. (a) reports the running time per feature for training the model; (b) reports the running time per feature for testing.}\label{fig:simu_scalability}
\end{figure}

}

\subsection{Application}

On the real-world data sets, we first demonstrate the existence of the heterogeneity of conditional relationship, then report the superiority of our~\methodname\, method over other methods considered in Section~\ref{subsec:method_comp}. We further interpret the advantage of our method by presenting the selected features. {Effectiveness of anomaly-task detection and task clustering strategy is also validated.}


\subsubsection{Data Description}

Both real-world data sets introduced in the following are longitudinal surveys for elder patients, which includes a set of questions. Some of the question answers are treated as input features and some of the questions related to indices of geriatric assessments are treated as targets.

\noindent\underline{\emph{\textbf{LSOA II Data}}}. This data is from the Second Longitudinal Study of Aging (LSOA II) \footnote{https://www.cdc.gov/nchs/lsoa/lsoa2.htm.}. LSOA II is a collaborative study by the National Center for Health Statistics (NCHS) and the National Institute of Aging conducted from 1994-2000. A national representative sample of $9447$ subjects $70$ years of age and over were selected and interviewed. Three separated interviews were conducted during the periods of 1994-1996, 1997-1998, and 1999-2000, respectively. The interviews are referred to as WAVE 1, WAVE 2, and WAVE 3 interviews, respectively. We use data WAVE 2 and WAVE 3, which includes a total of $4299$ sample subjects and 44 targets, and $188$ features are extracted from WAVE 2 interview.

Among the targets, specifically, three self-rated health measures, including overall health status, memory status and depression status, can be regarded as continuous outcomes; there are 41 binary outcomes, which fall into several categories: fundamental daily activity, extended daily activity, social involvement, medical condition, on cognitive ability, and sensation condition.
The features include records of demographics, family structure, daily personal care, medical history, social activity, health opinion, behavior, nutrition, health insurance and income and assets, the majority of which are binary measurements.
Both targets and features have missing values due to non-response and questionnaire filtering. The average missing value rates in targets and features are 13.7\% and 20.2\%, respectively. For the missing values in features, we adopt the following procedure for pre-processing. For continuous features, the missing values are imputed with sample mean. For binary features, a better approach is to treat missing as a third category as it may also carry important information; as such, two dummy variables are created from each binary feature with missing values (the third one is not necessary.) This results in totally $d=293$ features.
We randomly select 30\% of the samples for training, 30\% for validation and the rest for testing. 

\noindent\underline{\emph{\textbf{easySHARE Data}}}. This data is a simplified data set from the Survey of Heath, Aging, and Retirement
in Europe (SHARE)\footnote{http://www.share-project.org/data-access-documentation.html.}. SHARE includes multidisciplinary and cross-national panel databases
on health, socio-economic status, and social and family networks of more than 85,000
individuals from $20$ European countries aged $50$ or over. Four waves of interviews were
conducted during 2004 - 2011, and are referred to as WAVE 1 to WAVE 4 interviews. We use WAVE 1 and WAVE 2, which includes 20,449 sample persons and 15 targets (among which 11 are binary, and 4 are continuous), and totally 75 features are constructed from WAVE 1 interview.

The targets are from four interview modules: social support, mental health, functional limitation indices and cognitive function indices.
The features cover a wide range of assessments, including demographics, household composition, social support and network, physical health, mental health, behavior risk, healthcare, occupation and income. Detailed description features are not listed in this paper.
Both targets and features have missing values due to non-response and questionnaire filtering. The average missing value rates in targets and features are 6.9\% and 5.1\%, respectively. The same pre-processing procedure as that for LSOA II Data has been adopted and results in totally $d=118$ features.
We randomly select 10\% of the samples for training, 10\% for validation and the rest for testing.
%

\subsubsection{Comparison with FMR Method}\label{subsubsec:real_comp_FMR}
In this experiment, we compare our proposed~\methodname\, methods \textbf{Mix} and \textbf{Mix GS} which handle mixed type of outcomes with \textbf{Sep} which learns different types of tasks separately. \textbf{Single} is abandoned because it learns each task independently and is not able to use targets of other tasks to help increasing imputation performance.

Results are reported in Fig~\ref{fig:diffk_simu_real}, where 1) for both the real data sets, basically, the best pre-fixed $\hat{k}>1$, except for Bernoulli tasks of easySHARE data, suggesting that the heterogeneity of conditional relationship exists in LSOA II data and the Gaussian tasks of easySHARE data; 2) FMR models benefit Gaussian targets more than Bernoulli targets; 3) \textbf{Mix} and \textbf{Mix GS} outperform \textbf{Sep} in Gaussian tasks. However, their performances are comparable with \textbf{Sep} in Bernoulli tasks, which may be because that the number of Bernoulli tasks are much more than that of Gaussian tasks such that the benefit from Gaussian tasks is limited.

\begin{figure}[htbp]\small
\centering
\subfigure[nMSE of Gaussian targets]{\includegraphics[width=2.3in]{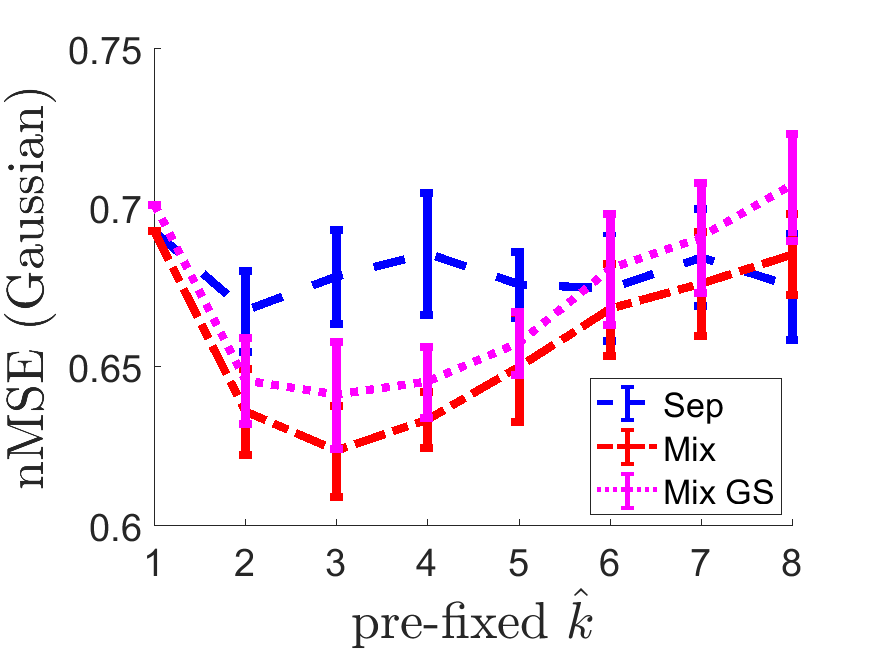}} 
\subfigure[aAUC of Bernoulli targets]{\includegraphics[width=2.3in]{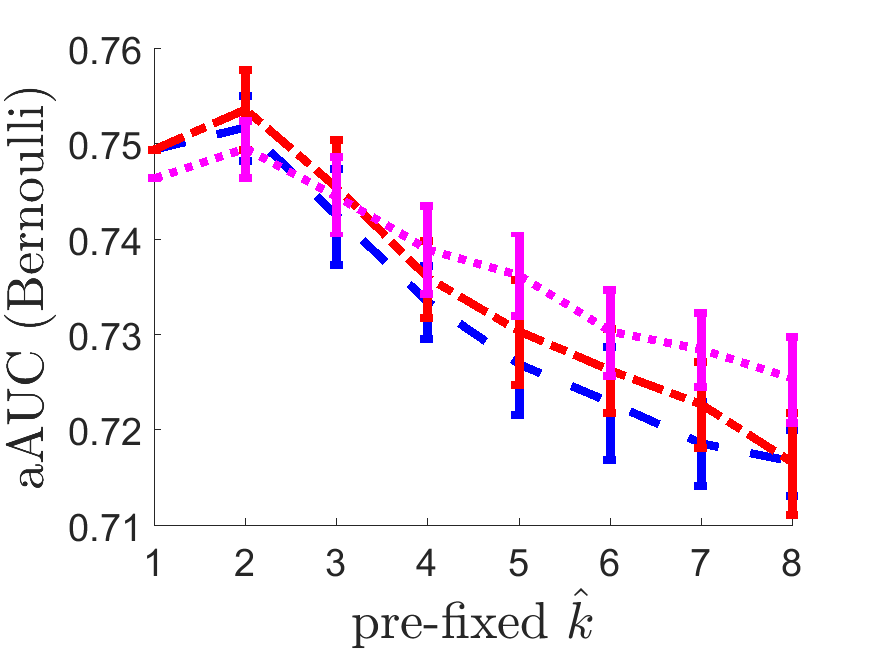}} 
\subfigure[nMSE of Gaussian targets]{\includegraphics[width=2.3in]{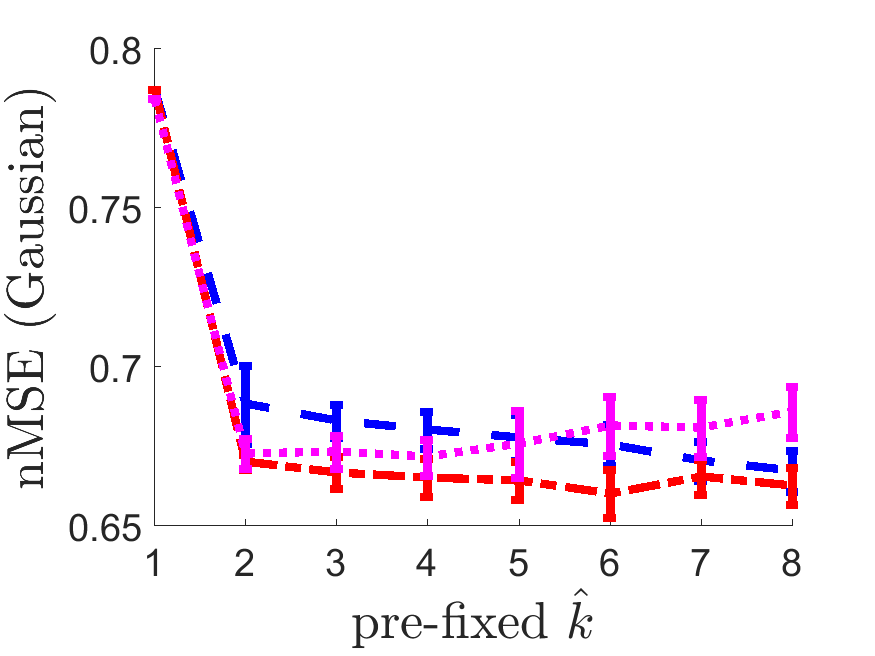}} 
\subfigure[aAUC of Bernoulli targets]{\includegraphics[width=2.3in]{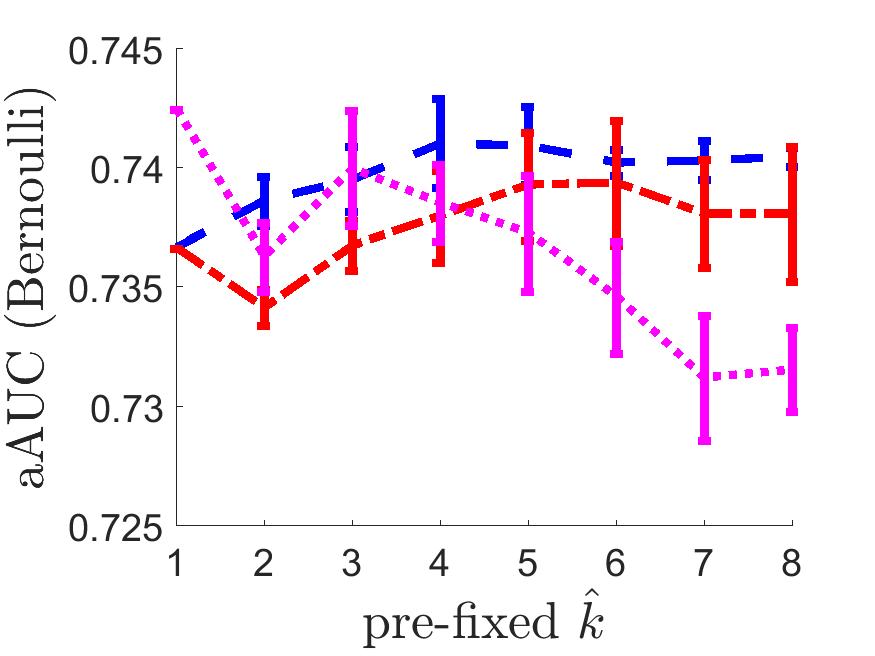}} 
\caption{Comparison with \textbf{Sep} on real data sets. (a) and (b) are results on LSOA II data, (c) and (d) are results on easySHARE data.}\label{fig:diffk_simu_real}
\end{figure}

\subsubsection{Comparison with Non-FMR Methods}
In this experiment we test imputation performance, comparing our~\methodname\, methods \textbf{Mix} and \textbf{Mix GS} with all the non-FMR methods.

Results are reported in Table~\ref{tab:comp_base_real}, where 1) our~\methodname\, methods \textbf{Mix} and \textbf{Mix GS} not only outperform their special cases, \textbf{LASSO} and \textbf{Group LASSO}, respectively, but also outperform other methods, including those handling certain kinds of heterogeneities, except for aAUC on easySHARE, the reason of which has been discussed in Section~\ref{subsubsec:real_comp_FMR}; 2) \textbf{Mix GS} increases nMSE by 9.76\% and 14.37\% on LSOA II data and easySHARE data, respectively, comparing with its non-FMR version \textbf{Group LASSO}. The similar improvements by \textbf{Mix} are witnessed as well.

\begin{table}[htbp]\small
  \centering
  \renewcommand{\multirowsetup}{\centering}
  \begin{tabular}{ccccc}
     \hline
   &\multicolumn{2}{c}{LSOA II}& \multicolumn{2}{c}{easySHARE}\\
   \hline
   &nMSE   &aAUC &nMSE   &aAUC\\
   \hline
   LASSO & 0.7051 & 0.7474 & 0.7869 & 0.7386\\
    \hline
    Mix & 0.6408 & \textbf{0.7525} &0.6601 & 0.7419 \\
    \hline
    Group LASSO & 0.6975 & 0.7413 & 0.7897 & 0.7413 \\
    \hline
    Mix GS & \textbf{0.6294} & 0.7481 & \textbf{0.6548} & 0.7402\\
    \hline
    Sep L2 & 0.7176 & 0.7392 & 0.7796 & 0.7464\\
    \hline
    GO-MTL & 0.8516 & 0.6972 & 0.8231 & 0.7288\\
    \hline
    CMTL & 0.8186 & 0.7089 & 0.7958 & 0.7364\\
    \hline
    MSMTFL & 0.7028 & 0.7473 & 0.7803 & 0.7411 \\
    \hline
    TraceReg & 0.7150 & 0.7408 & 0.7809 & \textbf{0.7496} \\
    \hline
    SparseTrace & 0.6972 & 0.7475 & 0.7791 & 0.7475\\
    \hline
    RMTL & 0.7145 & 0.7418 & 0.7808 & 0.7496\\
    \hline
    Dirty & 0.7032 & 0.7480 & 0.7781 & 0.7486\\
    \hline
    rMTFL & 0.6953 & 0.7418 &0.7781 & 0.7486\\
    \hline
  \end{tabular}
  \caption{Comparison with non-FMR methods on real data sets}\label{tab:comp_base_real}
\end{table}
%

\subsubsection{Feature Selection}\label{subsubsec:feature_selection}
We consider demonstrating the advantage of our~\methodname\, method on feature selection. We compare our~\methodname\, method \textbf{Mix GS} with its non-FMR version \textbf{Group LASSO}. Both methods select shared features across tasks. We collect the unique features that only selected for each sub-population.

For LSOA II data set, {the tuned $\hat{k}=2$}. \textbf{Mix GS} selects $47/294$ features (summing up selected features of both sub-populations), while \textbf{Group LASSO} selects $48/294$ features. Descriptions of unique features of both sub-populations are listed in Table~\ref{tab:sp_feature_disc_data_1}. Sub-population 1 seems considering worse condition of patients.

\setlength{\tabcolsep}{2pt}
\begin{table}[htbp]\small
  \centering
  \begin{tabular}{ll}
    \hline
    Sub-population 1 ($\pi_1 =  70.09\%$) & Sub-population 2 ($\pi_2 =  29.91\%$)\\
   \hline
  able or prevented to leave house & times seen doctor in past 3 months\\
  have problems with balance & easier or harder to walk 1/4 mile\\
  total number of living children  & \textbf{widowed}\\
  easier/harder than before: in/out of bed & \textbf{follow regular physical routine}\\
  \textbf{\#(ADL activities) SP is unable to perform} & \textbf{present social activities}\\
  easier or harder to walk 10 steps & \textbf{ever had a stress test}\\
  do you take aspirin & \textbf{do you take vitamins}\\
  \textbf{often troubled with pain} & \textbf{necessary to use steps or stairs}\\
  visit homebound friend for others & \textbf{had flu shot}\\
  ever had a hysterectomy & \textbf{ever had cataract surgery}\\
    & \textbf{physical activity more/less/same}\\
    \hline
  \end{tabular}
  \caption{Descriptions of unique features of each sub-population of LSOA II data. Features are sorted in descending order by the $\|\bbeta_r^l\|_2$ where $\bbeta_r^l$ is the $l$th row of $\bbeta_r$ ($l = 1,\ldots,d$). The bold denote the features that are not selected by \textbf{Group LASSO}. ``\#($\cdot$)'' denotes the number of the enclosed events.  ``ADL'' denotes Activity of Daily Livings. ``SP'' denotes Standardized Patients. }\label{tab:sp_feature_disc_data_1}
\end{table}

 For easySHARE data set, {the tuned $\hat{k}=5$}. \textbf{Mix GS} selects $58/118$ features, while \textbf{Group LASSO} selects $57/118$ features. Descriptions of unique features of two sub-populations are listed in Table~\ref{tab:sp_feature_disc_data_2}. Sub-population 1 seems considering more about personality and experience, while sub-population 2 seems considering more about politics and education.

 For both real data sets, our~\methodname\, method \textbf{Mix GS} recalls more useful features than \textbf{Group LASSO} does.

\setlength{\tabcolsep}{2pt}
\begin{table}[htbp]\small
  \centering
  \begin{tabular}{ll}
    \hline
    Sub-population 1 ($\pi_1 =  21.31\%$) & Sub-population 2 ($\pi_2 =  26.36\%$)\\
   \hline
  \textbf{fatigue} & \textbf{taken part in a political organization}\\
  guilt & \textbf{attended an educational or training course}\\
  enjoyment  & \textbf{taken part in religious organization}\\
  suicidality & \textbf{none of social activities}\\
  tearfullness & \textbf{cared for a sick or disabled adult}\\
  \textbf{interest} & \textbf{done voluntary or charity work}\\
  \textbf{current job situation:sick} & \textbf{education: lower secondary}\\
   & \textbf{education: first tertiary}\\
   & \textbf{education: post secondary}\\
   & \textbf{education: upper secondary}\\
    & \textbf{education: primary}\\
  & \textbf{education: second tertiary}\\
    \hline
  \end{tabular}
  \caption{Descriptions of unique features of two sub-populations of easySHARE data. Features are sorted in descending order by the $\|\bbeta_r^l\|_2$ where $\bbeta_r^l$ is the $l$th row of $\bbeta_r$ ($l = 1,\ldots,d$). The bold denote the features that are not selected by \textbf{Group LASSO}. }\label{tab:sp_feature_disc_data_2}
\end{table}

\subsubsection{Detection of Anomaly Tasks}\label{subsubsec:real_anmaly_task}

We firstly use~\eqref{eq:outlier_score} to compute concordant scores of tasks, which are reported in Fig~\ref{fig:outlier_score_real}. Clear separations are witnessed on both data sets. We select one-third of tasks with highest scores as concordant tasks and another third with lowest scores as anomaly tasks. The descriptions of the concordant and anomaly tasks are listed in Table~\ref{tab:task_disc_data_1} and Table~\ref{tab:task_disc_data_2}, respectively.

\begin{figure}[htbp]\small
\centering
\subfigure[LSOA II data set]{\includegraphics[width=2.3in]{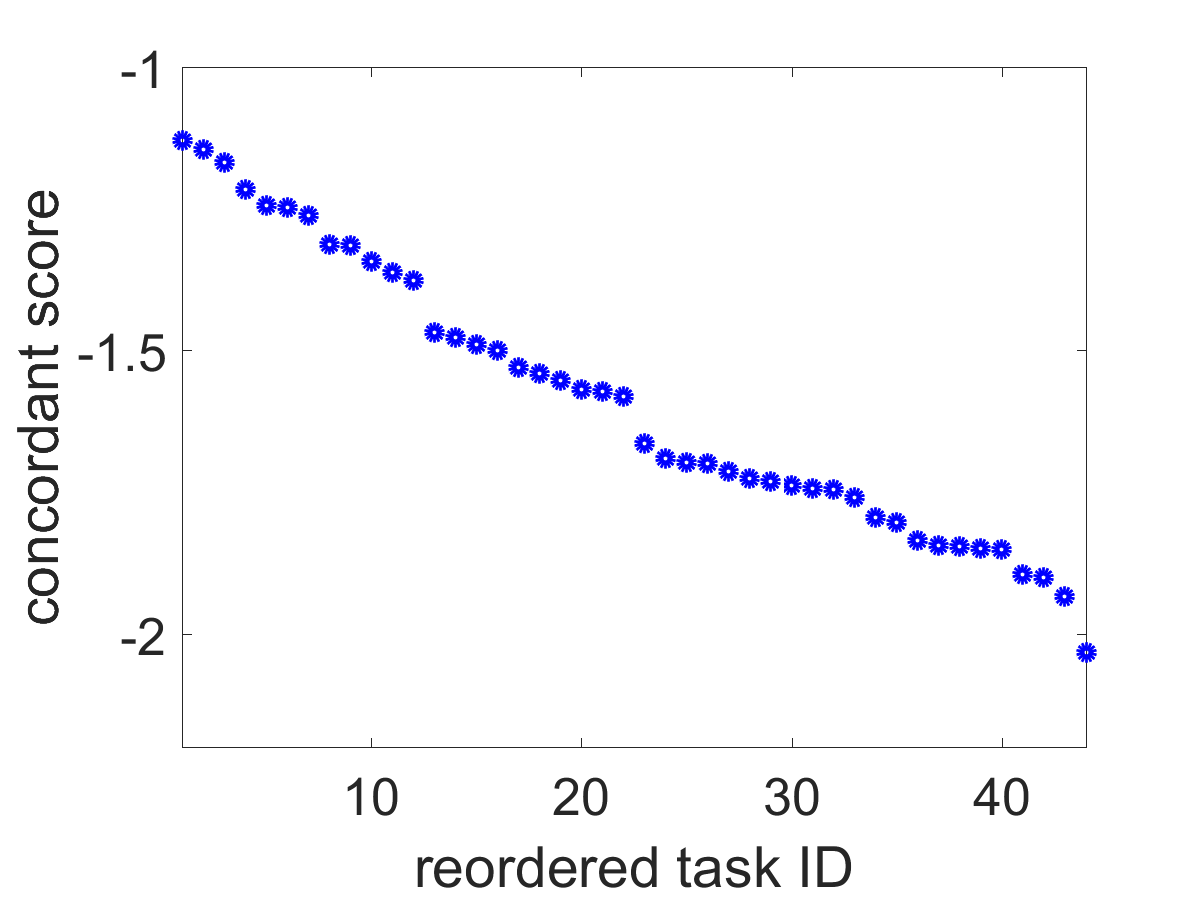}} 
\subfigure[easySHARE data set]{\includegraphics[width=2.3in]{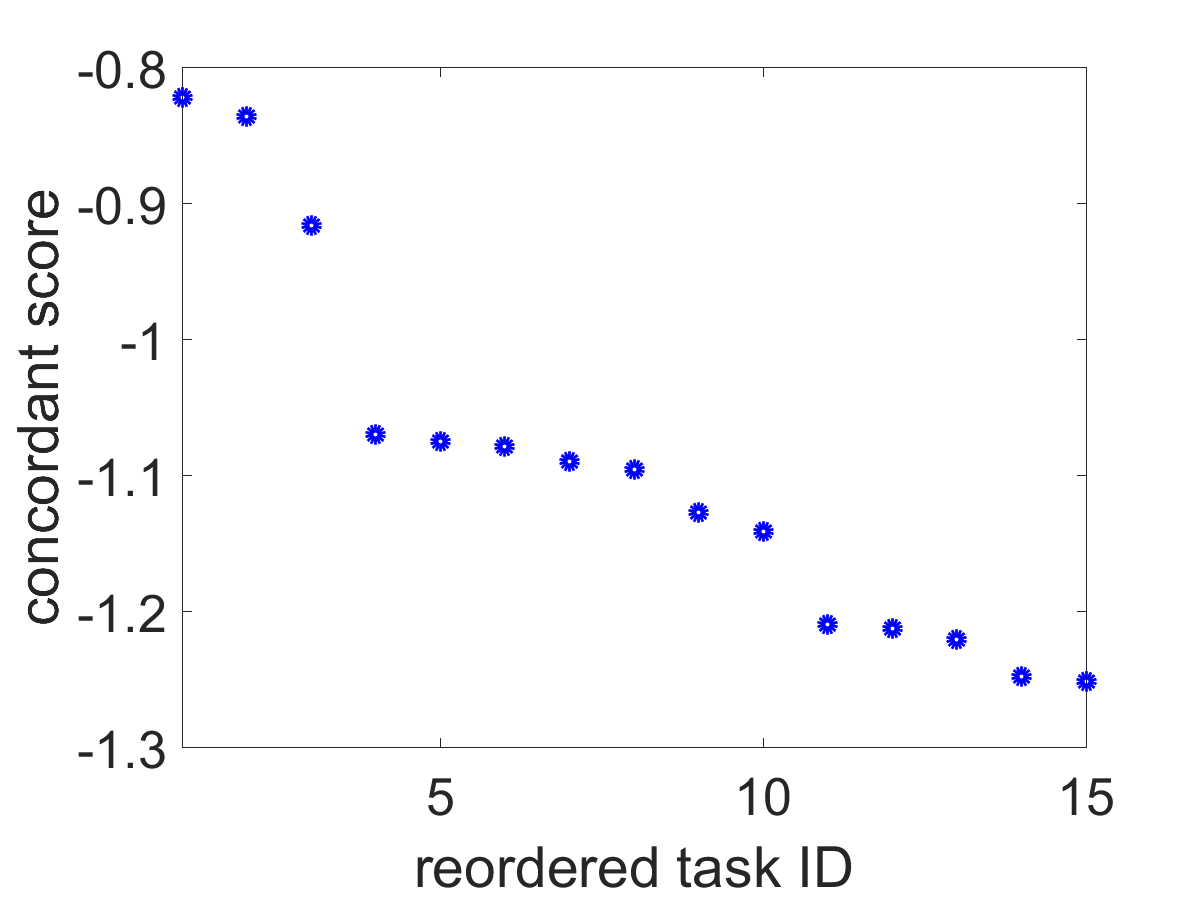}} 
\caption{Concordant scores of tasks, which were estimated by \textbf{Mix GS}. The tasks are reordered according to the scores.}\label{fig:outlier_score_real}
\end{figure}

\setlength{\tabcolsep}{2pt}
\begin{table}[htbp]\small
  \centering
  \begin{tabular}{ll}
    \hline
    Concordant tasks (top 7) & Anomaly tasks (top 8)\\
   \hline
  have difficulty dressing & go to movies, sports, events, etc.\\
  have difficulty doing light housewrk & now have asthma\\
  have difficulty using toilet & now have arthritis\\
  have difficulty managing medication & now have hypertension\\
  have difficulty bathing or showering & injured from fall(s)\\
  have difficulty managing money & memory of year\\
  have difficulty preparing meals & have deafness\\
   & get together with relatives\\
    \hline
  \end{tabular}
  \caption{Descriptions of tasks of LSOA II data}\label{tab:task_disc_data_1}
\end{table}

\setlength{\tabcolsep}{2pt}
\begin{table}[htbp]\small
  \centering
  \begin{tabular}{ll}
    \hline
    Concordant tasks (top 5) & Anomaly tasks (top 5)\\
   \hline
  activities of daily living index  & numeracy score\\
  instrumental activities of daily living indices  & gone to sport social or other kind of club\\
  mobility index & recall of words first trial\\
  appetite  & give help to others outside the household\\
  orientation to date & provided help to family friends or neighbors\\
    \hline
  \end{tabular}
  \caption{Descriptions of tasks of easySHARE data}\label{tab:task_disc_data_2}
\end{table}

The concordant tasks detected by our methods seem truly correlated with each other intuitively. And the information of detected anomaly tasks is diverse and seems different with that of concordant tasks.

{
For each data set, we apply our~\methodname\, method \textbf{Mix} (and \textbf{Mix GS}) to build two models for non-anomaly tasks (the first two-third tasks) and anomaly tasks, respectively. For LSOA II data set, the tuned $\hat{k}=4$ and $1$ for non-anomaly tasks and anomaly tasks, respectively. For easySHARE data set, the tuned $\hat{k}=6$ and $2$ for non-anomaly tasks and anomaly tasks, respectively.

Averaged imputation performances are shown in Table~\ref{tab:anomaly_tasks}. By providing separate models to handle anomaly tasks, the performances improve significantly, where \textbf{Mix GS} outperforms \textbf{Mix}, maybe because the non-anomaly tasks share some relevant features.


\setlength{\tabcolsep}{3.5pt}
\begin{table}[htbp]\small
  \centering
  \renewcommand{\multirowsetup}{\centering}
  \begin{tabular}{lcccc}
    \hline
     & \multicolumn{2}{c}{LSOA II} & \multicolumn{2}{c}{easySHARE}\\
    \hline
     & nMSE & aAUC & nMSE & aAUC  \\
     \hline
    Mix - All tasks & 0.6408 &  0.7525    &0.6601 &  0.7419   \\
    \hline
    Mix - Handle anomalies &0.5979 & 0.7602   &0.6569 & 0.7370    \\
    \hline
    Mix GS - All tasks &  0.6294 & 0.7481   &  0.6548  & 0.7402 \\
    \hline
    Mix GS - Handle anomalies & \textbf{0.5923 } & \textbf{0.7649}   & \textbf{0.6462} &  \textbf{0.7447} \\
    \hline
  \end{tabular}
  \caption{Comparison for imputation performances. ``All tasks'' denotes building one FMR model for all the tasks. ``Handle anomalies'' denotes building two models for non-anomaly tasks and anomaly tasks, respectively.}\label{tab:anomaly_tasks}
\end{table}

}

{

\subsubsection{Handling Clustered Relationship among tasks}\label{subsubsec:real_task_cluster}

We adopt the same strategy as that in Section~\ref{subsubsec:simu_task_cluster} to construct a similarity matrix and perform dimension reduction for each of the real-world data sets.

For LSOA II data set, the similarity matrix and results of 2D reduction are shown in Fig~\ref{fig:real_task_cluster_LSOA}. In Fig~\ref{fig:real_task_cluster_LSOA}(b), tasks are partitioned into groups. We apply k-means algorithm to separate the tasks into 4 groups. Tasks in Group 1 are mainly about current status. The descriptions of tasks of Group 2 are ``how often felt sad or depressed in the past 12 months'' and ``self rated memory''. Tasks in Group 3 and 4 are about having difficulty performing some certain actions. Group 3 is similar to Group 4, which can be reflected by Fig~\ref{fig:real_task_cluster_LSOA}(b).

\begin{figure}[htbp]\small
\centering
\subfigure[Similarity Matrix]{\includegraphics[width=2.3in]{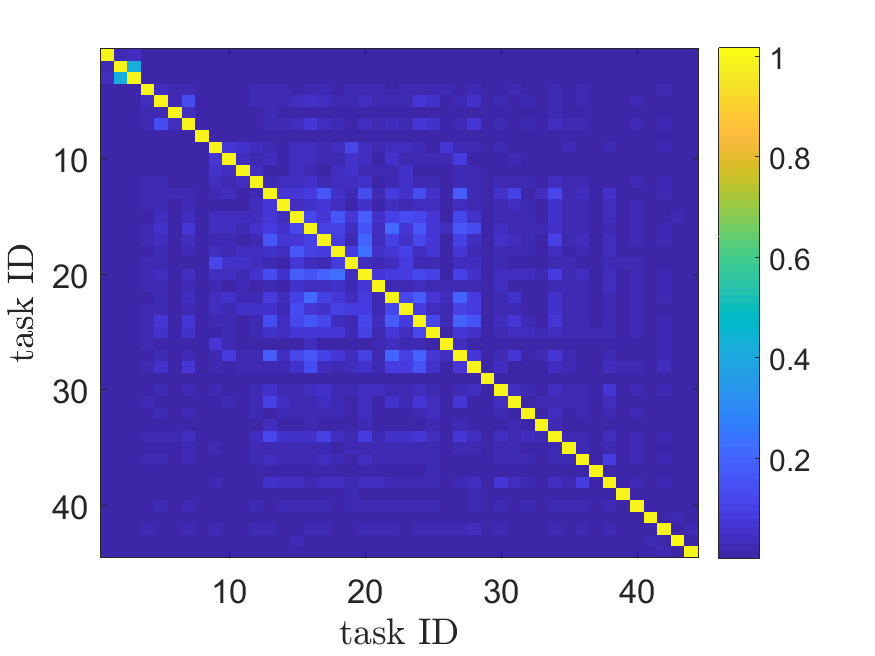}} 
\subfigure[Dimension Reduction by Kernel PCA]{\includegraphics[width=2.3in]{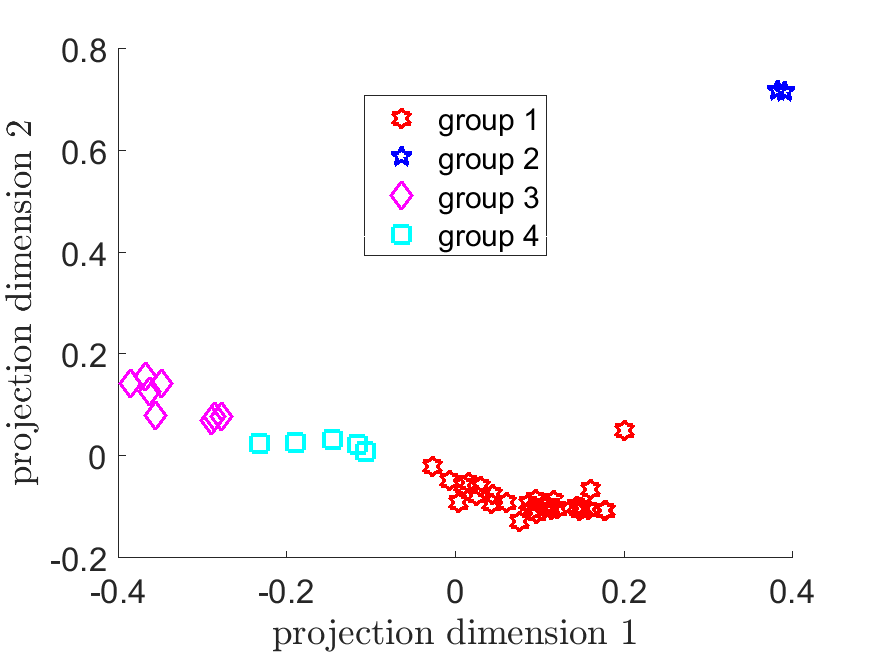}} 
\caption{Clustered Relationship among tasks on LSOA II. (a) Similarity matrix among tasks. First three tasks are Gaussian tasks. Other tasks are Bernoulli tasks. (b) Relationship among tasks shown by Kernel PCA.}\label{fig:real_task_cluster_LSOA}
\end{figure}

For easySHARE data set, the similarity matrix and results of 2D reduction are shown in Fig~\ref{fig:real_task_cluster_easyshare}. In Fig~\ref{fig:real_task_cluster_easyshare}(b), tasks are partitioned into groups as well. We also apply k-means algorithm to separate the tasks into 4 groups. The descriptions of tasks for each group are shown in Table~\ref{tab:clustered_task_easySHARE}, where descriptions of 4 types of interview modules are basically separated into 4 groups, respectively. The only ``misclassified'' task with the description of ``Orientation to date'' seems to be more related to other tasks in Group 2 than the tasks in Group 4.

\begin{figure}[htbp]\small
\centering
\subfigure[Similarity Matrix]{\includegraphics[width=2.3in]{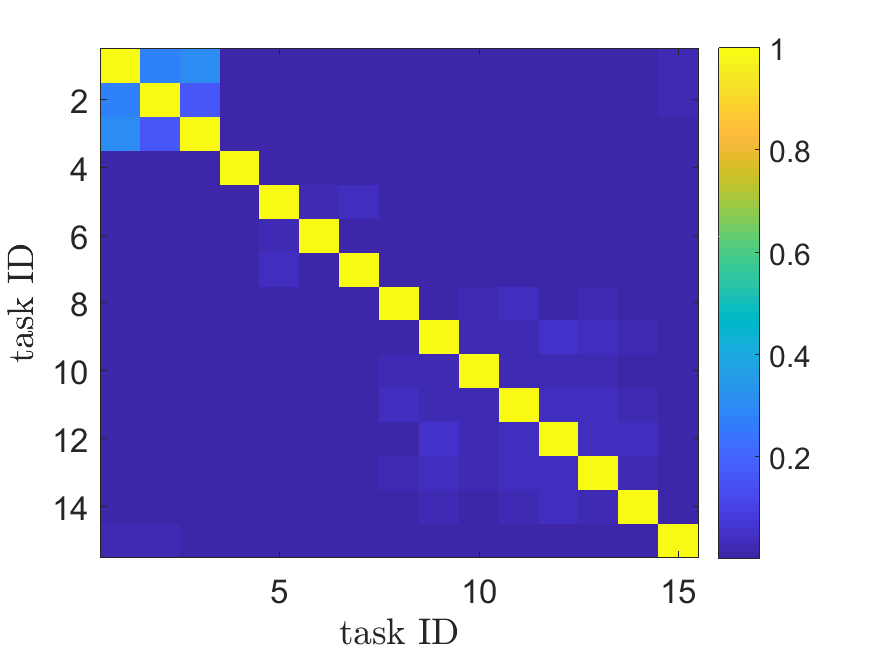}} 
\subfigure[Dimension Reduction by Kernel PCA]{\includegraphics[width=2.3in]{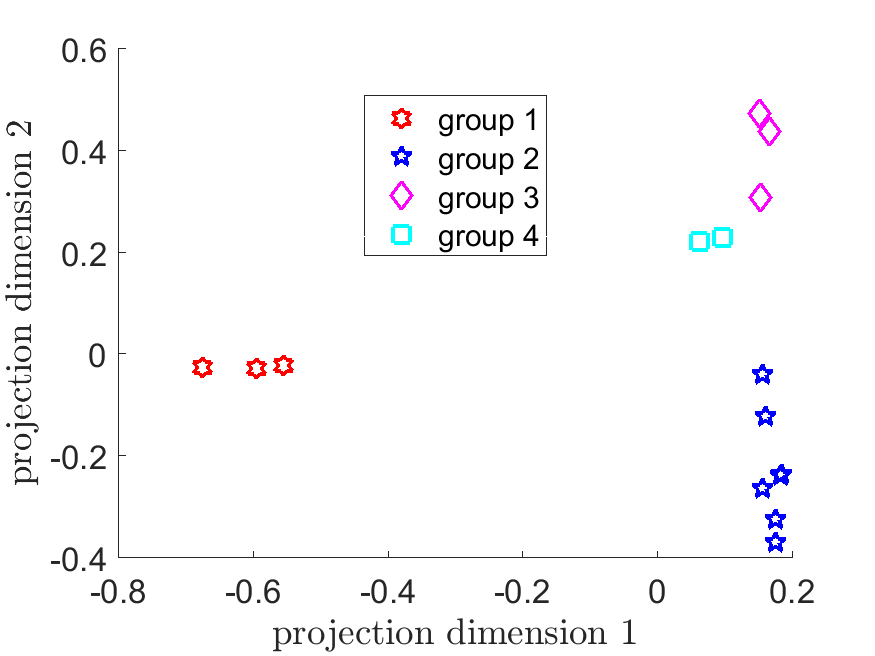}} 
\caption{Clustered Relationship among tasks on easySHARE. (a) Similarity matrix among tasks. First four tasks are Gaussian tasks. Other tasks are Bernoulli tasks. (b) Relationship among tasks shown by Kernel PCA.}\label{fig:real_task_cluster_easyshare}
\end{figure}

\begin{table}[htbp]\small
  \centering
  \renewcommand{\multirowsetup}{\centering}
  \begin{tabular}{cll}
     \hline
   Group & Targets &Interview module    \\
   \hline
  \multirow{3}*{1}  &Activities of daily living index& Functional Limitation Indices \\
    &Instrumental activities of daily living index& Functional Limitation Indices \\
   &Mobility index &Functional Limitation Indices  \\
   \hline
  \multirow{7}*{2}  &Depression& Mental Health \\
    &Pessimism &Mental Health \\
    &Sleep &Mental Health \\
    &Irritability &Mental Health \\
    &Appetite& Mental Health \\
    &Concentration &Mental Health \\
    &Orientation to date &Cognitive Function Indices \\
    \hline
  \multirow{3}*{3}&Provided help to family friends or neighbors & Social Support \& Network  \\
    &Gone to sport social or other kind of club& Social Support \& Network \\
    &Give help to others outside the household& Social Support \& Network \\
     \hline
  \multirow{2}*{4}&Recall of words score& Cognitive Function Indices \\
    &Numeracy score &Cognitive Function Indices \\
    \hline
  \end{tabular}
  \caption{Clustered tasks of easySHARE. 15 tasks are clustered into 4 groups.}\label{tab:clustered_task_easySHARE}
\end{table}

For each data set, we further apply our~\methodname\, methods \textbf{Mix} and \textbf{Mix GS} for each group of tasks. For LSOA II data set, tuned $\hat{k}=3,3,5$ and $2$ for Group 1,2,3 and 4, respectively. For easySHARE data set, tuned $\hat{k}=5,2,1$ and $1$ for Group 1,2,3 and 4, respectively. Imputation performances are shown in Table~\ref{tab:comp_clustered_task_all_task}. Performances increase by building separate models for each group, suggesting that separate models for clustered tasks are more accurate.



\setlength{\tabcolsep}{1pt}
\begin{table}[htbp]\small
  \centering
  \renewcommand{\multirowsetup}{\centering}
  \begin{tabular}{lcccc}
    \hline
     & \multicolumn{2}{c}{LSOA II} & \multicolumn{2}{c}{easySHARE}\\
    \hline
     & nMSE & aAUC & nMSE & aAUC   \\
     \hline
    Mix - All tasks             & 0.6408 & 0.7525   &0.6601 & 0.7419 \\
    \hline
    Mix - Clustered tasks       & 0.6370& \textbf{0.7592} & 0.6552 & 0.7439  \\
    \hline
    Mix GS - All tasks          & 0.6294 & 0.7481  & 0.6548 & 0.7402\\
    \hline
    Mix GS - Clustered tasks    & \textbf{0.6202}  & 0.7559 & \textbf{0.6533} & \textbf{0.7474}\\
    \hline
  \end{tabular}
  \caption{Comparison for imputation performances. ``All tasks'' denotes building one FMR model for all the tasks. ``Clustered tasks'' denotes building different FMR models for different groups of tasks.}\label{tab:comp_clustered_task_all_task}
\end{table}

}

{

\subsubsection{Feature-Based Prediction by MOE}\label{subsubsec:real_MOE}

We compare the methods using only features to predict targets on both real-world data sets. Our proposed MOE type of~\methodname\, methods, \textbf{Mix MOE} and \textbf{Mix MOE GS}, are compared with the non-FMR methods. We also integrate our strategies to handle anomaly tasks and clustered structure among tasks in both our proposed MOE type of~\methodname\, methods \textbf{Mix MOE} and \textbf{Mix MOE GS}. Concretely, we use the anomaly-task detection results in Section~\ref{subsubsec:real_anmaly_task} and the task clustering results in Section~\ref{subsubsec:real_task_cluster}.

The prediction results are reported in Table~\ref{tab:comp_base_real_MOE_combine}. Our proposed~\methodname\, method \textbf{Mix MOE} and \textbf{Mix MOE GS} outperform baseline methods on LSOA II and on Gaussian tasks of easySHARE, which is consistent with the results in Table~\ref{tab:comp_base_real}. In addition, by integrating our task clustering strategy, our proposed~\methodname\, methods \textbf{Mix MOE TC} and \textbf{Mix MOE GS TC} outperform other methods on Gaussian targets, while providing comparable results on Bernoulli tasks. \textbf{Mix MOE TC} and \textbf{Mix MOE GS TC} even outperform our proposed~\methodname\, methods \textbf{Mix MOE Robust} and \textbf{Mix MOE GS Robust} on Gaussian targets, suggesting that it is more accurate to build a specific model for each cluster of tasks.

Comparing Table~\ref{tab:comp_base_real_MOE_combine} with Table~\ref{tab:comp_base_real}, our MOE methods do not rival our FMR methods. We investigate the reason by showing the concordance between $p(\delta_{i,r} = 1\mid \x_i,\hat{\alpha}_r )$ and $p(\delta_{i,r} = 1 \mid y_{ij'},j' \in \Omega_i,\x_i,\hat{\theta}_2) = \hat{\rho}_{i,r}^{(T)}$ ($\hat{\rho}_{i,r}^{(T)}$ is defined in equation~\ref{eq:cond_prob_t_MOE}) in Table~\ref{tab:comp_nmi_pred_real}. In Table~\ref{tab:comp_nmi_pred_real}, the concordances of conditional probabilities measured by NMI are generally low, especially comparing with the results in Table~\ref{tab:comp_nmi_pred}, suggesting that on both real-world data sets, it is difficult to learn the mixture probabilities.

\begin{table}[htbp]\small
  \centering
  \renewcommand{\multirowsetup}{\centering}
  \begin{tabular}{lcccc}
     \hline
   &\multicolumn{2}{c}{LSOA II}& \multicolumn{2}{c}{easySHARE}\\
   \hline
   &nMSE  &aAUC &nMSE &aAUC\\
   \hline
   LASSO & 0.7051 & 0.7474 & 0.7869 & 0.7386\\
    \hline
    Group LASSO & 0.6975 & 0.7413 & 0.7897 & 0.7413 \\
    \hline
    MSMTFL & 0.7028 & 0.7473 & 0.7803 & 0.7411 \\
    \hline
    Sep L2 & 0.7176 & 0.7392 & 0.7796 & 0.7464\\
    \hline
    GO-MTL & 0.8516 & 0.6972 & 0.8231 & 0.7288\\
    \hline
    CMTL & 0.8186 & 0.7089 & 0.7958 & 0.7364\\
    \hline
    TraceReg & 0.7150 & 0.7408 & 0.7809 & \textbf{0.7496} \\
    \hline
    SparseTrace & 0.6972 & 0.7475 & 0.7791 & 0.7475\\
    \hline
    RMTL & 0.7145 & 0.7418 & 0.7808 & 0.7496\\
    \hline
    Dirty & 0.7032 & 0.7480 & 0.7781 & 0.7486\\
    \hline
    rMTFL & 0.6953 & 0.7418 &0.7781 & 0.7486\\
     \hline
      \hline
    Mix MOE &0.6935 & \textbf{0.7504} &  0.7991  & 0.7395\\
    \hline
    Mix MOE GS & 0.7054 & 0.7438 & 0.7774 & 0.7387\\
    \hline
    Mix MOE Robust & 0.6906 & 0.7436 & 0.7642 & 0.7351\\
    \hline
    Mix MOE GS Robust & 0.6981 & 0.7430 & 0.7668 & 0.7344\\
    \hline
    Mix MOE TC & \textbf{0.6859} & 0.7333 & \textbf{0.7584} & 0.7389\\
    \hline
    Mix MOE GS TC & 0.6925 & 0.7379 & 0.7657 & 0.7367\\
    \hline
  \end{tabular}
  \caption{Comparison for prediction performance with non-FMR methods on real data sets. ``Robust'' denotes adopting the strategy to handle anomaly tasks. ``TC'' denotes task clustering strategy.}\label{tab:comp_base_real_MOE_combine}
\end{table}

%
%

\begin{table}[htbp]\small
  \centering
  \renewcommand{\multirowsetup}{\centering}
  \begin{tabular}{lcccc}
    \hline
      & \multicolumn{2}{c}{LSOA II} & \multicolumn{2}{c}{easySHARE}\\
      \hline
      &Training&Testing&Training&Testing\\
    \hline
    Mix MOE &  {0.2745} & {0.1301} & 0.1314& 0.1068\\
    \hline
    Mix MOE GS & 0.1060 & 0.0673 &  {0.2527}& { 0.2054}\\
    \hline
  \end{tabular}
  \caption{The concordance between $p(\delta_{i,r} = 1\mid \x_i,\hat{\alpha}_r )$ and $p(\delta_{i,r} = 1 \mid y_{ij'},j' \in \Omega_i,\x_i,\hat{\theta}_2)$. The concordances are measured by NMI defined in~\eqref{eq:NMI}.}\label{tab:comp_nmi_pred_real}
\end{table}

}

\section{{Discussions} \& Conclusions}
\label{sec:conclusion}

In this paper, we propose a novel model~\methodname\, to explore heterogeneities of conditional relationship, output type and shared information among tasks.
Based on multivariate-target FMR {and MOE} models, our model jointly learns tasks with mixed type of output, allows incomplete data in the output, imposes inner component-wise group $\ell_1$ constraint and handles anomaly tasks {and clustered structure among tasks}. These key elements are integrated in a unified generalized mixture model setup so that they can benefit from and reinforce each other to discover the triple heterogeneities in data. Rigorous theoretical analyses under the high dimensional framework are provided.

{



We mainly consider the special setting {of MTL, where} the multivariate outcomes share the same set of instances and the same set of features because our main objective is to learn potentially shared sample clusters and feature sets among tasks. However, {as stressed in the introduction, the main definition of MTL considers tasks that do not necessarily share the same set of samples/instances and the same set of features}, such as distributed learning systems (different tasks have entirely different data instances, see~\citealt{jin2006fast} and~\citealt{boyd2011distributed}) and multi-source learning systems (different tasks have entirely different feature spaces, see~\citealt{zhang2011multi} and~\citealt{jin2015heterogeneous}). For such cases, one can define the specific expected shared information among tasks and then extend our methodology. For example, although tasks do not share the same instances, they could share the same mixture model structure. Then for the distributed learning systems, our model in Section~\ref{sec:method} can still be applied. Additionally, the tasks could still share the pattern/sparsity in feature selection even though the feature sets are different, e.g.,~\citet{liu2009multi} and~\citet{gong2012multi}. Then one can build FMR models for the tasks in which the regression coefficient vectors of the tasks share the same sparsity pattern achieved by group $\ell_1$ penalization. The case of multi-source learning systems can also be handled similarly by embedding features into a shared feature space, e.g.,~\citet{zhang2011multi} and~\citet{jin2015heterogeneous}.


}

There are many interesting future directions. It is worthwhile to explore the theoretical and empirical performance of non-convex penalties. Meanwhile, different components should share some features, and overlapping cluster pattern of conditional relationship should also be considered in real applications, both of which require further investigation. It is also interesting to explore other low-dimensional structures in the natural parameters, e.g., low-rank structure and its sparse composition~\citep{chen2012reduced}. { Our strategies on handling anomaly tasks and clustered structure among tasks require two stages. It is worthwhile to explore one-stage models to handle such task heterogeneities during a whole learning process. More complicated structure among tasks, such as graph-based structure, should also be explored. Our theoretical results cover our method introduced in Section~\ref{sec:method} and robust estimation introduced in Section~\ref{subsec:robust_estimation}. Nonetheless, theoretical guarantees for other extensions in Section~\ref{sec:extension} are still challenging due to joint learning complicated relationship among tasks and population heterogeneity, which will be focused on in our future research.}

\begin{acknowledgements}
The authors would like to thank the editors and reviewers for their valuable suggestions on improving this paper. This work of Jian Liang and Changshui Zhang is (jointly or partly) funded by National Natural Science Foundation of China under
Grant No.61473167 and Beijing Natural Science Foundation under Grant No. L172037. Kun Chen's work is partially supported by U.S. National Science Foundation under Grants DMS-1613295 and IIS-1718798. The work of Fei Wang is supported by National Science Foundation under Grants IIS-1650723 and IIS-1716432.
\end{acknowledgements}


\small
\appendices

\section{Definitions}

\begin{mydef}\label{th:def_subexp}
$Z = (Z_1,\ldots,Z_{m'})\trans \in \mathbb{R}^{m'} $ has a sub-exponential distribution with parameters $(\sigma,v,t)$ if for $M>t$, it holds
 \begin{align*}
 \mathbb{P}(\|Z\|_{\infty}>M)\leq
 \left\{
   \begin{array}{ll}
     \exp\biggl(-\frac{M^2}{\sigma^2}\biggr), & t\leq M\leq \frac{\sigma^2}{v}\\
     \exp\biggl(-\frac{M }{v }\biggr), & M>\frac{\sigma^2}{v}.
   \end{array}
 \right.
\end{align*}
\end{mydef}

\section{The Empirical Process}
In order to prove the first part of Theorem \ref{th:th_bound} that the bound in  (\ref{eq:bound_low}) has the probability in  (\ref{eq:prob_bound}), we firstly follow \citet{stadler2010} to define the empirical process for fixed data points $\x_1,\ldots,\x_n$. For $\oby_i = (y_{ij}, j\in\Omega_i)\trans\in \mathbb{R}^{|\Omega_i|}$ and $X = (X_1,\ldots,X_d)$, let
\begin{align*}
V_n(\theta) = \frac{1}{n}\sum_{i=1}^n\left(\ell_{\theta}(\x_i,\oby_i)-\mathbb{E}[\ell_{\theta}(\x_i,\oby_i)\mid X=\x_i]\right).
\end{align*}

By fixing some $T\geq 1$ and $\lambda_0\geq 0$, we define an event $\mathcal{T}$ below, upon which the bound in  (\ref{eq:bound_low}) can be proved. So the probability of the event $\mathcal{T}$ is the probability in  (\ref{eq:prob_bound}).
\begin{equation}\label{eq:event_T}\tag{21}
\mathcal{T} = \left\{\sup_{\theta \in \tilde{\Theta}}
\frac{|V_n(\theta)-V_n(\theta_0)|}{(\|\bbeta-\bbeta_0\|_1 + \|\eta-\eta_0\|_2 )\vee \lambda_0}\leq  T\lambda_0
\right\}.
\end{equation}
It can be seen that, (\ref{eq:event_T}) defines a set of the parameter $\theta$, and the bound in (\ref{eq:bound_low}) will be proved with $\hat{\theta}$ in the set.

For group-lasso type estimator, define an event similar to that in (\ref{eq:event_T}) in the following.
\begin{equation}\label{eq:event_T_GS}\tag{22}
\mathcal{T}_{group} = \left\{\sup_{\theta \in \tilde{\Theta}}
\frac{|V_n(\theta)-V_n(\theta_0)|}{(\sum_p\|\bbeta_{\mathcal{G}_p}-\bbeta_{0,\mathcal{G}_p}\|_2 + \|\eta-\eta_0\|_2 )\vee \lambda_0}\leq T\lambda_0
\right\}.
\end{equation}

\section{Lemmas}
In order to show that the probability of event $\mathcal{T}$ is large, we firstly invoke the following lemma.
\begin{mylem}\label{th:px}
Under Condition \ref{th:con_tail}, for model (\ref{eq:exp_fm}) with  $\theta_0 \in \tilde{\Theta}$, $M_n$ and $\lambda_0$ defined in  (\ref{eq:lambda_0}), some constants $c_6,c_7$ depending on $K$, and for $n\geq c_7$, we have
\begin{align*}
  \mathbb{P}_{\X}\left(\frac{1}{n}\sum_{i=1}^nF(\oby_i)>c_6\lambda_0^2/(mk)\right)\leq \frac{1}{n},
\end{align*}
where $\mathbb{P}_{\X}$ denote the conditional probability given $(X_1\trans,\ldots,X_n\trans)\trans=(\x_1\trans,\ldots,\x_n\trans)\trans = \X$, and
$
  F(\oby_i) = G_1(\oby_i)1\{G_1(\oby_i)>M_n\} + \mathbb{E}[G_1(\oby_i)1\{G_1(\oby_i)>M_n\}\mid X=\x_i],\forall i.
$
\end{mylem}
A proof is given in Appendix \ref{sec:proof_px}.

Then we can follow the Corollary 1 in \citet{stadler2010} to show that the probability of event $\mathcal{T}$ is large below.
\begin{mylem}\label{th:pT}
Use Lemma \ref{th:px}. For model (\ref{eq:exp_fm}) with  $\theta_0 \in \tilde{\Theta}$, some constants $c_7,c_8,c_9,c_{10}$ depending on $K$, for $\mathcal{T}$ is defined in (\ref{eq:event_T}), and for all $T\geq c_{10}$ we have
\begin{align*}
\mathbb{P}_{\X}(\mathcal{T})\geq 1 - c_9\exp\left(-\frac{T^2(\log n)^2\log(d\vee n)}{c_8}\right) - \frac{1}{n}, \forall n\geq c_7.
\end{align*}
\end{mylem}
A proof is given in Appendix \ref{sec:proof_pT}.

\section{Corollaries for Models Considering Outlier Samples}\label{sec:outlier}

When considering outlier samples and modifying the natural parameter model as in \eqref{eq:outlier_g}, we can show in this section the similar results.

First, as $\bbeta$ and $\zzeta$ are treated in the similar way, we denote them together by $\xxi\in \mathbb{R}^{((d+n)\times m)\times k}$, and $\xi = vec(\xxi) \in\mathbb{R}^{(d+n)mk}$ such that for all $r = 1,\ldots,k$,
\begin{eqnarray*}
&\vvarphi_r &= \X\bbeta_r + \zzeta_r \  \Rightarrow \vvarphi_r = \A\xxi_r,\\
& \A  & = [\X,  \I_{n}]\in \mathbb{R}^{n\times(d+n)}, \ \xxi_r  = [\bbeta_r\trans,\zzeta_r\trans]\trans
\in \mathbb{R}^{ (d+n)\times m},
\end{eqnarray*}
where $\I_{n}\in \mathbb{R}^{n\times n}$ is a identity matrix.

Thus it can be seen that the modification only results in new design matrix and regression coefficient matrix, therefore, we can apply Theorem \ref{th:th_bound} $\sim$ \ref{th:th_bound_GS} to have similar results for the modified models.

For lasso-type penalties, denote the set of indices of non-zero entries of $\beta_0$ by $S_{\beta}$, and the set of indices of non-zero entries of $\zeta_0$ by $S_{\zeta}$, where $\zeta = \mbox{vec}(\zzeta_1,\ldots,\zzeta_k)$. Denote by $s = |S_{\beta}| + |S_{\zeta}|$. Then for entry-wise $\ell_1$ penalties in \eqref{eq:pen_lasso} (for $\bbeta$) with $\gamma = 0$ and $\mathcal{R}(\zzeta) = \lambda\|\zeta\|_1$ (for $\zzeta$), we need the following modified restricted eigenvalue condition.

\begin{mycon}\label{th:con_REC_outlier}
For all $ \beta \in \mathbb{R}^{dmk}$ and all $ \zeta \in \mathbb{R}^{nmk}$ satisfying $\|\beta_{S_{\beta}^c}\|_1 + \|\zeta_{S_{\zeta}^c}\|_1 \leq 6(\|\beta_{S_{\beta}}\|_1+\|\zeta_{S_{\zeta}}\|_1)$, it holds for some constant $\kappa\geq 1$ that,
\begin{align*}
\|\beta_{S_{\beta}}\|_2^2 + \|\zeta_{S_{\zeta}}\|_2^2 \leq \kappa^2 \|\varphi\|_{Q_n}^2 =  \frac{\kappa^2}{n}\sum_{i=1}^n\sum_{j\in\Omega_i}\sum_{r=1}^k (\x_i\bbeta_{jr}+\zeta_{ijr})^2.
\end{align*}
\end{mycon}

\begin{mycor}\label{th:cor_outlier_l1}
Consider the \methodname\, model in (\ref{eq:exp_fm}) with $\theta_0\in\tilde{\Theta}$, and consider the penalized estimator (\ref{eq:estimator_outlier}) with the $\ell_1$ penalties in (\ref{eq:pen_lasso}) and $\mathcal{R}(\zzeta) = \lambda\|\zeta\|_1$.

\noindent (a) Assume conditions 1-3 and 6 hold. Suppose $\sqrt{mk} \lesssim n/M_n$, and take $\lambda > 2T\lambda_0$ for some constant $T>1$. For some constant $c>0$ and large enough $n$, with probability
$
1 - c\exp\left(-\frac{(\log n)^2\log(d\vee n)}{c}\right) - \frac{1}{n},
$
 we have
\begin{equation*}
 \bar{\varepsilon}(\hat{\theta}\mid \theta_0) + 2(\lambda-T\lambda_0) (\|\hat{\beta}_{S_{\beta}^c}\|_1 + \|\hat{\zeta}_{S_{\zeta}^c}\|_1)
 \leq 4(\lambda+T\lambda_0)^2\kappa^2 c_0^2s,
\end{equation*}

\noindent(b) Assume conditions 1-3 hold (without condition 6), assume
\begin{align*}
 \|\beta_0\|_1 + \|\zeta_0\|_1 &= o(\sqrt{n/((\log n)^{2+2c_1}\log(d\vee n)mk)}),\\
 \sqrt{mk} &= o(\sqrt{n/((\log n)^{2+2c_1}\log(d\vee n))})
\end{align*}
as $n\rightarrow \infty$. If $\lambda = C\sqrt{(\log n)^{2+2c_1}\log(d\vee n)mk/n}$ for some $C>0$ sufficiently large,
and for some constant $c>0$ and large enough $n$, with the following probability
$
1 - c\exp\left(-\frac{ (\log n)^2\log(d\vee n)}{c}\right) - \frac{1}{n},
$
we have
$
\bar{\varepsilon}(\hat{\theta}\mid\theta_0) = o_P(1).
$
\end{mycor}

For group-lasso type penalties, denote
\begin{align*}
 &\mathcal{I}_{\beta} = \{p: \bbeta_{0,\mathcal{G}_{\beta,p}} = \mathbf{0}\}, \ \mathcal{I}_{\beta}^c = \{p: \bbeta_{0,\mathcal{G}_{\beta,p}} \neq \mathbf{0}\},\\
 &\mathcal{I}_{\zeta} = \{q: \zzeta_{0,\mathcal{G}_{\zeta,q}} = \mathbf{0}\}, \ \mathcal{I}_{\zeta}^c = \{q: \zzeta_{0,\mathcal{G}_{\zeta,q}} \neq \mathbf{0}\},
\end{align*}
where $\bbeta_{0,\mathcal{G}_{\beta,p}}$ and $\zzeta_{0,\mathcal{G}_{\zeta,q}}$ denote the $p$th group of $\bbeta_0$ and the $q$th group of $\zzeta_0$, respectively. Now denote $s = |\mathcal{I}_{\beta}| + |\mathcal{I}_{\zeta}|$ with some abuse of notation.

Then for group $\ell_1$ penalties in \eqref{eq:pen_lasso_group_general} (for $\bbeta$) and $\mathcal{R}(\zzeta) = \sum_q^Q\|\zzeta_{\mathcal{G}_{\zeta,q}}\|_F$ (for $\zzeta$), we need the following modified restricted eigenvalue condition.
\begin{mycon}\label{th:con_REC_GS_outlier}
For all $ \bbeta \in \mathbb{R}^{d\times mk}$ and all $ \zzeta \in \mathbb{R}^{n\times mk}$ satisfying
\begin{align*}
\sum_{p\in \mathcal{I}_{\beta}^c}\|\bbeta_{\mathcal{G}_{\beta,p}}\|_F + \sum_{q\in \mathcal{I}_{\zeta}^c}\|\zzeta_{\mathcal{G}_{\zeta,q}}\|_F
\leq 6\left(\sum_{p\in \mathcal{I}_{\beta}}\|\bbeta_{\mathcal{G}_{\beta,p}}\|_F + \sum_{q\in \mathcal{I}_{\zeta}}\|\zzeta_{\mathcal{G}_{\zeta,q}}\|_F \right),
\end{align*}
it holds that for some constant $\kappa\geq 1$,
\begin{align*}
 \sum_{p\in \mathcal{I}_{\beta}}\|\bbeta_{\mathcal{G}_{\beta,p}}\|_F^2 + \sum_{q\in \mathcal{I}_{\zeta}}\|\zzeta_{\mathcal{G}_{\zeta,q}}\|_F^2
 \leq \kappa^2 \|\varphi\|_{Q_n}^2 =\frac{\kappa^2}{n}\sum_{i=1}^n\sum_{j\in\Omega_i}\sum_{r=1}^k (\x_i\bbeta_{jr}+\zeta_{ijr})^2.
\end{align*}
\end{mycon}

\begin{mycor}\label{th:cor_outlier_l1_group}
Consider the \methodname\, model in (\ref{eq:exp_fm}) with $\theta_0\in\tilde{\Theta}$, and consider estimator (\ref{eq:estimator_outlier}) with the group $\ell_1$ penalties in (\ref{eq:pen_lasso_group_general}) and $\mathcal{R}(\zzeta) = \sum_q^Q\|\zzeta_{\mathcal{G}_{\zeta,q}}\|_F$.

\noindent (a) Assume conditions 1-3 and 7 hold. Suppose $\sqrt{mk} \lesssim n/M_n$, and take $\lambda > 2T\lambda_0$ for some constant $T>1$. For some constant $c>0$ and large enough $n$, with probability
$
1 - c\exp\left(-\frac{(\log n)^2\log(d\vee n)}{c}\right) - \frac{1}{n},
$
 we have
\begin{align*}
 \bar{\varepsilon}(\hat{\theta}\mid\theta_0) + 2(\lambda-T\lambda_0)\biggl(\sum_{p\in\mathcal{I}_{\beta}^c}\|\hat{\bbeta}_{\mathcal{G}_{\beta,p}}\|_F+\sum_{q\in\mathcal{I}_{\zeta}^c}\|\hat{\zzeta}_{\mathcal{G}_{\zeta,q}}\|_F\biggr)
 \leq 4(\lambda+T\lambda_0)^2\kappa^2 c_0^2s,
\end{align*}

\noindent(b) Assume conditions 1-3 hold (without condition 7), assume
\begin{align*}
 \sum_{p=1}^P\|\bbeta_{0,\mathcal{G}_{\beta,p}}\|_F + \sum_{q=1}^Q\|\zzeta_{0,\mathcal{G}_{\zeta,q}}\|_F & = o(\sqrt{n/((\log n)^{2+2c_1}\log(d\vee n)mk)}),\\
  \sqrt{mk} &= o(\sqrt{n/((\log n)^{2+2c_1}\log(d\vee n))})
\end{align*}
as $n\rightarrow \infty$. If $\lambda = C\sqrt{(\log n)^{2+2c_1}\log(d\vee n)mk/n}$ for some $C>0$ sufficiently large,
and for some constant $c>0$ and large enough $n$, with the following probability
$
1 - c\exp\left(-\frac{ (\log n)^2\log(d\vee n)}{c}\right) - \frac{1}{n},
$
we have
$
\bar{\varepsilon}(\hat{\theta}\mid\theta_0) = o_P(1).
$
\end{mycor}

\section{Proof of Lemma \ref{th:lem_subexp_tail}}\label{sec:condition}
\begin{proof}

For non-negative continuous variable $X$, we have
\begin{align*}
 \mathbb{E}[X1\{X>M\}] & = \int_M^{\infty}tf_X(t)dt = \int_M^{\infty}\int_0^tf_X(t)dxdt \nonumber \\
& = \int_0^M\int_M^{\infty}f_X(t)dtdx + \int_M^{\infty}\int_x^{\infty}f_X(t)dtdx \nonumber \\
&= M\mathbb{P}(X>M) + \int_M^{\infty}\mathbb{P}(X>x)dx.
\end{align*}

Similarly, we have
$
 \mathbb{E}[X^21\{X>M\}]   = M^2\mathbb{P}(X>M) + \int_M^{\infty}2x\mathbb{P}(X>x)dx.
$

For $X$ sub-exponential with parameters $(\sigma ,v ,t) $ such that for   $M>t $
 \begin{align*}
 \mathbb{P}(X>M)\leq
 \left\{
   \begin{array}{ll}
     \exp\biggl(-\frac{M^2}{\sigma^2}\biggr), & t \leq M\leq \frac{\sigma^2}{v}\\
     \exp\biggl(-\frac{M }{v  }\biggr), & M\geq\frac{\sigma^2}{v},
   \end{array}
 \right.
\end{align*}
we have the following.

If $M\leq  \frac{\sigma^2}{v} $, we have
\begin{align*}
 \mathbb{E}[X1\{X>M\}] &= M\mathbb{P}(X>M) + \int_M^{\infty}\mathbb{P}(X>x)dx\\
& \leq M\exp\biggl(-\frac{M^2}{\sigma^2}\biggr) + \int_M^{\frac{\sigma^2}{v}}\exp\biggl(-\frac{x^2}{\sigma^2}\biggr)dx + \int_{\frac{\sigma^2}{v}}^{\infty}\exp\biggl(-\frac{x }{v  }\biggr)dx\\
& \leq M\exp\biggl(-\frac{M^2}{\sigma^2}\biggr) + (\frac{\sigma^2}{v} - M)\exp\biggl(-\frac{M^2}{\sigma^2}\biggr) + v\exp\biggl(-\frac{M}{v}\biggr)\\
& =  M \exp\biggl(-\frac{M^2}{\sigma^2}\biggr) + v\exp\biggl(-\frac{M}{v}\biggr)\leq (M+v) \exp\biggl(-\frac{M^2}{\sigma^2}\biggr),
\end{align*}
and similarly, $\mathbb{E}[X^21\{X>M\}] \leq \biggl(M^2+ 2v^2+2\sigma^2\biggr)\exp\biggl(-\frac{M^2}{\sigma^2}\biggr).$

If $M> \frac{\sigma^2}{v} $, we have $\mathbb{E}[X1\{X>M\}] \leq (M+v)\exp\biggl(-\frac{M }{v  }\biggr)$ and $\mathbb{E}[X^21\{X>M\}] \leq (M^2+2v^2+2vM)\exp\biggl(-\frac{M }{v  }\biggr)$.

%

Then for some constants $c_1,c_2,c_3,c_4,c_5>0$, for non-negative continuous variable $X$ which is sub-exponential with parameters $(\sigma,v,t)$, for $M>c_4>t$ and $c' = 2+\frac{3}{c_1}$, we have
\begin{align*}
 &\mathbb{E}[X1\{X>M\}] \leq \biggl[ c_3\biggl(\frac{M}{c_2}\biggr)^{c'}+ c_5 \biggr]\exp\biggl\{-\biggl(\frac{M}{c_2}\biggr)^{1/c_1}\biggr\},\\
 &\mathbb{E}[X^21\{X>M\}] \leq \biggl[ c_3\biggl(\frac{M}{c_2}\biggr)^{c'}+ c_5 \biggr]^2\exp\biggl\{-2\biggl(\frac{M}{c_2}\biggr)^{1/c_1}\biggr\}.
\end{align*}

If $t \leq M\leq \frac{\sigma^2}{v}$, $c_1 =1/2, c_2 = \sqrt{2}\sigma, c_3 = 16\sigma^8$. And if $M\geq\frac{\sigma^2}{v}$, $c_1 = 1,c_2 =  2v,c_3 = 32v^5$. And $c_5 = \sqrt{2}(v + \sigma)$.

For non-negative discrete variables, the result is the same.

The result of Lemma \ref{th:lem_subexp_tail} follows from the result above, $\oby_i$ has a finite mixture distribution for $i=1,\ldots,n$ and the following.

When dispersion parameter $\phi$ is known, for a constant $c_K$ depending on $K$, we have
\begin{equation*}
  G_1(\oby_i) = e^K \max_{j\in\Omega_i}|y_{ij}| + c_K, \ i=1,\ldots,n.
\end{equation*}

\end{proof}

\section{Proof of Lemma \ref{th:px}}\label{sec:proof_px}
\begin{proof}
Under Condition \ref{th:con_tail}, $M_n = c_2(\log n)^{c_1}$, and $\lambda_0$ defined in  (\ref{eq:lambda_0}), for a constant $c_6$  depending on $K$, for $i=1,\ldots,n$, we have
\begin{align*}
&\mathbb{E}[|G_1(\oby_i)|1\{|G_1(\oby_i)|>M_n\}] \leq   c_6\lambda_0^2/(mk), \\
&\mathbb{E}[|G_1(\oby_i)|^21\{|G_1(\oby_i)|>M_n\}] \leq c_6^2\lambda_0^4/(mk)^2.
\end{align*}
The we can get
\begin{align*}
  &\mathbb{P}_{\X}\biggl(\frac{1}{n}\sum_{i=1}^nG_1(\oby_i)1\{G_1(\oby_i)>M_n\} + \mathbb{E}[G_1(\oby_i)1\{G_1(\oby_i)>M_n\}]>3c_6\lambda_0^2/(mk) \biggr)\\
  &\leq \mathbb{P}_{\X}\biggl(\frac{1}{n}\sum_{i=1}^nG_1(\oby_i)1\{G_1(\oby_i)>M_n\} - \mathbb{E}[G_1(\oby_i)1\{G_1(\oby_i)>M_n\}]>c_6\lambda_0^2/(mk) \biggr)\\
  &\leq \frac{\mathbb{E}[|G_1(\oby_i)|^21\{|G_1(\oby_i)|>M_n\}]}{n}\frac{m^2k^2}{c_6^2\lambda_0^4} \leq \frac{1}{n}.
\end{align*}
\end{proof}

\section{Proof of Lemma \ref{th:pT}}\label{sec:proof_pT}
\begin{proof}
We follow \citet{stadler2010} to give a Entropy Lemma and then prove Lemma \ref{th:pT}.

We use the following norm $\|\cdot\|_{P_n}$ introduced in the Proof of Lemma 2 in \citet{stadler2010} and use $H(\cdot,\mathcal{H},\|\cdot\|_{P_n})$ as the entropy of covering number (see \citet{van2000applications}) which is equipped the metric induced by the norm for a collection $\mathcal{H}$ of functions on $\mathcal{X}\times \mathcal{Y}$,
\begin{align*}
   \|h(\cdot,\cdot)\|_{P_n} = \sqrt{\frac{1}{n}\sum_{i=1}^nh^2(\x_i,\oby_i)}.
\end{align*}
Define $
  \tilde{\Theta}(\epsilon) = \{\theta\in \tilde{\Theta}: \|\bbeta-\bbeta_0\|_1 + \|\eta - \eta_0\|_2 \leq \epsilon  \}$.

\begin{mylem}
(Entropy Lemma) For a constant $c_{12}>0$, for all $u>0$ and $M_n>0$, we have
 \begin{align*}
  H\biggl(u,\biggl\{(\ell_{\theta} - \ell_{\theta^{\star}})1\{G_1\leq M_n\}: \theta\in\tilde{\Theta}(\epsilon)\biggr\},
  \|\cdot\|_{P_n}\biggr)
  \leq c_{12}\frac{mk\epsilon^2M_n^2}{u^2}\log\biggl(\frac{\sqrt{mk}\epsilon M_n}{u}\biggr).
\end{align*}
\end{mylem}

\begin{proof}
(For Entropy Lemma) The difference between this proof and that of Entropy Lemma in the proof of Lemma 2 of \citet{stadler2010} is in the notations and the effect of multivariate responses.

For multivariate responses we have for $i=1,\ldots,n$,
 \begin{align*}
   |\ell_{\theta}(\x_i,\oby_i) - \ell_{\theta'} (\x_i,\oby_i)|^2 &\leq G_1^2(\oby_i)\|\psi_i - \psi'_i \|_1^2
    \leq d_{\psi}G_1^2(\oby_i)\|\psi_i- \psi'_i\|_2^2\\
   & = d_{\psi}G_1^2(\oby_i) \biggl[\sum_{r=1}^k\sum_{j\in\Omega_i}|\x_i(\bbeta_{rj}- \bbeta'_{rj}) |^2
   +\|\eta -  \eta'\|_2^2\biggr],
\end{align*}
where $d_{\psi} = (2m+1)k$ is the maximum of dimension of $\psi_i$ for $i=1,\ldots,n$.

Under the definition of the norm $\|\cdot\|_{P_n}$ we have
 \begin{align*}
   \|(\ell_{\theta} - \ell_{\theta'})1\{G_1\leq M_n\}\|_{P_n}^2
   \leq d_{\psi}M_n^2\left[ \frac{1}{n}\sum_{i=1}^n\sum_{r=1}^k\sum_{j\in\Omega_i}|\x_i(\bbeta_{rj}- \bbeta'_{rj}) |^2 + \|\eta - \eta'\|_2^2   \right].
\end{align*}
Then by the result of \citet{stadler2010} we have
 \begin{align*}
   H (u,\{\eta\in\mathbb{R}^{d_{\eta}}: \|\eta-\eta_0\|_2\leq \epsilon\},\|\cdot\|_2 )\leq d_{\eta}\log\biggl(\frac{5\epsilon}{u}\biggr),
\end{align*}
where $d_{\eta} = (m+1)k$ is the dimension of $\eta$.

And we follow \citet{stadler2010} to apply Lemma 2.6.11 of \citet{van1996weak} to give a bound as
 \begin{align*}
  & H \biggl(2u,\biggl\{ \sum_{r=1}^k\sum_{j\in\Omega_i}\x_i(\bbeta_{rj}- {\bbeta}_{0,rj}) :  \|\bbeta-\bbeta_0\|_1\leq\epsilon \biggr\},
   \|\cdot\|_{P_n}\biggr)    \leq \biggl(\frac{\epsilon^2}{u^2}+1\biggr)\log(1+kmd).
\end{align*}

Thus we can get
 \begin{align*}
   &H \biggl(3\sqrt{d_{\psi}}M_nu,\biggl\{ (\ell_{\theta} - \ell_{\theta_0})1\{G_1\leq M_n\} : \theta\in\tilde{\Theta}(\epsilon) \biggr\},
   \|\cdot\|_{P_n}\biggr) \\
   &\leq \biggl(\frac{\epsilon^2}{u^2}+1+d_{\eta}\biggr)\biggl(\log(1+kmd)+\log\biggl(\frac{5\epsilon}{u}\biggr)\biggr).
\end{align*}
\end{proof}

Now we turn to prove Lemma \ref{th:pT}.

We follow \citet{stadler2010} to use the truncated version of the empirical process below.
 \begin{align*}
    V_n^{trunc}(\theta) = \frac{1}{n}\sum_{i=1}^n\biggl(
    \ell_{\theta}(\x_i,\oby_i)1\{G_1(\oby_i)\leq M_n\} - \mathbb{E}[\ell_{\theta}(\x_i,\oby_i)1\{G_1(\oby_i)\leq M_n\}\mid X=\x_i].
    \biggr)
\end{align*}
We follow \citet{stadler2010} to apply the Lemma 3.2 in \citet{van2000applications} and a conditional version of Lemma 3.3 in \citet{van2000applications} to the class
 \begin{align*}
    \biggl\{ (\ell_{\theta} - \ell_{\theta_0})1\{G_1\leq M_n\} : \theta\in\tilde{\Theta}(\epsilon) \biggr\}, \forall \epsilon>0.
\end{align*}
For some constants $\{c_{t}\}_{t>12}$  depending on $K$  and $\Lambda_{\max}$ in Condition 2 of \citet{stadler2010}, using the notation of Lemma 3.2 in \citet{van2000applications}, we follow \citet{stadler2010} to choose $\delta = c_{13} T\epsilon \lambda_0$ and $R = c_{14}(\sqrt{mk}\epsilon \wedge 1)M_n$.

Thus we by choosing $M_n = c_2(\log n)^{c_1}$ we can satisfy the condition of Lemma 3.2 of \citet{van2000applications} to have
 \begin{align*}
    & \int_{\epsilon/c_{15}}^R H^{1/2} \biggl(u,\biggl\{(\ell_{\theta} - \ell_{\theta^{\star}})1\{G_1\leq M_n\}: \theta\in\tilde{\Theta}(\epsilon)\biggr\},
  \|\cdot\|_{P_n}\biggr) du \vee R \\
  =&\int_{\epsilon/c_{15}}^{c_{14}\sqrt{mk}(\epsilon \wedge 1)M_n}
  c_{12}\biggl(\frac{\sqrt{mk}\epsilon M_n}{u}\biggr)\log^{1/2}\biggl(\frac{\sqrt{mk}\epsilon M_n}{u}\biggr)du \vee (c_{14}(\epsilon \wedge 1)M_n)\\
    \leq& \frac{2}{3}c_{12}\sqrt{mk}\epsilon M_n
[ \log^{3/2} (c_{15}\sqrt{mk}M_n) - \log^{3/2} (\frac{\sqrt{mk}\epsilon M_n}{c_{14}\sqrt{mk}(\epsilon \wedge 1)M_n}) ]\vee (c_{14}\sqrt{mk}(\epsilon \wedge 1)M_n)   \\
\leq& \frac{2}{3}c_{12}\sqrt{mk}\epsilon M_n\log^{3/2} (c_{15}\sqrt{mk}M_n)\\
\leq& c_{16} \sqrt{mk}\epsilon M_n\log^{3/2} (n) \quad (\mbox{by choosing} \ M_n = c_2(\log n)^{c_1}, \mbox{and} \ \sqrt{mk} \leq   c_{17}\frac{n}{M_n}) \\
\leq&  c_{18}  \sqrt{n} T\epsilon \lambda_0\leq \sqrt{n}(\delta - \epsilon).
\end{align*}
Now for the rest we can apply Lemma 3.2 of \citet{van2000applications} to give the same result with Lemma 2 of \citet{stadler2010}.

So we have
\begin{align*}
 \sup_{\theta \in \tilde{\Theta}}
\frac{|V_n^{trunc}(\theta) - V_n^{trunc}(\theta_0)|}{(\|\bbeta-\bbeta_0\|_1 + \|\eta-\eta_0\|_2 )\vee \lambda_0}
\leq 2c_{23}T \lambda_0
\end{align*}
with probability at least $
 1 - c_{9}\exp\biggl[- \frac{T^2(\log n)^2\log(d\vee n) }{c_{8}^2}\biggr].$

At last, for the case when $G_1(\oby_i)>M_n$, for $i=1,\ldots,n$, we have
\begin{align*}
 | (\ell_{\theta}(\x_i,\oby_i) - \ell_{\theta_0}(\x_i,\oby_i))1\{G_1(\oby_i)> M_n\} |\leq d_{\psi}KG_1(\oby_i)1\{G_1(\oby_i)> M_n\},
\end{align*}
and
\begin{align*}
 &\frac{|(V_n^{trunc}(\theta) - V_n^{trunc}(\theta_0)) -(V_n(\theta)-V_n(\theta_0)) |}
 {(\|\bbeta-\bbeta_0\|_1 + \|\eta-\eta_0\|_2 )\vee \lambda_0}\\
 &\leq \frac{d_{\psi}K}{n\lambda_0}\sum_{i=1}^n
\biggl( G_1(\oby_i)1\{G_1(\oby_i)>M_n\} + \mathbb{E}[G_1(\oby_i)1\{G_1(\oby_i)>M_n\}\mid X=\x_i]
\biggr).
\end{align*}
Then the probability of the following inequality under our model is given in Lemma \ref{th:px}.
\begin{align*}
  \frac{d_{\psi}K}{n\lambda_0}\sum_{i=1}^n
\biggl( G_1(\oby_i)1\{G_1(\oby_i)>M_n\} + \mathbb{E}[G_1(\oby_i)1\{G_1(\oby_i)>M_n\}\mid X=\x_i]
\biggr) \leq  c_{23}T \lambda_0,
\end{align*}
where $d_{\psi} = 2(m+1)k$.
\end{proof}

\section{Proof of Theorem \ref{th:th_bound}}
\begin{proof}
This proof mostly follows that of Theorem 3 of \citet{stadler2010}. The only difference is in the notations. As such, we omit the details.
\end{proof}

\section{Proof of Theorem \ref{th:high_dim}}\label{sec:app_proof_high_dim}
\begin{proof}
This proof also mostly follows that of Theorem 5 of \citet{stadler2010}. The difference is in the notations and the choice of $M_n$.

If the event $\mathcal{T}$ happens, with $M_n = c_2(\log n)^{c_1}$ for some constants $0\leq c_1,c_2<\infty$, where $c_2$ depends on $K$,
$$\lambda_0 =   \sqrt{mk} M_n\log n\sqrt{\log(d\vee n)/n} = c_2\sqrt{mk\log^{2+2c_1}\log(d\vee n)/n},$$
we have
\begin{align*}
 \bar{\varepsilon}(\hat{\psi}\mid \psi_0) + \lambda\|\hat{\beta}\|_1
 \leq T\lambda_0[(\|\hat{\beta}-\beta_0\|_1 + \|\eta-\eta_0\|_2 )\vee \lambda_0] + \lambda\|\beta_0\|_1 + \bar{\varepsilon}(\psi_0\mid\psi_0).
\end{align*}
Under the definition of $\theta \in \tilde{\Theta}$ in (\ref{eq:tTheta}) we have $\|\eta-\eta_0\|_2\leq 2K$. And as $ \bar{\varepsilon}(\psi_0\mid\psi_0) =0$ we have for $n$ sufficiently large.
\begin{align*}
 &\bar{\varepsilon}(\hat{\psi}\mid \psi_0) + \lambda\|\hat{\beta}\|_1
 \leq T\lambda_0(\|\hat{\beta}\|_1 +\|\beta_0\|_1 + 2K ) + \lambda\|\beta_0\|_1\\
 &\rightarrow \bar{\varepsilon}(\hat{\psi}\mid \psi_0) + (\lambda-T\lambda_0)\|\hat{\beta}\|_1
 \leq T\lambda_0 2K   + (\lambda+T\lambda_0)\|\beta_0\|_1
\end{align*}
As $C>0$ sufficiently large we have $\lambda\geq 2T\lambda_0$.

And using the condition on $\|\beta_0\|_1$ and $\sqrt{mk}$, we have both $T\lambda_02K = o(1)$ and $(\lambda+T\lambda_0)\|\beta_0\|_1 = o(1)$, so we have $\bar{\varepsilon}(\hat{\psi}\mid \psi_0)\rightarrow 0 \ (n\rightarrow \infty)$.

At last, as the event $\mathcal{T}$ has large probability, we have
$
\bar{\varepsilon}(\hat{\theta}_{\lambda}\mid\theta_0) = o_P(1) \ (n\rightarrow \infty).
$
\end{proof}

\section{Proof of Theorem \ref{th:th_bound_GS} }

\begin{proof}
First we discuss the bound for the probability of $\mathcal{T}_{group}$ in \eqref{eq:event_T_GS}.

The difference between $\mathcal{T}_{group}$ and $\mathcal{T}$ in  (\ref{eq:event_T}) is only related to the following entropy of the Entropy Lemma in the proof of Lemma \ref{th:pT}.
 \begin{align*}
    H \biggl(2u,\biggl\{ \sum_{r=1}^k\sum_{j\in\Omega_i}\x_i(\bbeta_{rj}- {\bbeta}_{0,rj})  : \sum_p\|\bbeta_{\mathcal{G}_p}-\bbeta_{0,\mathcal{G}_p}\|_F\leq\epsilon \biggr\},
   \|\cdot\|_{P_n}\biggr) , \ \mbox{for} \ i = 1\ldots,n,
\end{align*}
where $\sum_p\|\bbeta_{\mathcal{G}_p}-\bbeta_{0,\mathcal{G}_p}\|_F\leq\epsilon$ still maintains a convex hull for $\bbeta$ in the metric space equipped with the metric induced by the norm $\|\cdot\|_{P_n}$ defined in the proof of Lemma \ref{th:pT}.  Thus it still satisfies the condition of Lemma 2.6.11 of \citet{van1996weak} which can still be applied to give
\begin{align*}
  & H \biggl(2u,\biggl\{ \sum_{r=1}^k\sum_{j\in\Omega_i}\x_i(\bbeta_{rj}- {\bbeta}_{0,rj}) : \sum_p\|\bbeta_{\mathcal{G}_p}-\bbeta_{0,\mathcal{G}_p}\|_F\leq\epsilon \biggr\},
   \|\cdot\|_{P_n}\biggr) \\
  & \leq \biggl(\frac{\epsilon^2}{u^2}+1\biggr)\log(1+kmd), \ \mbox{for} \ i = 1\ldots,n.
\end{align*}
So the probability of event $\mathcal{T}_{group}$ remains the same form with that in Lemma \ref{th:pT}.

Then we discuss the bound for the average excess risk and feature selection.

If the event $\mathcal{T}_{group}$ happens, we have
\begin{align*}
 \bar{\varepsilon}(\hat{\psi}\mid \psi_0) + \lambda\sum_p\|\hat{\bbeta}_{\mathcal{G}_p}\|_F
 &\leq T\lambda_0\biggl[\biggl(\sum_{\mathcal{G}_p}\|\hat{\bbeta}_{\mathcal{G}_p}-\bbeta_{0,\mathcal{G}_p}\|_F + \|\eta-\eta_0\|_2 \biggr)\vee \lambda_0\biggr]\\
  &+ \lambda\sum_p\| \bbeta_{0,\mathcal{G}_p}\|_F + \bar{\varepsilon}(\psi_0\mid\psi_0).
\end{align*}
Using Condition \ref{th:con_fisher}  we have $ \bar{\varepsilon}(\psi_0\mid\psi_0) =0$ and
$
 \bar{\varepsilon}(\hat{\psi}\mid \psi_0) \geq  {\|\hat{\psi}-\psi_0\|_{Q_n}^2}/{c_0^2}.
$

\noindent\textbf{Case 1} When the following is true:
\begin{align*}
\sum_p\|\hat{\bbeta}_{\mathcal{G}_p}-\bbeta_{0,\mathcal{G}_p}\|_F + \|\hat{\eta}-\eta_0\|_2 \leq \lambda_0,
\end{align*}
we have
\begin{align*}
 \bar{\varepsilon}(\hat{\psi}\mid \psi_0) &\leq T\lambda_0^2 + \lambda\sum_p\|\hat{\bbeta}_{\mathcal{G}_p}-\bbeta_{0,\mathcal{G}_p}\|_F  + \bar{\varepsilon}(\psi_0\mid\psi_0) \leq (\lambda+T\lambda_0)\lambda_0.
\end{align*}
\textbf{Case 2} When the following is true:
\begin{align*}
&\sum_p\|\hat{\bbeta}_{\mathcal{G}_p}-\bbeta_{0,\mathcal{G}_p}\|_F + \|\hat{\eta}-\eta_0\|_2 \geq \lambda_0,\\
& T\lambda_0\|\hat{\eta}-\eta_0\|_2 \geq (\lambda+T\lambda_0)\sum_{p\in\mathcal{I}}\|\hat{\bbeta}_{\mathcal{G}_p}-\bbeta_{0,\mathcal{G}_p}\|_F.
\end{align*}
As $\sum_{p\in\mathcal{I}^c}\|\bbeta_{0,\mathcal{G}_p}\|_F=0$, we have
\begin{align*}
 &\bar{\varepsilon}(\hat{\psi}\mid \psi_0) + (\lambda-T\lambda_0)\sum_{p\in\mathcal{I}^c}\|\hat{\bbeta}_{\mathcal{G}_p}\|_F  \leq 2T\lambda_0\|\hat{\eta}-\eta_0\|_2\\
&\leq 2T^2\lambda_0^2c_0^2 + \|\hat{\eta}-\eta_0\|_2^2/(2c_0^2) \leq 2T^2\lambda_0^2c_0^2 + \bar{\varepsilon}(\hat{\psi}\mid \psi_0)/2.
\end{align*}
Then we get
\begin{align*}
\bar{\varepsilon}(\hat{\psi}\mid \psi_0) + 2(\lambda-T\lambda_0)\sum_{p\in\mathcal{I}^c}\|\hat{\bbeta}_{\mathcal{G}_p}\|_F\leq 4T^2\lambda_0^2c_0^2.
\end{align*}
\textbf{Case 3} When the following is true:
\begin{align*}
&\sum_p\|\hat{\bbeta}_{\mathcal{G}_p}-\bbeta_{0,\mathcal{G}_p}\|_F + \|\hat{\eta}-\eta_0\|_2 \geq \lambda_0,\\
&T\lambda_0\|\hat{\eta}-\eta_0\|_2 \leq (\lambda+T\lambda_0)\sum_{p\in\mathcal{I}}\|\hat{\bbeta}_{\mathcal{G}_p}-\bbeta_{0,\mathcal{G}_p}\|_F,
\end{align*}
we have
\begin{align*}
&\bar{\varepsilon}(\hat{\psi}\mid \psi_0) + (\lambda-T\lambda_0)\sum_{p\in\mathcal{I}^c}\|\hat{\bbeta}_{\mathcal{G}_p}\|_F \leq 2(\lambda+T\lambda_0)\sum_{p\in\mathcal{I}}\|\hat{\bbeta}_{\mathcal{G}_p}-\bbeta_{0,\mathcal{G}_p}\|_F.
\end{align*}
Thus we have
\begin{align*}
\sum_{p\in\mathcal{I}^c}\|\hat{\bbeta}_{\mathcal{G}_p}\|_F \leq 6 \sum_{p\in\mathcal{I}}\|\hat{\bbeta}_{\mathcal{G}_p}-\bbeta_{0,\mathcal{G}_p}\|_F,
\end{align*}
so we can use the Condition \ref{th:con_REC_GS} for $\hat{\bbeta} -\bbeta_0$ to have
\begin{align*}
&\bar{\varepsilon}(\hat{\psi}\mid \psi_0) + (\lambda-T\lambda_0)\sum_{p\in\mathcal{I}^c}\|\hat{\bbeta}_{\mathcal{G}_p}\|_F    \leq 2(\lambda+T\lambda_0)\sqrt{s}\sum_{p\in\mathcal{I}}\|\hat{\bbeta}_{\mathcal{G}_p}-\hat{\bbeta}_{0,\mathcal{G}_p}\|_F\\
&\leq 2(\lambda+T\lambda_0)\sqrt{s}\kappa\|\hat{\varphi}-(\varphi_0)\|_{Q_n} \leq 2(\lambda+T\lambda_0)^2 s \kappa^2 c_0^2 + \|\hat{\varphi}-(\varphi_0)\|_{Q_n}^2/(2c_0^2)\\
&\leq 2(\lambda+T\lambda_0)^2 s \kappa^2 c_0^2 +\bar{\varepsilon}(\hat{\psi}\mid \psi_0)/2.
\end{align*}
So we have
\begin{align*}
\bar{\varepsilon}(\hat{\psi}\mid \psi_0) + 2(\lambda-T\lambda_0)\sum_{p\in\mathcal{I}^c}\|\hat{\bbeta}_{\mathcal{G}_p}\|_F\leq 4(\lambda+T\lambda_0)^2s \kappa^2 c_0^2.
\end{align*}

And without restricted eigenvalue condition \ref{th:con_REC_GS},
we can prove similarly as in Appendix~\ref{sec:app_proof_high_dim}, assuming event $\mathcal{T}_{group}$ happens and using the condition on $\sum_p\|\bbeta_{0,\mathcal{G}_p}\|_F$ and $\sqrt{mk}$.
\end{proof}

\newpage
\vskip 0.2in
\bibliographystyle{spbasic}      
\bibliography{ref}

\begin{thebibliography}{68}
\providecommand{\natexlab}[1]{#1}
\providecommand{\url}[1]{{#1}}
\providecommand{\urlprefix}{URL }
\expandafter\ifx\csname urlstyle\endcsname\relax
  \providecommand{\doi}[1]{DOI~\discretionary{}{}{}#1}\else
  \providecommand{\doi}{DOI~\discretionary{}{}{}\begingroup
  \urlstyle{rm}\Url}\fi
\providecommand{\eprint}[2][]{\url{#2}}

\bibitem[{Aho et~al(2014)Aho, Derryberry, and Peterson}]{aho2014model}
Aho K, Derryberry D, Peterson T (2014) {Model selection for ecologists: the
  worldviews of AIC and BIC}. Ecology 95(3):631--636

\bibitem[{Alf{\`o} et~al(2016)Alf{\`o}, Salvati, and Ranallli}]{alfo2016finite}
Alf{\`o} M, Salvati N, Ranallli MG (2016) {Finite mixtures of quantile and
  M-quantile regression models}. Statistics and Computing pp 1--24

\bibitem[{Argyriou et~al(2007{\natexlab{a}})Argyriou, Evgeniou, and
  Pontil}]{evgeniou2007multi}
Argyriou A, Evgeniou T, Pontil M (2007{\natexlab{a}}) Multi-task feature
  learning. In: Advances in Neural Information Processing Systems, pp 41--48

\bibitem[{Argyriou et~al(2007{\natexlab{b}})Argyriou, Pontil, Ying, and
  Micchelli}]{argyriou2007spectral}
Argyriou A, Pontil M, Ying Y, Micchelli CA (2007{\natexlab{b}}) A spectral
  regularization framework for multi-task structure learning. In: Advances in
  Neural Information Processing Systems, pp 25--32

\bibitem[{Bai et~al(2016)Bai, Chen, and Yao}]{bai2016mixture}
Bai X, Chen K, Yao W (2016) Mixture of linear mixed models using multivariate t
  distribution. Journal of Statistical Computation and Simulation
  86(4):771--787

\bibitem[{Bartolucci and Scaccia(2005)}]{bartolucci2005use}
Bartolucci F, Scaccia L (2005) The use of mixtures for dealing with non-normal
  regression errors. Computational Statistics \& Data Analysis 48(4):821--834

\bibitem[{Barzilai and Borwein(1988)}]{barzilai1988two}
Barzilai J, Borwein JM (1988) Two-point step size gradient methods. IMA Journal
  of Numerical Analysis 8(1):141--148

\bibitem[{Becker et~al(2011)Becker, Cand{\`e}s, and
  Grant}]{becker2011templates}
Becker SR, Cand{\`e}s EJ, Grant MC (2011) Templates for convex cone problems
  with applications to sparse signal recovery. Mathematical Programming
  Computation 3(3):165--218

\bibitem[{Bhat and Kumar(2010)}]{bhat2010derivation}
Bhat HS, Kumar N (2010) {On the derivation of the Bayesian Information
  Criterion}. School of Natural Sciences, University of California

\bibitem[{Bickel et~al(2009)Bickel, Ritov, and
  Tsybakov}]{bickel2009simultaneous}
Bickel PJ, Ritov Y, Tsybakov AB (2009) {Simultaneous analysis of Lasso and
  Dantzig selector}. The Annals of Statistics pp 1705--1732

\bibitem[{Bishop(2006)}]{bishop2006pattern}
Bishop CM (2006) Pattern recognition. Machine Learning 128:1--58

\bibitem[{Boyd et~al(2011)Boyd, Parikh, Chu, Peleato, and
  Eckstein}]{boyd2011distributed}
Boyd S, Parikh N, Chu E, Peleato B, Eckstein J (2011) Distributed optimization
  and statistical learning via the alternating direction method of multipliers.
  Foundations and Trends{\textregistered} in Machine Learning 3(1):1--122

\bibitem[{Cand{\`e}s and Recht(2009)}]{candes2009}
Cand{\`e}s EJ, Recht B (2009) Exact matrix completion via convex optimization.
  Found Comput Math 9(6):717--772

\bibitem[{Chen et~al(2011)Chen, Zhou, and Ye}]{chen2011integrating}
Chen J, Zhou J, Ye J (2011) Integrating low-rank and group-sparse structures
  for robust multi-task learning. In: Proceedings of the 17th ACM SIGKDD
  International Conference on Knowledge Discovery and Data Mining, ACM, pp
  42--50

\bibitem[{Chen et~al(2012{\natexlab{a}})Chen, Liu, and Ye}]{chen2012learning}
Chen J, Liu J, Ye J (2012{\natexlab{a}}) Learning incoherent sparse and
  low-rank patterns from multiple tasks. ACM Transactions on Knowledge
  Discovery from Data (TKDD) 5(4):22

\bibitem[{Chen et~al(2012{\natexlab{b}})Chen, Chan, and
  Stenseth}]{chen2012reduced}
Chen K, Chan KS, Stenseth NC (2012{\natexlab{b}}) Reduced rank stochastic
  regression with a sparse singular value decomposition. Journal of the Royal
  Statistical Society: Series B (Statistical Methodology) 74(2):203--221

\bibitem[{Chen et~al(2010)Chen, Kim, Lin, Carbonell, and Xing}]{chen2010graph}
Chen X, Kim S, Lin Q, Carbonell JG, Xing EP (2010) Graph-structured multi-task
  regression and an efficient optimization method for general fused lasso.
  arXiv preprint arXiv:10053579

\bibitem[{Cover and Thomas(2012)}]{cover2012elements}
Cover TM, Thomas JA (2012) Elements of information theory. John Wiley \& Sons

\bibitem[{Dempster et~al(1977)Dempster, Laird, and Rubin}]{dempster1977maximum}
Dempster AP, Laird NM, Rubin DB (1977) {Maximum likelihood from incomplete data
  via the EM algorithm}. Journal of the Royal Statistical Society Series B
  (methodological) pp 1--38

\bibitem[{Do{\u{g}}ru and Arslan(2016)}]{dougru2016robust}
Do{\u{g}}ru FZ, Arslan O (2016) Robust mixture regression using mixture of
  different distributions. In: Recent Advances in Robust Statistics: Theory and
  Applications, Springer, pp 57--79

\bibitem[{Do{\u{}}ru and Arslan(2017)}]{dougru2016parameter}
Do{\u{}}ru FZ, Arslan O (2017) {Parameter estimation for mixtures of skew
  Laplace normal distributions and application in mixture regression modeling}.
  Communications in Statistics-Theory and Methods 46(21):10,879--10,896

\bibitem[{Fahrmeir et~al(2013)Fahrmeir, Kneib, Lang, and
  Marx}]{fahrmeir2013regression}
Fahrmeir L, Kneib T, Lang S, Marx B (2013) Regression: models, methods and
  applications. Springer Science \& Business Media

\bibitem[{Fan and Lv(2010)}]{FanLv2010}
Fan J, Lv J (2010) A selective overview of variable selection in high
  dimensional feature space. Statistica Sinica 20(1):101--148

\bibitem[{Fern and Brodley(2003)}]{fern2003random}
Fern XZ, Brodley CE (2003) {Random projection for high dimensional data
  clustering: A cluster ensemble approach}. In: Proceedings of the 20th
  International Conference on Machine Learning (ICML-03), pp 186--193

\bibitem[{Van~de Geer(2000)}]{van2000applications}
Van~de Geer SA (2000) Applications of empirical process theory, vol~91.
  Cambridge University Press Cambridge

\bibitem[{Gong et~al(2012{\natexlab{a}})Gong, Ye, and Zhang}]{gong2012robust}
Gong P, Ye J, Zhang C (2012{\natexlab{a}}) Robust multi-task feature learning.
  In: Proceedings of the 18th ACM SIGKDD International Conference on Knowledge
  Discovery and Data Mining, ACM, pp 895--903

\bibitem[{Gong et~al(2012{\natexlab{b}})Gong, Ye, and Zhang}]{gong2012multi}
Gong P, Ye J, Zhang Cs (2012{\natexlab{b}}) Multi-stage multi-task feature
  learning. In: Advances in Neural Information Processing Systems, pp
  1988--1996

\bibitem[{Hanley and McNeil(1982)}]{hanley1982meaning}
Hanley JA, McNeil BJ (1982) {The meaning and use of the area under a receiver
  operating characteristic (ROC) curve.} Radiology 143(1):29--36

\bibitem[{He and Lawrence(2011)}]{he2011graph}
He J, Lawrence R (2011) A graph-based framework for multi-task multi-view
  learning. In: Proceedings of the 28th International Conference on Machine
  Learning (ICML-11), pp 25--32

\bibitem[{Huang et~al(2012)Huang, Breheny, and Ma}]{Huang2012}
Huang J, Breheny P, Ma S (2012) A selective review of group selection in
  high-dimensional models. Statistical Science 27(4):481--499

\bibitem[{Jacob et~al(2009)Jacob, Vert, and Bach}]{jacob2009clustered}
Jacob L, Vert Jp, Bach FR (2009) Clustered multi-task learning: A convex
  formulation. In: Advances in Neural Information Processing Systems, pp
  745--752

\bibitem[{Jacobs et~al(1991)Jacobs, Jordan, Nowlan, and
  Hinton}]{jacobs1991adaptive}
Jacobs RA, Jordan MI, Nowlan SJ, Hinton GE (1991) Adaptive mixtures of local
  experts. Neural Computation 3(1):79--87

\bibitem[{Jalali et~al(2010)Jalali, Sanghavi, Ruan, and
  Ravikumar}]{jalali2010dirty}
Jalali A, Sanghavi S, Ruan C, Ravikumar PK (2010) A dirty model for multi-task
  learning. In: Advances in Neural Information Processing Systems, pp 964--972

\bibitem[{Ji and Ye(2009)}]{ji2009accelerated}
Ji S, Ye J (2009) An accelerated gradient method for trace norm minimization.
  In: Proceedings of the 26th Annual International Conference on Machine
  Learning, ACM, pp 457--464

\bibitem[{Jin et~al(2006)Jin, Goswami, and Agrawal}]{jin2006fast}
Jin R, Goswami A, Agrawal G (2006) Fast and exact out-of-core and distributed
  k-means clustering. Knowledge and Information Systems 10(1):17--40

\bibitem[{Jin et~al(2015)Jin, Zhuang, Pan, Du, Luo, and
  He}]{jin2015heterogeneous}
Jin X, Zhuang F, Pan SJ, Du C, Luo P, He Q (2015) Heterogeneous multi-task
  semantic feature learning for classification. In: Proceedings of the 24th ACM
  International on Conference on Information and Knowledge Management, ACM, pp
  1847--1850

\bibitem[{Jorgensen(1987)}]{jorgensen1987exponential}
Jorgensen B (1987) Exponential dispersion models. Journal of the Royal
  Statistical Society Series B (Methodological) pp 127--162

\bibitem[{Khalili(2011)}]{khalili2011overview}
Khalili A (2011) An overview of the new feature selection methods in finite
  mixture of regression models. Journal of Iranian Statistical Society
  10(2):201--235

\bibitem[{Khalili and Chen(2007)}]{khalili2012variable}
Khalili A, Chen J (2007) Variable selection in finite mixture of regression
  models. Journal of the American Statistical Association 102(479):1025--1038

\bibitem[{Koller and Sahami(1996)}]{koller1996toward}
Koller D, Sahami M (1996) Toward optimal feature selection

\bibitem[{Kubat(2015)}]{kubat2015introduction}
Kubat M (2015) An introduction to machine learning. Springer

\bibitem[{Kumar and Daum{\'e}~III(2012)}]{kumar2012learning}
Kumar A, Daum{\'e}~III H (2012) Learning task grouping and overlap in
  multi-task learning. In: Proceedings of the 29th International Conference on
  Machine Learning, Omnipress, pp 1723--1730

\bibitem[{Kyung~Lim et~al(2016)Kyung~Lim, Narisetty, and
  Cheon}]{kyung2016robust}
Kyung~Lim H, Narisetty NN, Cheon S (2016) Robust multivariate mixture
  regression models with incomplete data. Journal of Statistical Computation
  and Simulation pp 1--20

\bibitem[{Law et~al(2002)Law, Jain, and Figueiredo}]{law2002feature}
Law MH, Jain AK, Figueiredo M (2002) Feature selection in mixture-based
  clustering. In: Advances in Neural Information Processing Systems, pp
  625--632

\bibitem[{Li et~al(2014)Li, Liu, and Chan}]{li2014heterogeneous}
Li S, Liu ZQ, Chan AB (2014) Heterogeneous multi-task learning for human pose
  estimation with deep convolutional neural network. In: Proceedings of the
  IEEE Conference on Computer Vision and Pattern Recognition Workshops, pp
  482--489

\bibitem[{Liu et~al(2009)Liu, Ji, and Ye}]{liu2009multi}
Liu J, Ji S, Ye J (2009) Multi-task feature learning via efficient
  $\ell_{2,1}$-norm minimization. In: Proceedings of the 25th Conference on
  Uncertainty in Artificial Intelligence, AUAI Press, pp 339--348

\bibitem[{McLachlan and Peel(2004)}]{mclachlan2004finite}
McLachlan G, Peel D (2004) Finite mixture models. John Wiley \& Sons

\bibitem[{Neal and Hinton(1998)}]{neal1998view}
Neal RM, Hinton GE (1998) {A view of the EM algorithm that justifies
  incremental, sparse, and other variants}. In: Learning in Graphical Models,
  Springer, pp 355--368

\bibitem[{Nelder and Baker(1972)}]{nelder1972generalized}
Nelder JA, Baker RJ (1972) Generalized linear models. Encyclopedia of
  Statistical Sciences

\bibitem[{Nesterov et~al(2007)}]{nesterov2007gradient}
Nesterov Y, et~al (2007) Gradient methods for minimizing composite objective
  function. Tech. rep., UCL

\bibitem[{Passos et~al(2012)Passos, Rai, Wainer, and
  Daum{\'e}~III}]{passos2012flexible}
Passos A, Rai P, Wainer J, Daum{\'e}~III H (2012) Flexible modeling of latent
  task structures in multitask learning. In: Proceedings of the 29th
  International Conference on Machine Learning, Omnipress, pp 1283--1290

\bibitem[{Sch{\"o}lkopf et~al(1998)Sch{\"o}lkopf, Smola, and
  M{\"u}ller}]{scholkopf1998nonlinear}
Sch{\"o}lkopf B, Smola A, M{\"u}ller KR (1998) Nonlinear component analysis as
  a kernel eigenvalue problem. Neural Computation 10(5):1299--1319

\bibitem[{She and Chen(2017)}]{she2015robust}
She Y, Chen K (2017) Robust reduced-rank regression. Biometrika 104(3):633--647

\bibitem[{She and Owen(2011)}]{she2011outlier}
She Y, Owen AB (2011) Outlier detection using nonconvex penalized regression.
  Journal of the American Statistical Association 106(494):626--639

\bibitem[{St{\"a}dler et~al(2010)St{\"a}dler, B{\"u}hlmann, and Van
  De~Geer}]{stadler2010}
St{\"a}dler N, B{\"u}hlmann P, Van De~Geer S (2010) $\ell_1$-penalization for
  mixture regression models. Test 19(2):209--256

\bibitem[{Strehl and Ghosh(2002{\natexlab{a}})}]{Strehl2003Cluster}
Strehl A, Ghosh J (2002{\natexlab{a}}) Cluster ensembles---a knowledge reuse
  framework for combining multiple partitions. Journal of Machine Learning
  Research 3(Dec):583--617

\bibitem[{Strehl and Ghosh(2002{\natexlab{b}})}]{strehl2002cluster}
Strehl A, Ghosh J (2002{\natexlab{b}}) Cluster ensembles: a knowledge reuse
  framework for combining partitionings. In: 18th National Conference on
  Artificial Intelligence, American Association for Artificial Intelligence, pp
  93--98

\bibitem[{Tan et~al(2010)Tan, Kaddoum, and Le~Yi~Wang}]{tan2010decision}
Tan Z, Kaddoum R, Le~Yi~Wang HW (2010) Decision-oriented multi-outcome modeling
  for anesthesia patients. The Open Biomedical Engineering Journal 4:113

\bibitem[{Van Der~Maaten et~al(2009)Van Der~Maaten, Postma, and Van~den
  Herik}]{van2009dimensionality}
Van Der~Maaten L, Postma E, Van~den Herik J (2009) {Dimensionality reduction: A
  comparative}. J Mach Learn Res 10:66--71

\bibitem[{Van Der~Vaart and Wellner(1996)}]{van1996weak}
Van Der~Vaart AW, Wellner JA (1996) Weak convergence. Springer

\bibitem[{Wedel and DeSarbo(1995)}]{wedel1995mixture}
Wedel M, DeSarbo WS (1995) A mixture likelihood approach for generalized linear
  models. Journal of Classification 12(1):21--55

\bibitem[{Weruaga and V{\'\i}a(2015)}]{weruaga2015sparse}
Weruaga L, V{\'\i}a J (2015) Sparse multivariate gaussian mixture regression.
  IEEE Transactions on Neural Networks and Learning Systems 26(5):1098--1108

\bibitem[{Xian~Wang et~al(2004)Xian~Wang, bing Zhang, Luo, and
  Wei}]{xian2004robust}
Xian~Wang H, bing Zhang Q, Luo B, Wei S (2004) Robust mixture modelling using
  multivariate t-distribution with missing information. Pattern Recognition
  Letters 25(6):701--710

\bibitem[{Yang et~al(2009)Yang, Kim, and Xing}]{yang2009heterogeneous}
Yang X, Kim S, Xing EP (2009) Heterogeneous multitask learning with joint
  sparsity constraints. In: Advances in Neural Information Processing Systems,
  pp 2151--2159

\bibitem[{Yuksel et~al(2012)Yuksel, Wilson, and Gader}]{yuksel2012twenty}
Yuksel SE, Wilson JN, Gader PD (2012) Twenty years of mixture of experts. IEEE
  Transactions on Neural Networks and Learning Systems 23(8):1177--1193

\bibitem[{Zhang et~al(2012)Zhang, Shen, Initiative et~al}]{zhang2012multi}
Zhang D, Shen D, Initiative ADN, et~al (2012) Multi-modal multi-task learning
  for joint prediction of multiple regression and classification variables in
  alzheimer's disease. NeuroImage 59(2):895--907

\bibitem[{Zhang and Yeung(2011)}]{zhang2011multi}
Zhang Y, Yeung DY (2011) Multi-task learning in heterogeneous feature spaces.
  In: 25th AAAI Conference on Artificial Intelligence and the 23rd Innovative
  Applications of Artificial Intelligence Conference, AAAI-11/IAAI-11, San
  Francisco, CA, United States, 7-11 August 2011, Code 87049, Proceedings of
  the National Conference on Artificial Intelligence, p 574

\bibitem[{Zhou et~al(2011)Zhou, Chen, and Ye}]{zhou2011clustered}
Zhou J, Chen J, Ye J (2011) Clustered multi-task learning via alternating
  structure optimization. In: Advances in Neural Information Processing
  Systems, pp 702--710

\end{thebibliography}

\end{document}